\documentclass[aos]{imsart}

\RequirePackage{amsthm,amsmath,amsfonts,amssymb}
\RequirePackage[authoryear]{natbib}%% uncomment this for author-year citations
\RequirePackage[colorlinks,citecolor=blue,linkcolor=blue,urlcolor=blue,pagebackref]{hyperref}
\RequirePackage{graphicx}
\usepackage{t1enc}
\usepackage[mathscr]{eucal}
\usepackage{indentfirst}
\usepackage{graphicx}
\usepackage{graphics}
\usepackage{pict2e}
\usepackage{epic}
\usepackage{epstopdf}
\usepackage[utf8]{inputenc}
\usepackage{amsmath,amsthm,amssymb,amsfonts}%,thmtools}
\usepackage{mathrsfs}
\usepackage{bbm,bm}
\usepackage{color}
\usepackage{latexsym}
\usepackage{graphicx}
\usepackage{enumerate}
\usepackage{comment}
\usepackage{subcaption}
\usepackage{caption}
\usepackage{multirow}
\usepackage{algpseudocode}
\usepackage{bm}
\usepackage{float}
\usepackage[ruled,linesnumbered]{algorithm2e}
\usepackage{booktabs}
\usepackage{hyperref}
\usepackage{adjustbox}
\usepackage{changepage}
\usepackage{algorithm2e}
\usepackage{caption}
\usepackage{yhmath}
\usepackage{ulem}

\usepackage{algorithmicx,algpseudocode}
\usepackage{amssymb}
\usepackage{pstricks,pst-node,pst-tree}
\usepackage{xr}
\usepackage{tikz}

\usepackage{stackengine}

\startlocaldefs
\theoremstyle{plain}

\newtheorem{theorem}{Theorem}[section]
\newtheorem{lemma}[theorem]{Lemma}
\newtheorem{corollary}{corollary}
\newtheorem{proposition}{proposition}
\theoremstyle{remark}
\newtheorem{definition}[theorem]{Definition}
\newtheorem{assumption}{assumption}
\newtheorem{example}{Example}

\endlocaldefs

\begin{document}

\begin{frontmatter}
\title{Neighborhood Adaptive Estimators for Causal Inference under Network Interference}
%\title{A sample article title with some additional note\thanksref{t1}}
\runtitle{Neighborhood Adaptive Estimators}
%\thankstext{T1}{A sample additional note to the title.}

\begin{aug}
%%%%%%%%%%%%%%%%%%%%%%%%%%%%%%%%%%%%%%%%%%%%%%%
%% Only one address is permitted per author. %%
%% Only division, organization and e-mail is %%
%% included in the address.                  %%
%% Additional information can be included in %%
\author[A]{\fnms{Alexandre}~\snm{Belloni} \ead[label=e1]{ab5@duke.edu}},
\author[B]{\fnms{Fei}~\snm{Fang}\ead[label=e2]{fei.fang@yale.edu}%\orcid{0000-0000-0000-0000}
}
\and
\author[C]{\fnms{Alexander}~\snm{Volfovsky}\ead[label=e3]{alexander.volfovsky@duke.edu}}
%%%%%%%%%%%%%%%%%%%%%%%%%%%%%%%%%%%%%%%%%%%%%%
%% Addresses                                %%
%%%%%%%%%%%%%%%%%%%%%%%%%%%%%%%%%%%%%%%%%%%%%%
\address[A]{The Fuqua School of Business,
Duke University \printead[presep={ ,\ }]{e1}}
\address[B]{Department of Biostatistics, Yale University\printead[presep={,\ }]{e2}}
\address[C]{Department of Statistical Science, Duke University\printead[presep={,\ }]{e3}}

\end{aug}

\begin{abstract}
Estimating causal effects has become an integral part of most applied fields. In this work we consider the violation of the classical no-interference assumption with units connected by a network. For tractability, we consider a known network that describes how interference may spread. Unlike previous work the radius (and intensity) of the interference experienced by a unit is unknown and can depend on different (local) sub-networks and the assigned treatments. We study estimators for the average direct treatment effect on the treated in such a setting under additive treatment effects. We establish rates of convergence and distributional results. The proposed estimators considers all possible radii for each (local) treatment assignment pattern. In contrast to previous work, we approximate the relevant network interference patterns that lead to good estimates of the interference. To handle feature engineering, a key innovation is to propose the use of synthetic treatments to decouple the dependence. 
We provide simulations, an empirical illustration and insights for the general study of interference.
\end{abstract}

%\begin{keyword}[class=MSC]
%\kwd[Primary ]{00X00}
%\kwd{00X00}
%\kwd[; secondary ]{00X00}
%\end{keyword}


\end{frontmatter}
%%%%%%%%%%%%%%%%%%%%%%%%%%%%%%%%%%%%%%%%%%%%%%
%% Please use \tableofcontents for articles %%
%% with 50 pages and more                   %%
%%%%%%%%%%%%%%%%%%%%%%%%%%%%%%%%%%%%%%%%%%%%%%
%\tableofcontents

\section{Introduction}

Estimation of causal effects has become a staple, if not a requirement, in many applied fields. This breadth comes due to the realization that scientific and societal knowledge requires rigor beyond predictions. This boon in applications has strained existing causal inference tools and led to a mass development of theory and methodology that can address modern causal questions. Much of the early work on causal inference concerned itself with the design and analysis of controlled experiments where a plethora of independence assumptions made them plausible and implementable \citep{splawa1990application,fisher1935design}. This work was then generalized by developing classes of assumptions and methodologies that allowed practitioners to emulate simple experimental designs using observational data, still under various independence assumptions \citep{rubin1978bayesian,manski1990nonparametric}. However, many modern causal settings 
% and questions 
have made these assumptions untenable: individuals, corporations, and countries are increasingly connected, with the deployment of individual treatments potentially affecting the outcomes of large groups of individuals, either directly or indirectly. 

For example, an individual's probability of changing housing may increase if (i) he or she receives a direct incentive to move or (ii) if his or her nearby households receive a housing voucher to move \citep{sobel2006randomized}. Similarly, the probability that an individual will be infected by a virus can decrease by several mechanisms: (i) he or she can receive the vaccine, which provides direct protection against infection, or (ii) his or her friends could be vaccinated, making them less likely to be infected, which in turn reduces the amount of exposure on their friends \citep{ross1916application,zivich2021assortativity}. In these settings, there is a potential direct treatment effect (from a unit's own treatment status) and a potential indirect treatment effect (due to the treatment status of other units); see \citet{cai2015social,hu2021average} for other examples. Importantly, the possibility of this indirect treatment effect, dubbed ``interference'' among units is a violation of the standard Stable Unit Treatment Value Assumption \citep[SUTVA,][]{rubin1978bayesian,rubin1980randomization} and Individualized Treatment Response \citep[ITR,][]{manski1990nonparametric} assumption. Understanding what estimands are plausible in these settings, what identification assumptions are needed, and what estimation strategies can be employed has been the central focus of recent work. This is particularly critical in the context of observational studies, where confounding can further complicate the identification and estimation processes. 

In order to identify and estimate the treatment effects under interference, different types of estimands and causal assumptions have been considered \citep{toulis2013estimation,aronow2017estimating,sussman2017elements,jagadeesan2020designs}. For example, \citet{hudgens2008toward} propose a partial interference model in which units are split into groups where interference exists within each group but not outside the group. This assumption offers relatively well-understood dependence structures between one's potential outcomes and the treatments of others \citep{tchetgen2012causal}. The constant treatment response (CTR) assumption of \citet{manski2013identification} defines 
% \cite{aronow2017estimating} propose
a class of exposure mappings which can be seen as a representation of the whole treatment vector and any other covariates. 
% (including network structure if it affects potential outcome jointly with treatments). 
The exposure mapping is a general framework to capture complex structures that can influence interference. This has recently been specialized for the setting of network interference \citep{aronow2017estimating,savje2017average,savje2021causal}.

Although it is widely recognized that interference and other ``network effects'' can depend on various types of network structures \citep{cartwright1956structural,awan2020almost}, most of the existing literature focuses on neighborhood interference problems \citep{karwa2018systematic, sussman2017elements}. In these settings, interference is restricted to those who share edges and, importantly, the strength or type of interference does not depend on whether those with direct edges have connections between them. In many applied settings, this interference characterization may not be sufficiently rich to fully capture the variability in potential outcomes as a function of the full treatment vector \citep{savje2021causal}. We consider a broad generalization of the neighborhood interference approach by studying interference that can occur within a distance $m_i$ of the unit $i$ in a graph; importantly, this interference can depend on the full subgraph of radius $m_i$, and not just on edge counts or density measures. 
% Under this assumption, the effective treatments \citep{manski2013identification} which affect potential outcome are within an $m_i$-distance neighborhood of a unit. 
When $m_i=1$ for every unit, this setting recovers the neighborhood interference assumption of many papers in the literature, e.g. \citet{sussman2017elements, forastiere2021identification}. However, we allow the value of $m_i$ to be unknown, to vary between units, to depend on the treatment assignment configurations, and to possibly be large. 

One of the main limitations in estimating causal effects in the presence of interference is the large amount of data that is needed to estimate treatment effects based on flexible exposure mappings. For example, the exposure mapping ``(is treated?, has at least one treated friend?)'' maps to four possible counterfactuals, while the naive generalization to ``(is treated?, has $k$ treated friends?)'' maps to $2 d$ possible counterfactuals where $d$ is the degree of an individual. Even this simple generalization can lead to positivity violations if care is not taken in specifying how the potential outcomes interact with exposure mappings.
% At the same time, estimators directly tied to such  extremely rich models for interference might not lead to good estimators due to the availability of data. 
With that in mind, our aim is to develop adaptive estimators for interference that balance the availability of data and the interference patterns themselves. That is, we will consider unit-specific values of $m_i$ that
control the complexity of individual potential outcomes. 
% for each individual.

This approach is in the spirit of the Constant Treatment Response assumption of \citet{manski2013identification} which suggests that different individuals may have different global treatment patterns that correspond to the same counterfactual values. Thus, our work aims to contribute to the somewhat fundamental question of obtaining an adaptive restriction of the potential outcomes that allows for a more efficient estimation of the estimands of interest. In our case, we search for these restrictions based on available network information, treatment assignment configurations, and the induced bias and variance trade-off. The network provides a natural hierarchy to search over, while the treatment assignments (the postulated source of interference) allow us to match patterns across the network. From an estimation perspective, since we are searching across possible patterns, this approach is connected with model selection and feature engineering, where regularization tools play a key role. As discussed in the following, we construct a form of regularization based on synthetic treatments that replaces the role of sample splitting, which is harder in network settings. Moreover, our approach incorporates possible misspecifications on model selection and controls their impact to obtain adaptive estimators. 

Our setting follows the additive effect under the Neighborhood Interference Assumption (ANIA) framework in \cite{sussman2017elements}. \cite{sussman2017elements} showed that the potential outcomes satisfy additivity of main effects if and only if the potential outcomes can be parameterized through a partially linear model.  Although we have different results for the estimation of the average direct treatment effect on the treated, our main result pertains to the case where there is selection on the treatment assignment but those are independent across units given observables.

\subsection{Recent development beyond neighborhood interference}
Our proposed approach defines an explicit hierarchy of interference patterns in a network, allowing an estimator to trade off bias and variance of the estimation of counterfactual outcomes. 
% the estimator will consider larger neighborhoods only if there is a substantial reduction in the misspecification relative to the incremental variance of considering a larger model. 
We note that the general notion that interference can be hierarchical and reach deep through a social network has previously been mentioned, albeit in passing and with the requirement that interference is \textit{the same} for each individual within the network \citep{puelz2019graph,eckles2016design,jagadeesan2020designs}. 
% Importantly, in the previously considered settings, such interference was required to be \textit{the same} for each individual within the network. This is clearly a generalization of the simplest exposure mapping approaches to neighborhood interference (as one simply considers a friends-of-friends neighborhood, and so on) and so can be considered a directly extension of that work, albeit at sometimes unacceptable computational and scientific costs. For example, positivity violations that frequently plague exposure mapping approaches are liable to be greatly exacerbated by expanding the types of acceptable exposure mappings. 
Although only briefly mentioned, the generalization of these approaches is natural: if friends can lead to interference, then why not friends of friends? However, on closer inspection, it is not immediately clear how this would be implemented. For example, in \citet{jagadeesan2020designs} this would require specifying the smoothness of the interference function that would presumably depend on the depth of interference. This may or may not lead to a specification with a computationally tractable design. Similar issues with computation are likely when you need to calculate the statistic in tests for deeper than one neighborhood network interference. 

Recently, \citet{leung2022causal} proposed an approximate neighborhood interference (ANI) assumption that states that interference decays when the distance between treatments and the ego increases. When the distance goes to infinity, the interference decays to zero. In this approach exposure mappings are used to define interesting estimands that summarize treatment and spillover effects. An explicit connection between this assumption and a type of generative network dependence called $\psi$-dependence leads to the result that the Horvitz$ $-Thompson estimator for the direct effect is asymptotically normal \citep[building upon tools from,][]{kojevnikov2021limit}.  This approximation is very much in the spirit of our work.

Another important development in the literature concerns misspecified exposure mappings. While knowing the correct exposure mapping provides a natural recipe for asymptotically high-quality estimators of many estimands, misspecification of the exposure mapping leads to inappropriate bias in estimation. \citet{savje2021causal} characterizes this bias and provides conditions under which consistency of estimators can be achieved (essentially a rate of decay of the misspecification of the exposure mapping). Although our work is not guaranteed to be robust to all misspecification, by choosing a rich class of interference patterns, we hope to alleviate many of these concerns. 

In the same spirit, \cite{yuan2021causal} propose a methodology to detect the exposure conditions for a given dataset by precomputing a large class of causal network motifs for each unit (i.e., all possible treated neighborhood graphs for each unit). They then use regression trees to identify important causal network motifs and thus limit the classes of exposures that they need to study. This is similar in spirit to the approach in \citet{awan2020almost}, which selects relevant causal motifs to adjust for interference in an experiment. Both of these approaches become computationally challenging as the degree of nodes and the depth of interference increase even to modest sizes. 

An interesting recent paper \cite{leung2022unconfoundedness} proposes the use of graphic neural networks (GNN) to learn the relevant representation of interference in a very general setting with possible joint determination of treatment. The paper establishes both positive and negative identification results in general settings. The work also derives asymptotic normality results under appropriate assumptions on the rates of convergence of the GNN predictions. A key related question in that paper is the determination of the relevant depth of interference which directly relates with the number of layers in the GNN. \cite{leung2022unconfoundedness} highlights the need to impose additional structure to strike some balance between applicability, interpretability, and tractability. Our setting attempts one such balance by imposing assumptions of additivity of the direct treatment effect while allowing nonparametric estimates for the interference function where the relevant depth of interference can be node specific and dependent on the treatment assignments as well.
% Then for each unit in the network, it corresponds to a vector which denotes the proportion for the causal network motifs. Based on these vectors, they utilize regression tree to construct partitions for the causal network motifs and pick each leaf as the important exposure condition.

\subsection{Our contribution}
Our proposed approach differs from the previous literature by recognizing that not all individuals experience interference in identical ways. We codify this by allowing the interference experienced by each individual to depend on a potentially different subgraph of the larger social network. The following examples illustrate the need for different interference maps for different individuals within the same network.

\begin{example}[Vaccine efficacy in an enclosed population]
    Consider studying the efficacy of a vaccine in an enclosed population, where individuals are free to interact in small groups, but most individuals are very far apart in the network. For example, this can happen in a large hospital where patients interact closely with patients in their room, slightly less closely with those in their wing and contact between wings is facilitates only through medical personnel. 
    Note that there are clearly different patterns of interference that can be expected: a vaccine given to a patient in \texttt{Room 1A} is more likely to reduce the chances of infection for others in \texttt{Room 1A}, and to a lesser extent to those in \texttt{Wing 1} (where \texttt{Room 1A} is located). On the other hand, a vaccine administered to a patient in any of the \texttt{Wings} is unlikely to interfere with the pharmacist, who has no direct patient contact whatsoever and interacts with them only through a series of doctors (whose treatment can, in fact, interfere with the pharmacist's results).  
\end{example}

\begin{example}[After school programs on child outcome]
Consider studying the impact of an after school program on child outcomes (e.g., behavior, choice) of middle school students. In this setting, participants develop friendship and relations during the regular class time, but only a subset of students decides to enroll into the after-school program (which might not be of the interest of the parents or might have a limited number of participants allowed). However, participants in the after school program typically interact and discuss with friends (regardless of whether these friends are in the program or not). It is also possible that a student hears about the program by talking to friends who participated in the program or talking to friends who have friends who participated in the program.
\end{example}

Our broad goal in this work is to study the direct effect of a treatment on individuals while accounting for any potential and variable interference. Under common additivity, positivity and unconfoundedness assumptions and using our new generalization of neighborhood interference (described in detail in Section~\ref{Assum}), we derive an augmented inverse probability weighting (AIPW) estimator and an outcome regression (OR) estimator for the average direct effect on the treated. We demonstrate that these estimators are asymptotically well behaved when the types of interference are known, and more importantly, even when those need to be adaptively selected within the estimator.  

Our analysis of the proposed interference estimator builds upon ideas from the variable length Markov chain literature for discrete alphabet time series; see, e.g., \cite{buhlmann1999variable,ferrari2003estimation,garivier2011context,o2012adaptive, belloni2017approximate}. In that setting, the order of the Markov process is unknown and can depend on the specific past history. %Therefore, the (path-dependent) order also needs to be learned. 
In our setting, for each node, we will learn the number of hops\footnote{Here a $m$-hop neighborhood of node $i$ includes the induced subgraph of all nodes for whom the shortest path to $i$ is less than or equal to $m$ edges.} needed to obtain a good approximation for the interference which depends on the (sub)network and assigned treatments. We will rely on a Lepski-type estimator (e.g. \cite{lepskii1992asymptotically,lepski1997optimal,birge2001alternative,spokoiny2016nonparametric}) to select the relevant order $k$ for each individual that defines the relevant (sub)network interference patterns. However, because the choice of $k$ is allowed to also depend on the treatment assignments to nearby nodes, the learning of such treatment-graph patterns is a form of feature engineering that creates challenges. Thus, the proposed approach is an adaptive extension of many existing estimators that assumes a known and fixed $k$-hop neighborhood determines the interference that complements the literature in that direction. Our interference estimator builds on \cite{belloni2017approximate} which estimates a variable length Markov chain for a stationary time series. In our setting, we establish oracle inequalities for the performance of the interference estimator. This result implies that our interference estimator performs similarly to an (oracle) estimator that optimally balances (the unknown) bias and variance within the class. In well-specified (low-degree) cases, we provide specific rates implied by the oracle inequality. 

As mentioned earlier, the estimation of the interference function is an intermediary step in the estimation of the average direct treatment effect (ADTE). 
% Indeed, we will propose two different estimators for the ADTE. 
The analysis of the outcome regression estimator is a direct consequence of the estimation of the interference function, and it inherits the rates of convergence. On the other hand, the augmented inverse probability weighted estimator leverages the independence between treatment assignments (while allowing heterogeneous propensity scores) to create an orthogonal moment condition which helps diminish the impact of using estimated nuisance parameters (e.g. interference function estimates). However, our interference estimator not only learns the relevant interference values, but also generates the patterns to be tested. Such feature engineering creates an additional dependence that impacts the rate of convergence of the interference estimator. This dependence is problematic in the analysis of the distribution of the AIPW estimator. To mitigate this dependence, we pursue a way to decouple feature engineering from the rest of the estimation. Sample-splitting has been a standard way to decouple the dependence within nuisance parameters, but, in our setting, the network dependence severely complicates the use of sample-splitting since we are allowing for the dependence radius to be unknown. A major contribution of our work is to propose the so-called synthetic treatments to bypass these difficulties. This leads to a modified interference estimator that breaks the dependency between the feature engineering step and the estimation of the interference function. In our analysis, the generation of synthetic treatments acts similarly to sampling-splitting and cross-fitting (used in cross-sectional data) to reduce the dependence across the nuisance function estimators. Our analysis of the AIPW estimator accounts for the feature engineering step of learning the interference patterns and provides conditions under which the estimator is asymptotically unbiased and asymptotically normal. Finally, we discuss conservative and non-conservative estimates for the variance. 

\section{Setup}
\label{Assum}

We consider a finite population model with $n$ units connected through a network $G$  that articulates how interference can propagate between units. We denote by $\widetilde Y_i(z) = \widetilde Y_i(z_i,z_{-i})$ the potential outcome of the $i$th unit when $z\in\{0,1\}^n$ is the assigned treatment to the whole network.\footnote{Under the standard SUTVA, which we do not make, one would have $\widetilde Y_i(z)=\widetilde Y_i(z_i,0_{-i})$ for all $z\in\{0,1\}^{n}$.}  We observe the network $G$ (interference graph), and the triple of treatment, covariates and outcomes. $(Z_i, X_i, Y_i)_{i=1}$, for each unit $i \in [n]$. The outcomes are related to the potential outcomes by \begin{equation}\label{def:NoMultipleOutComes}Y_i=\widetilde Y_i(Z), \ \ i\in[n]. \end{equation} 

We will consider the following assumptions.

\begin{assumption}[Additivity of Main Effects]
\label{additivity}
For each unit $i \in [n]$ and any treatment assignment $z\in\{0,1\}^{n}$, we have that
\begin{equation*}
    \begin{split}
       \tau_i := \mathbb{E}\left[ \widetilde{Y}_i(1,0_{-i}) - \widetilde{Y}_i(0,0_{-i}) \mid X_i \right] = \mathbb{E}[\widetilde{Y}_i(1,z_{-i})  - \widetilde{Y}_i(0,z_{-i} ) \mid X_i  ]& 
    \end{split}
\end{equation*}
\end{assumption}

\begin{assumption}[Conditional Unconfoundedness]
\label{cond_unconf}
We have that 
\begin{equation*}
    \begin{split}
     \widetilde{Y}_i(z_i,\bm{z}_{-i}) \perp \!\!\! \perp  Z_i \mid X_i  \ \ \ \ \  \mbox{for all} \  z \in \{0,1\}^{n}.
    \end{split}
\end{equation*}    
\end{assumption}

\begin{assumption}[Conditionally Independence of Treatment]
\label{ass:cond_ind}
Conditionally on $X_i$, the treatment assignment of $Z_i$ is independent of treatments and features of other units, namely
\begin{equation*}
    \begin{split}
     (Z_{-i},X_{-i}) \perp \!\!\! \perp  Z_i \mid X_i  
    \end{split}
\end{equation*}    
and $\exists c>0$ such that $c \leq \mathbb{P}(Z_i=1\mid X_i) \leq 1-c$ for all values of $X_i$.
\end{assumption}

\begin{assumption}[Covariate independent baseline]\label{Assump:NoDirectImpact} 
We assume that $\mathbb{E}[\widetilde{Y}_i(0,z_{-i})\mid X_i]=\mathbb{E}[\widetilde{Y}_i(0,z_{-i})]$ 
for all $i\in[n]$ and $z\in\{0,1\}^{n-1}$ . 
\end{assumption}

%\textcolor{blue}{[CHECK THE WHOLE PARAGRAPH]} 
Assumption \ref{additivity} is standard in the literature on network interference
\citep{toulis2013estimation,sussman2017elements,karwa2018systematic,awan2020almost,jagadeesan2020designs}. 
% Assumption \ref{cond_unconf} has two parts.  
Conditional on covariate information, Assumption \ref{cond_unconf} implies that treatment assignment is independent of potential outcomes, while Assumption \ref{ass:cond_ind} implies that treatment assignments are independent. Both of these assumptions are trivially satisfied in the experimental setting, where the treatment is assigned uniformly at random to units. Assumption \ref{Assump:NoDirectImpact} rules out a direct effect of $X$ on the potential outcomes of the untreated units.   When $X_i$ is of low dimension, it is practical to stratify the analysis by the values of $X$. In that case Assumption \ref{Assump:NoDirectImpact} can be relaxed and ideas developed here would also extend to that (admittedly with larger sample size requirements). Indeed, because we do not need to identify their impacts separately, we would still be able to recover the estimand of interest. Nonetheless, Assumptions \ref{cond_unconf}, \ref{ass:cond_ind} and \ref{Assump:NoDirectImpact} allow designs with a nonconstant propensity score $ e(X_i) := \mathbb{P}(Z_i=1\mid X_i)$ and heterogeneous treatment effects $\tau_i$ that depend on $X_i$.

Assumption \ref{Assump:NoDirectImpact} also allows for heterogeneous direct treatment effect and is made to simplify the exposition and technical results but can be relaxed in different ways. For example, if covariates are discrete it is possible to estimate the interference function for each covariate value and our results would carry over with minor adjustments. In Section \ref{sec:Additional} we discuss a modified estimator that allows for different specifications that postulate the interference to be a (smooth) function of the covariance {\it given an interference pattern}. Moreover, if the covariates $X_i$'s are independent of the graph $G$, the interference is likely to be approximately well balanced between the treatment and control groups. % and the naive estimator is expected to perform well for the ATT but not for ATE. (In particular, under SUTVA, i.e. no interference, the selection bias would be zero and the naive estimator is consistent for ATT but not for ATE.) 
 However, we are motivated by situations where we would expect some relation between the covariate $X_i$ and the graph $G$ (for example, $X_i$ is the degree of unit $i$). %(in particular when $G$ is a random graph which depends on $X$).   

Next, we impose restrictions on the potential outcomes. In particular, how the interference network $G$ defines the interference patterns in the potential outcomes. We will allow the interference to depend on a $m$-hop neighborhood of the labelled graph\footnote{Formally we define $G^z$ as the network $G$ where the $i$th node is labelled as $Z_i$. We let $G_i^z$ as $G^z$ but the $i$th node is labeled as the ``ego" that is the center to grow a $m$-hop subgraph.} $G^z$ but the value of $m$ is unknown and potentially dependent on the local network and treatments. Here, a $m$-hop neighborhood of node $i$ includes the induced subgraph of all nodes for whom the shortest path to $i$ is less than or equal to $m$ edges. 

\begin{assumption}[Neighborhood Interference]
\label{Neigh_inter} 
There is a mapping $\gamma(\cdot)$ on labelled subgraphs such that for any treatment assignment vectors $\bm{z}$ and $\bm{z}^\prime$ with $\gamma( G_i^z)=\gamma(G_i^{z'})$ 
$$\widetilde{Y}_i(1,\bm{z}_{-i})= \widetilde{Y}_i(1,\bm{z}^\prime_{-i}) \ \ \mbox{and} \ \ \widetilde{Y}_i(0,\bm{z}_{-i})= \widetilde{Y}_i(0,\bm{z}^\prime_{-i})$$
 and  for any $i \in [n]$ such that  $g=\gamma(G_i^{z})$, $f(g) := \mathbb{E}[\widetilde Y_i(0,z_{-i})]$ is unit independent.  
\end{assumption}

Assumption \ref{Neigh_inter} is a restriction to the potential outcomes in the spirit of the constant treatment response assumption of \citet{manski2013identification}. Without loss of generality, it postulates that the labeled subgraph centered on $i$ with radius $m_i$ is sufficient to determine potential outcomes. Throughout the paper, we have 
% $$\gamma(G_i^{z}) = \gamma_0( \mbox{$m_i$-hop neighborhood from $i$ in} \ G^{z} )$$
$$\gamma(G_i^z) = \gamma_0(m_i\textrm{-hop neighborhood from }i\textrm{ in }G^z)$$
where $\gamma_0$ is a known mapping and $m_i=m(G_i^{z})$ is an unknown function. Importantly, we allow $m_i$ to be equal $n$ for every $i\in[n]$\footnote{In the next section we introduce approximation errors of using a smaller value than $m_i$ to estimate the interference function for a pattern $g$.}.  As discussed below, combined with the additivity assumption \ref{additivity}, the interference will be a function of $\gamma(G_i^Z)$, that is, $f_i:=f(\gamma(G_i^Z))$. Although the function $\gamma_0$ might be the identity, in applications  $\gamma_0$ might summarize the corresponding subgraph. We will require that if the  $m$-hop neighborhoods of two units $i$ and $j$ match, then they also match the number of treated nodes in each $k$ hop neighborhood for $k\leq m$. 
That is, letting $G_i^Z(k):=k\mbox{-hop neighborhood from $i$ in $G^Z$}$, we have that  $\gamma_0(G_i^Z(m))=\gamma_0(G_j^Z(m))$ implies that $\gamma_0(G_i^Z(k))=\gamma_0(G_j^Z(k))$ for $k\leq m$. In what follows, $T_{i,k}$ denotes the number of treated nodes at distance $k$ from $i$.

\begin{example}[Number of treated neighbors]\label{Example:NumberOfTreatedNeighbors}
Consider the (common in the literature) assumption that the exposure mapping depends on the number of treated neighbors $T_{i,1}:=\sum_{j=1}^n A_{ij} Z_j$ where $A$ is the adjacency matrix associated with $G$. Typically, it is assumed that $m_i=1$ for every $i\in[n]$. In our setting, this would lead to an interference function of the form $f(\gamma(G_i^Z))=f(\gamma_0(G_i^Z(1))=f(T_{i,1})$ where $\gamma(G_i^Z)=\gamma_0(G_i^Z(1))=\sum_{j=1}^n A_{ij} Z_j=:T_{i,1}$. In this case, there is no approximation error considering $G_i^Z(1)$ instead of the whole graph. By construction we have $\gamma_0$ to define a partition by having $\gamma_0(G_i^Z(k))=(T_{i,1},T_{i,2},\ldots,T_{i,k})$ where $T_{i,k}$ denotes the number of treated nodes at distance (exactly) $k$. Thus, if $\gamma_0(G_i^Z(m))=\gamma_0(G_j^Z(m))$, it follows that $\gamma_0(G_i^Z(k))=\gamma_0(G_j^Z(k))$ for $k\leq m$.
\end{example}
\begin{example}[Generalized Number of treated neighbors]\label{Example:Second}  Suppose that the interference function is given by \begin{equation}\label{eq:example:gamma1}f(\gamma(G_i^Z)):= \mu(T_{i,1},T_{i,2},\ldots,  T_{i,n})\end{equation} where $\mu$ is an unknown function and $T_{i,k}$ denotes the number of treated nodes at distance (exactly) $k$. This allows for a rich family of interference patterns including layers to impact interference at a different marginal rate (e.g., $\mu(T_{i,1},T_{i,2},\ldots,  T_{i,n})=\sum_{k=1}^n \mu_k \cdot T_{i,k}$) or a non-linear function (e.g., $\mu(T_{i,1},T_{i,2},\ldots,  T_{i,n})=\bar\mu/\min_\ell\{ \ell \geq 0 : T_{i,\ell}>0\}$). In this case, we note that $\gamma_0( G_i^Z(m)):=(T_{i,1},T_{i,2},\ldots,T_{i,m})$ which is used in the implementation of the estimator while the knowledge of the functional form of $\mu$ is not needed. In applications we expect that the dependence on $T_{i,\ell}$ for large $\ell$ is negligible.
\end{example}

\begin{example}[Isomorphic patterns]
Consider the more general case where the interference function is given by $\gamma_0$ being the identity function, $ \gamma_0(G_i^Z(m)) \cong G_i^Z(m)$ (where $\cong$ denotes equality up to an isomorphism), so that
$$ f(\gamma(G_i^Z))=f(\gamma_0(G_i^Z(m_i)).$$
In this case we have $\gamma_0(G_i^Z(m))\cong \gamma_0(G_j^Z(m))$ implies that $\gamma_0(G_i^Z(k)) \cong \gamma_0(G_j^Z(k))$ for $k\leq m$. This specification allows for a very rich class of interference functions with a large number of potential outcomes. However, such generality comes at the cost of (i) the existence of matches harder and the computational problem of verifying that two nodes have the same pattern (i.e. $\gamma_0(G_i^Z(m)) \cong 
\gamma_0(G_j^Z(m))$) more demanding as it is testing if two graphs are isomorphic.
\end{example}

\begin{example}[Linear-in-means] Considering the special case of the linear-in-means model (\cite{manski1993identification}) where it is assumed 
$ Y_i = \alpha + \beta \frac{\mbox{$\sum_{j=1}^n A_{ij}Y_j$}}{\mbox{$\sum_{j=1}^nA_{ij}$}} +  \gamma \frac{\mbox{$\sum_{j=1}^n A_{ij}z_j$}}{\mbox{$\sum_{j=1}^nA_{ij}$}} +\xi_i  $ where $A$ denotes the adjacency matrix, $\xi$ independent zero mean shocks, $\beta$ the coefficient of endogenous peer effects, and $\gamma$ the exogenous peer effect. By letting $\tilde A = ( A_{ij}/ \sum_{j=1}^nA_{ij} )_{i,j \in [n]}$ denote the row-normalized adjacency matrix, we can rewrite the potential outcomes as
$$ Y_i(z) = z_i\left\{ \gamma + \gamma\beta \sum_{k=0}^\infty \beta^k (\tilde A^{k+1} e_i)_i \right\} + \frac{\alpha}{1-\beta}+  \gamma\beta \sum_{k=0}^\infty \beta^k (\tilde A^{k+1} z_{-i})_i + \sum_{k=0}^\infty \beta^k (\tilde A^k \xi)_i $$
where $e_i$ is a vector of zeros with a one in the $i$th components and $z_{-i}=z-e_iz_i$ is the vector $z$ with the $i$th component set to zero. In this case, we would have $\gamma(G_i^z)=( (\tilde A z_{-i})_i, (\tilde A^{2} z_{-i})_i, \ldots)$ and $\gamma_0(G_i^z(m)) = ((\tilde A z_{-i})_i, \ldots, (\tilde A^{m+1} z_{-i})_i)$. In this model, we would see correlation across units of the random shocks.\footnote{Tools of the literature would allow us to extend our results to this case but to highlight our contributions on the adaptive selection of the neighborhoods and its consequence we impose independence of shocks.} 
\end{example}
From an interpretation and algorithmic point of view, it is convenient to represent this interference function as a tree. Each local pattern $\gamma_0(G_i^Z(m))$ (built from $m$-hops from unit $i$) corresponds to a node of the tree that is at depth $m$ from the root. The root corresponds to zero hops (e.g., no interference) and the leaves correspond to a relevant interference pattern for the function. (We will associate an interference estimate to each node of the tree, see Section \ref{sec:InterferenceEstimator}). Figure \ref{fig:figure1} considers an example where $\gamma_0$ is the sum of treated nodes in the relevant local subnetwork $G_i^Z(m)$. The left panel illustrates an interference function that depends only on its neighbors. Figure \ref{fig:figure1} on the right panel illustrates an interference function that depends on the ``neighbors of the neighbors" if no (immediate) neighbors were treated. 

Thus, this assumption is related to the Neighborhood Interference assumption in the causal inference literature (see, e.g., Assumption 2 in \citet{forastiere2021identification}). The function $\gamma_0$ defines an interference pattern that is similar to the ``exposure mapping'' discussed in \citet{aronow2017estimating}. 
% The $m$-NI assumption discussed here makes the dependence of the radius $m_i$ to be a function of the treatments of other units as well. 
As we discuss below, when we estimate $m_i$ we are concerned with approximating the corresponding interference, and we will balance a bias and variance trade off of the interference (rather than aiming for a perfect recovery of $m_i$ that would rely on strong separation assumptions).  

\vspace{-0.5cm}

\begin{figure}[H]
  \centering
  \begin{subfigure}[bt]{0.45\textwidth}
    \begin{tikzpicture}
    
    \node[align=left] at (4,7) {$*$};
    \draw[thick] (4,7) circle (0.4cm);
    
    \draw[thick] (3,6.2) -- (4,6.6); 
    \draw[thick] (5,6.2) -- (4,6.6);
    
    \node[align=left] at (3,5.8) {$0$};
    \node[align=left] at (5,5.8) {$\geq 1$};
    
    \draw[thick] (3,5.8) circle (0.4cm);
    \draw[thick] (5,5.8) circle (0.4cm);
    
    %\draw[thick,opacity=0] (4,4.7) -- (4,5.7); 
\end{tikzpicture}
  \end{subfigure}
  \begin{subfigure}[bt]{0.45\textwidth}
    \begin{tikzpicture}
    
    \node[align=left] at (4,7) {$*$};
    \draw[thick] (4,7) circle (0.4cm);
    
    \draw[thick] (3,6.2) -- (4,6.6); 
    \draw[thick] (5,6.2) -- (4,6.6);
    
    \node[align=left] at (3,5.8) {$0$};
    \node[align=left] at (5,5.8) {$\geq 1$};
    
    \draw[thick] (3,5.8) circle (0.4cm);
    \draw[thick] (5,5.8) circle (0.4cm);

\node[align=left] at (2,4.6) {$0$};
    \node[align=left] at (4,4.6) {$\geq 1$};

    \draw[thick] (3,5.4) -- (4,5); 
    \draw[thick] (3,5.4) -- (2,5);
    
    \draw[thick] (2,4.6) circle (0.4cm);
    \draw[thick] (4,4.6) circle (0.4cm);

\end{tikzpicture}
  \end{subfigure}
    \caption{The depth of the tree corresponds to the number of hops in the network. The leaves of the context tree correspond to the relevant interference patters (each will be allowed to have its own interference intensity). The panel on the left represents one interference function that depends only if at least one immediate neighbor was treated or not. The panel on the right represents an interference function that when no friends were treated, it depends if any friends of friends were treated or not. Note that a tree with only a root node represents a setting where there is no interference.}
    \label{fig:figure1}
\end{figure}

\vspace{-0.5cm}

The combination of Assumptions \ref{additivity}-\ref{Neigh_inter}  implies the  the following partially linear form for the potential outcomes: 
\begin{equation}\label{def:outcomevalues} \begin{array}{rl}
\widetilde Y_i(z) & = z_i\{\widetilde Y_i(z_i,z_{-i}) - \widetilde Y_i(0,z_{-i})\} + \widetilde Y_i(0,z_{-i})  %\\ - \widetilde Y_i(0,0) + \widetilde Y_i(0,0) \\
%& = z_i\{\widetilde Y_i(z_i,z_{-i}) - \widetilde Y_i(0,z_{-i})\} + \{ \widetilde Y_i(0,z_{-i}) - \widetilde Y_i(0,0) + \alpha\} + \{\widetilde Y_i(0,0) - \alpha\} \\
%& 
= z_i \tau_i + f( \gamma(G_i^z) ) + \epsilon_i
\end{array}
\end{equation}
%$$ \widetilde Y_i(z) = z_i \tau_i + f( \gamma(G_i^z) ) + \epsilon_i $$ 
where $\epsilon_i$ is a zero mean term (heteroskedastic), and $f(\gamma(G_i^z))$ is the interference term that absorbs the (constant) baseline, see, e.g., Proposition 1 in \citet{awan2020almost} and \citet{sussman2017elements} for similar development. Moreover the individual specific $\tau_i:=\mathbb{E}\left[ \widetilde Y_i(1,0_{-i}) - \widetilde Y_i(0,0_{-i}) \mid X_i \right]$ can be dependent of $X_i$. 
% The key restriction is to rule out  interactions.  \volf{
There is a complication if the shocks $\epsilon_i$'s are allowed to be dependent as differentiating between $f(\cdot)$ and the error term would not necessarily be possible \citep{shalizi2011homophily}. We will make the assumption that the errors are uncorrelated. First, it is plausible that there is no such excess homophily in the problem (such as when connections are enforced exogenously), and so the no-correlation assumption is reasonable. Furthermore, recent work by \citet{mcfowland2021estimating} suggests that the structure on the covariance of $\epsilon_i$ still allows us to identify $f(\cdot)$ and the error separately. While knowing the structure a priori is a slight generalization to the no-correlation setting, it substantially complicates the derivations without contributing to the understanding of the proposed method and so it is left to future work.
%}

The main estimand of interest we consider is the average direct treatment effect on the treated, i.e.,
\begin{equation}\label{ATT}
\tau := \frac{1}{\mbox{$\sum_{i=1}^n Z_i$}} \sum_{i=1}^n Z_i \mathbb{E}\left[ \widetilde Y_i(1,0_{-i}) - \widetilde Y_i(0,0_{-i}) \mid X_i \right] = \frac{1}{\mbox{$\sum_{i=1}^n Z_i$}} \sum_{i=1}^n Z_i\tau_i.
\end{equation}

Related to (\ref{ATT}) we can define the  average direct effect on the treated (ADET) as the average over the treatment effect (\ref{ATT}) over the treatment assignments \begin{equation}\label{ATT-norm}
\tau_{ADET} := \mathbb{E}\left[ \sum_{i=1}^n \frac{Z_i}{\mbox{$\sum_{j=1}^n Z_i$}} \mathbb{E}\left[ \widetilde Y_i(1,0_{-i}) - \widetilde Y_i(0,0_{-i}) \mid X_i \right] \right]
\end{equation}where the outermost expectation is over the treatment assignments. We provide results for this case as well.

\section{Estimators}
\label{Proce_and_Est}
%\label{Est}

In this section, we propose estimators for the average direct treatment effect on the treated that leverage the decomposition (\ref{def:outcomevalues}). These estimators build upon the following high-level ideas: (i) under the partially linear assumption, we can obtain estimates $\hat f_i$ of the interference $f_i:=f(\gamma(G_i^Z))$ based on the control units; (ii) subtracting the estimates of the interference from the observed outcomes, $Y_i-\hat f_i$, up to estimation errors, we recover a setting with no interference for which estimators are available. However, estimation errors here include sampling errors, regularization and model selection biases, and errors from learning the features themselves.

To estimate the interference function, we propose a data-driven choice of a unit-specific value of $\hat m_i$-hops that approximates the relevant neighborhood for the interference function. In this setting, up to logarithmic terms, we establish that the proposed estimate balances bias and variance. \footnote{It is based on a Lepski-type procedure to adaptively select a node specific value $m_i$-hops within the control group.} 

Based on interference estimates, we can construct the triple $(Y_i-\hat f_i,Z_i,X_i)$ which can be used to construct different estimators for the direct effect. Under the assumption that the treatment assignments are independent conditional on the covariates, we can leverage the propensity score to obtain another estimator: Letting $\hat e_i$ denote an estimate for the propensity score, the {\it augmented inverse propensity weighted} (AIPW) estimator in our setting is given by
\begin{equation}\label{def:tauDR}
    \begin{split}
      \hat{\tau}^{AIPW}= \frac{1}{\mbox{$\sum_{i=1}^n Z_i$}}\sum_{i=1}^n \left[ Z_i ( Y_i -  \hat{f}_i)  -  \frac{(Y_i-\hat f_i)(1-Z_i)\hat e_i}{1-\hat e_i} \right]   \end{split}
\end{equation}
This AIPW estimator inherits several good properties that are well known when SUTVA holds. Nevertheless, the construction of the interference function plays a central role in the analysis of these estimators (in particular the impact of the feature engineering step). 

Although our main focus is on the estimator (\ref{def:tauDR}), under the additive assumption, despite interference, we can estimate the ADTT without the assumption of independent treatment assignment assumption using an {\it outcome regression} (OR) estimator for (\ref{ATT}), namely
\begin{equation}\label{def:tauOR}
    \begin{split}
      \hat{\tau}^{OR}= \frac{1}{\mbox{$\sum_{i=1}^n Z_i$}} \sum_{i=1}^n Z_i \left( Y_i - \hat{f}_i \right) 
    \end{split}
\end{equation}
We will provide rates of convergence based on oracle inequalities for the interference estimator. %As the latter is not not a root-$n$ estimator,Since  but We do not expect this estimator to be root-$n$ consistent. 
%as the counterfactual for no treatment with no interference would have zero expected value following (\ref{def:outcomevalues}). 
As described in the following, we will rely on different interference estimators in (\ref{def:tauOR}) and (\ref{def:tauDR}). The latter has a slower than root-$n$ rate of convergence which is a consequence of the interference estimator. For the former, we need to modify the interference estimator with the use of synthetic treatments to cope with the dependence in our estimation to achieve asymptotic normality for (\ref{def:tauDR}). 

\subsection{Interference Estimator}\label{sec:InterferenceEstimator}

A key step in our approach is to estimate the interference function, which also needs to identify the relevant interference patterns. Our proposal will estimate a relevant $m$-hop neighborhood associated with each $G_i^z$ that achieves a near-optimal bias and variance tradeoff. More precisely, it is a Lepski-type procedure that aims to balance (estimated) bias and variance when choosing $m$. The estimation will construct a tree representation for the relevant patterns that are relevant to approximate interference as discussed in Figure \ref{fig:figure1}. The estimation of such tree can be viewed as selecting the appropriate subtree out of the tree that contains all possible interference patterns.

In what follows, we use the following notation. Let $G_i^Z(m):=$ $m$-hop subgraph from $i$ in $G^Z$. In order to define a proper partition of the nodes, we match along all intermediate hops as well, that is, if $\gamma_0(G_i^Z(m)) = \gamma_0(G_j^Z(m))$, it implies that $\gamma_0(G_i^Z(k)) = \gamma_0(G_j^Z(k))$ for $k=1,\ldots,m$. For $g=\gamma_0(G_i^Z(m))$, we denote $par(g) = \gamma_0(G_i^Z(m-1))$ and say that $g'\succeq g$ if $\exists m' \geq m$ is such that $g'=\gamma_0(G_i^Z(m'))$.

 ~\\ 
\uline{{\bf  Algorithm 1. Estimation of interference function}\hfill}
% \hline
% \noindent\rule{\textwidth}{1pt}\textbf{}
~\\
\noindent {\bf Step 0.} Initialize $\lambda > 1$, confidence level $1-\delta$, variance upper bound $\bar{\sigma}^2$,  
$$ \widehat E_n:= \{ g : \exists i,m \in[n], Z_i=0, g=\gamma_0(G_i^Z(m)) \} \ \ \mbox{and} \ \ \widehat{\mathcal{K}}:=\emptyset $$
\noindent {\bf Step 1.} For $g \in \widehat E_n$, let $V_g := \{ i \in [n] : Z_i=0, \exists m \in [n] \mbox{ s.t.} \  g=\gamma_0( G^Z_i(m) )\} $ and define 
$$ \mbox{$\hat f(g):= \frac{1}{|V_g|}\sum_{i\in V_g} Y_i$} $$ 
%\indent \indent \indent where $V_g := \{ i \in [n] : Z_i=0, \exists m \in [n] \mbox{ s.t.} \  g=\gamma_0( G^Z_i(m) )\} $.\\
\noindent {\bf Step 2.} For each $g \in \widehat E_n$, set $\widehat{\mathcal{K}} \leftarrow \widehat{\mathcal{K}} \cup \{g\} $ if $ \exists g',g'' \in \widehat E_n : g'\succeq g, g'' \succeq par(g) $ s.t.
$$ \mbox{$ |\hat f(g') - \hat f(g'')| > \lambda \bar\sigma\sqrt{2\log(2n^2/\delta)}\left(\frac{1}{\sqrt{|V_{g'}|}} + \frac{1}{\sqrt{|V_{g''}|}}\right)$}$$%\lambda  \alpha(g')+\lambda \alpha(g'')   $$
\noindent {\bf Step 3.} To estimate $f(\gamma(G_i^Z))$, $i\in[n]$,  we define
\vspace{-0.35cm}
$$ \begin{array}{rl}
\ \ \ \ \ \ \ \ \ \  \hat f_i= \hat f (\hat g_i ) \ \ \ & \mbox{where} \ \hat g_i := \gamma_0( G^Z_i(\hat m_i)) \ \ \mbox{and} \ \ \hat m_i := \max_m \{ m : \gamma_0(G^Z_i(m)) \in \widehat{\mathcal{K}} \} \end{array}$$ 
\vspace{-1.5 cm}
~\\

\noindent\uline{\noindent { } \hfill}
~\\

Estimation of interference is challenging as we need to estimate the relevant $m$-hop neighborhood that provides a good approximation for the interference function. Algorithm 1 is based on the idea that, in the control group, the interference function is the signal that we are trying to recover.  In its Step 0 the set $\widehat E_n$ denotes the set of different configurations of relevant patterns in the given labeled graph. Because we have a nested subset of graphs as we increase $m$ in $G_i^Z(m)$, we note that the set $\widehat E_n$ can be represented as a tree starting at the root with $m=0$. 
The set $\widehat{\mathcal{K}}$ is a subtree of $\widehat{E}_n$ that will collect the patterns selected to estimate the interference function. 

For each pattern $g \in \widehat E_n$, Step 1 defines an estimate of the interference $\hat f(g)$.\footnote{Note that this is generalizable for more general estimands beyond conditional mean and with covariates.} Step 2 of Algorithm 1 balances the estimates of bias and variance. Informally, a pattern $g$ in $\widehat{E}_n$ should be selected if there are two deeper patterns that have sufficiently different interference estimates relative to their (square root) variance. That is, a pattern $g$ will be selected into $\widehat{\mathcal{K}}$ only if it is needed to keep the bias small relative to an estimate of the variance. 
The parameter $\lambda$ is a regularization parameter in the spirit of shrinkage estimators, e.g., lasso \cite{bickel2009simultaneous}.
Finally, for each $i$, Step 3 defines the relevant value of hops $\hat m_i$ as the largest value of $m$ that is still in the estimated tree $\widehat{K}$.  In the next section, we provide theoretical guarantees on the performance of the estimator. We note that unless one is willing to impose strong separation assumptions, it is unlikely that one can estimate $m$ uniformly across different data generating processes. However, the misspecification cannot be much larger due to the selection criterion in Algorithm 1. 

There are several variants of Algorithm 1 that relax some of the requirements. For example, it is possible to use a pattern-specific estimate of the variance $\hat\sigma_g^2$ (instead of assuming a known $\bar \sigma^2$) and our results (e.g., Theorem \ref{thm:single_node}) would hold with small modifications. Moreover, the proposed Step 2 that balances bias and variance relies on concentration inequalities (i.e. through the factor $\sqrt{2\log(2n^2/\delta)}$) which can be conservative. It is also possible to rely on bootstrap approximations to take into account correlations to reduce the penalty levels and still achieve the same theoretical guarantees, see, e.g., \cite{belloni2018high}. Finally, in Section \ref{sec:Additional}, we discuss different ways to incorporate covariates which would allow us to relax Assumption \ref{Assump:NoDirectImpact}.

\subsection{Decoupling feature engineering based on Synthetic Treatments}

Here, we propose a variant that changes the definition of the initial tree $E_n$ that will be critical to establishing distributional limits for the AIPW estimator. To see why, we note that Algorithm 1 can be viewed through the lens of feature engineering. In fact, it uses the data to define the patterns to construct the relevant feature (that is, the appropriate $k$ hop neighborhood of the labeled graph). Although such dependence does not affect the rate of convergence of the interference estimator, it is critical for the distributional results of the AIPW estimator.

The next algorithm will be useful when we want to decouple feature generation from treatment assignments. We will generate a different initial set of patterns to replace $\widehat E_n$ as follows.

~\\
\uline{{\bf Algorithm 2. Estimation of interference function with synthetic treatments}\hfill}
% \hline \\
~\\
\noindent {\bf Step 1.} Consider the original (unlabeled) interference graph $G$.\\
\noindent {\bf Step 2.} Assign (new) labels $\widetilde Z^o \in \{0,1\}^n$ independently from $Y$ and $Z$.\\
\noindent {\bf Step 3.} Define $ \displaystyle \widetilde E_n^o := \{ g : \exists i, m \in [n], g=\gamma_0(G_i^{\widetilde{Z}^o}(m))\} $ \\
\uline{\noindent {\bf Step 4.} Run Steps 1, 2 and 3 from Algorithm 1 with $ \widetilde E_n^o$ instead of $\widehat E_n$. \hfill}\\
%\uline{{ }\hfill}
~\\
This decoupling step reduces overfitting in a similar spirit to sample splitting. Nevertheless, the use of synthetic treatments avoids the difficulties that sample splitting would create due to graph dependence across units. The proposal exploits that we only need to create (synthetic) treatments to achieve independence between pruning and feature generation. We acknowledge that, in principle, this decoupling might miss some potential feature that appeared in the true data. However, for a pattern to be informative, it needs to be seen across multiple nodes, and those are likely to appear if we could resample the data.  In practice, the label $\widetilde{Z}^0_i$ can be generated from independent Bernoulli random variables with parameters $\hat{e}(X_i)$ for $i \in [n]$. It is also possible to use a small subset of nodes to obtain a rough estimate of the propensity score.\footnote{Under Assumption \ref{assum:DR}(ii) our main results would also hold if we use a small number of observations $q$, $q \ll n^{1/2}/\{(1+D_n) \log n\}$, to learn a rough estimate for $e(X)=\mathbb{P}(Z=1\mid X)$.}  

\section{Theoretical Results}
\label{The_Res}

In this section, we establish theoretical guarantees associated with the two causal estimators (\ref{def:tauOR}) and (\ref{def:tauDR}). We also establish the properties of the interference estimator which is key for our proposal. 

Our setting follows the Additive effect under Neighborhood Interference Assumption (ANIA) framework in \cite{sussman2017elements} as the potential outcomes satisfy additivity of main effects if and only if the potential outcomes can be parameterized through a partially linear model. The next assumption is a summary of assumptions \ref{additivity}, \ref{Assump:NoDirectImpact} and \ref{Neigh_inter}  (presented to facilitate readability) with additional moment conditions on the error term.

\begin{assumption}[Potential Outcome Model]
\label{assump:outc_mod} The potential outcome equation satisfies
\begin{equation*}
    \begin{split}
        \widetilde Y_i(z_i,z_{-i})=  z_i \tau_i + f(\gamma(G_i^z))+ \epsilon_i
    \end{split}
\end{equation*}
where, conditional on $X_i,Z_i$, the error term $\epsilon_i$ is: (i) independent of $(Z_{-i},X_{-i},\epsilon_{-i})$, (ii) a zero mean subgaussian random variable with parameter $\bar\sigma^2$, and (iii) $\mathbb{E}[\epsilon_i^2 \mid Z_i,X_i]\geq \underline{\sigma}^2$. Furthermore, we assume that $|\tau_i| + |f(\gamma(G_i^z))| \leq \bar C$ for all $z \in \{0,1\}^n$, $i\in [n]$.
\end{assumption}

Assumption \ref{assump:outc_mod} imposes (conditional) independence on the error terms and subgaussianity. It is a technical assumption that can be relaxed in various ways, but it would take the focus out of the main issues the work is focusing on (e.g., finite moments and limited dependence across units). Assumption \ref{assump:outc_mod}  also imposes that the individual treatment effects and the interference function are uniformly bounded. 

Next, we define the maximum approximation error associated with the $m$-hop subgraphs. 

\begin{definition}[Interference approximation error]
\label{bias_term}
For a given pattern $g$ we define $\overline{V}_g := \{ i \in [n] : \exists m \in [n] \mbox{ s.t.} \  g=\gamma_0( G^z_i(m) )\} $ and the approximation error at  $g$ as 
\begin{equation*}
\begin{array}{rl}
r^z(g):= &\displaystyle \max_{ i,i' \in \overline{V}_g} \left\lvert \mathbb{E}[ \widetilde Y_i(0,z_{-i})-\widetilde Y_{i'}(0,z_{-i'})]\right\rvert
=\max_{ i,i' \in \overline{V}_g} \left\lvert  f (\gamma(G_i^z))- f(\gamma(G_{i'}^z))\right\rvert
\end{array}
\end{equation*}
%For convenience we omit the dependence on $Z$ and denote $r(g)=r(g,Z)$. %where $k \leq \min(a,b)$. 
\end{definition}
The approximation is defined conditional on a specific vector of treatment assignment $z$. This term provides a quantification of the misspecification one obtains by restricting the model at a specific number of hops for a specific pattern. We note that the set $\bar V_g$ includes both controls and treated units while its subset $V_g$ (used in Algorithm 1) includes only control units. The approximation error $r^Z(g)$ will be used to establish an adaptive bound for our interference estimator. A common assumption in the literature is that controlling for the $1$-hop neighborhood suffices for no approximation errors. There are a few exceptions in the literature that allow for misspecification. The closest to our setting is \cite{leung2022causal} which considers an approximation $m \to \infty$ so that $r(g_m)\to 0$ and derives asymptotic normality for the inverse propensity weighted estimator under appropriate dependence decay that covers several cases of interest. In contrast, we exploit the heterogeneous radius $m_i=m(G_i^z)$ across different units that depends on the network and the treatment assignments to control the approximation error. This is of interest in obtaining more efficient estimates of the interference. However, this approach creates a model selection problem as we try to balance bias and variance.

\subsection{Oracle Inequality for Interference Function Estimator}\label{sec:OracleInterference}

In this section, we establish an oracle inequality for the estimator of the interference function defined by Algorithm 1. The interference function is defined on all labeled graphs and our objective is to obtain a uniform guarantee across all values $f_i := f(\gamma(G_i^Z))$, $i\in [n]$. We obtain an oracle inequality that shows the proposed estimator $\hat f_i$ that is associated with a specific $\hat m_i$-hop subgraph, achieves a similar performance to the best (unknown) choice of $m$ that optimally balances the unobserved bias (measured by $r(g_m)$) and the variance of $\hat f_{g_m}$.  

\begin{theorem}
\label{thm:single_node}
Suppose that Assumption \ref{assump:outc_mod} holds. Then, with probability at least $1-\delta$, uniformly over all nodes $i\in [n]$ we have   
% $Tr_g  \in Tr_g^*= \{ Tr_g:  \tilde{Tr}_g \in g(G^z_W) \}$, 
\begin{equation*}
\begin{split}
\lvert f_i - \hat f_i \rvert \leq \min_{m \in [n], g_m=\gamma_0(G_i^Z(m))} \frac{\lambda}{\lambda-1} r^Z(g_m)+ 3\lambda \bar \sigma \sqrt{\frac{2\log(2n^2/\delta)}{|V_{g_m}|}}
\end{split}
\end{equation*}
where  $V_g := \{ i \in [n] : Z_i=0, \exists m \in [n] \mbox{ s.t.} \  g=\gamma_0( G^Z_i(m) )\}$ as in Step 1 of Algorithm 1.
\end{theorem}

Theorem \ref{thm:single_node} establishes a performance guarantee for the estimation of interference in all units. The interference estimate in Algorithm 1 is based only on the control units, thus the effective sample size is $|V_{g_k}|$ (instead of $|\bar V_{g_k}|$). Importantly, the bound holds uniformly for all units. Crucially, the analysis accounts for the fact that $G_i^Z$ is random (Z dependent), which in principle could lead to $n^22^n$ different possible $G_i^z(k)$, $i,k\in [n]$. The result in Theorem \ref{thm:single_node} is adaptive in the sense that for patterns $g_k$ that are observed more often, faster rates are obtained, adapting the choice of $k$ for each $i\in [n]$.  To provide explicit rates of convergence, we discuss two examples based on two-dimensional regular lattices. 

\begin{example}
	\label{example_threshold_intf}
	Suppose that $G$ is a regular two-dimensional lattice with $n$ nodes and any internal unit has degree $d=4$. Suppose that there is $m_0$ such that $r(g_k)=0$ if $k\geq m_0$ and the treatment is randomly assigned with probability $1/2$. Then with probability at least $1-\delta - n\exp \left(- \frac{n}{2d^2(m_0+1)^4 \cdot 2^{d(m_0+1)^2} } \right)$, we have  %$|V_{g_{m_0}}|\geq n \cdot \left( 1/2 \right)^{2+2(m_0+1)m_0}$ uniformly across different units and  
		\begin{equation*}
		\begin{split}
		& \max_{i\in [n]} |f(G^Z_{i})-\hat{f}_i|\leq C\lambda \bar{\sigma} 2^{\frac{d}{4}(m_0+1)m_0} \cdot \sqrt{ \frac{\log (2 n^2/\delta)}{n}} .  \\
		\end{split}
		\end{equation*}
\end{example}

\begin{example}
	\label{example_exp_decay_intf}
	Suppose that $G$ is a regular two-dimensional lattice with $n$ nodes, any internal unit has degree $d=4$ and $\bar{\sigma}=1$ in Assumption \ref{assump:outc_mod}. Suppose that the approximation exhibits sub-Gaussian decay, namely $r^z(g_k)\leq C' \exp(-4k^2)$ for any $z\in \{0,1\}^n$, and the treatment is randomly assigned with probability $1/2$. 
	%Furthermore, we set\begin{footnotesize}
	% $a:=O\left(\sqrt{\log \left(  n^{1/8} (\log n)^{1/2}   \right) } \right)$, $b:=O\left(\sqrt{ \log\left(\frac{n^{1/d}}{2^{1/d}(\log n)^{1/d}}\right)} \right)$ \end{footnotesize} and 
	%\begin{footnotesize} $c:=O\left(\sqrt{\log\left(\frac{n^{1/d}}{(\log (\frac{2n^2}{\delta}))^{2/d}(\log n)^{4/d}  } \right) } \right)$.  
	%\end{footnotesize}
	%Then by taking the patterns $g_{i,k}$ for $i \in [n]$ where $a \leq k \leq \min(b,c)$. 
 	%$\hat{f}_i$ for $i\in [n]$ are obtained from Algorithm $1$. 
 	Then when $n\geq n_0$ where $n_0$ is a sufficiently large constant, with probability at least $1-\frac{2}{n}$, %where $\delta$ can take any power of $\frac{1}{n}$, i.e., $\delta=\frac{1}{n^\alpha}$ for some $\alpha>0$ and $n$ satisfies
 	we have 
 	%that $|V_{g_k}|\geq n \cdot \left( \frac{1}{2} \right)^{2+2(k+1)k}$ and 
	$$\max_{i\in[n]} |f(G_i^Z)- \hat f_i |  \leq   \frac{\lambda C^\prime \bar{\sigma}^2}{\lambda-1} \cdot \frac{1}{n^{1/4}\log n}.$$
%when $n$ is sufficiently large. 
\end{example}

We close this section with an analogous result for the interference estimator based on synthetic treatments proposed in Algorithm 2. The following oracle inequality is a direct corollary of Theorem \ref{thm:single_node}. 

\begin{corollary}[Interference Estimation with Synthetic Treatment]
\label{cor:second_context_tree} 
Suppose that Assumption \ref{assump:outc_mod} holds and that Step 0 in Algorithm 1 is replaced by Algorithm 2 (i.e., we set $\widehat{E}_n = \widetilde E_n^o$). Then, with probability at least $1-\delta$, uniformly over all nodes $i\in [n]$ we have   
% $Tr_g  \in Tr_g^*= \{ Tr_g:  \tilde{Tr}_g \in g(G^z_W) \}$, 
\begin{equation*}
\begin{split}
\left\lvert f(\gamma(G_i^Z)) - \hat f_i \right\rvert \leq \min_{k \in [n], g_k=\gamma_0(G_i^{\widetilde{Z}^o}(k))} \frac{\lambda}{\lambda-1} r^{\widetilde{Z}^o}(g_k)+ 3\lambda \bar \sigma \sqrt{\frac{2\log(2n^2/\delta)}{|V_{g_k}|}}
\end{split}
\end{equation*}
where $V_g := \{ i \in [n] : Z_i=0, \exists m \in [n] \mbox{ s.t.} \  g=\gamma_0( G^{Z}_i(m) )\} $.
\end{corollary}

Corollary \ref{cor:second_context_tree} has the patterns generated with $\widetilde{E}_n^o$ instead of $\widehat{E}_n$ of Algorithm 1. By resampling the treatment assignments, Algorithm 2 decouples the generation of features from the pruning of the tree. The pruning steps (i.e., Steps 1 and 2) are based on the original assignments and outcomes (since we do not have synthetic outcomes). Corollary \ref{cor:second_context_tree} does not provide improvements in the oracle inequality of Theorem \ref{thm:single_node}. The advantage of using the estimator based on Algorithm 2 is to reduce dependence within the error and feature engineering that appear in some estimators of the ADTT.  

\subsection{Distributional results for the AIPW estimator}\label{sec:MainATT}

Next, we proceed to construct guarantees for the AIPW estimator for the ADTT. Recall that we apply Algorithm 2 to construct the initial set of patterns for Algorithm 1 (that is, to replace Step 0). We use the resulting estimator of the interference function in (\ref{def:tauDR}). Our analysis will leverage the independence of the assignment of treatment between units.  Next, we state our conditions for the estimation of nuisance functions and their estimability.

\begin{assumption}\label{assum:DR} Let $\delta_n (\geq n^{-2}), n\geq 1,$ be a fixed sequence converging to zero. Suppose\\
(i) the propensity score estimator $\hat e_i:=\hat e(X_i) \leq 1-\tilde c$ belongs to a function class $\mathcal{E}$ that is a VC-type class with $\varepsilon$-covering number bounded by $(A/\varepsilon)^{\bar v}$, such that with probability $1-\Delta_n$ we have
\mbox{$\frac{1}{n}\sum_{i=1}^n\left(\frac{e_i}{1-e_i}-\frac{\hat e_i}{1-\hat e_i}\right)^2 \leq  C \frac{\bar v\log n}{n} \leq C\frac{\delta_n^2}{\bar v \log n} $}

\noindent (ii) With probability $1-\Delta_n$ there is an oracle trimming $ \mathcal{K}^* \subseteq \widetilde E_n^o$  such that associated estimator $\hat f^* $ satisfies 
$$ % \max_{i \in [n]} \mathbb{E}[ (f_i - \hat f_i^*)^2 \mid Z ] \leq \bar \sigma^2 \delta_n^2 n^{-1/2}\log^{-2}(n),  \ \ \   
\max_{g\in \mathcal{K}^*} r^Z(g) \leq \bar \sigma \delta_n n^{-1/2} \ \ \ \mbox{and} \ \   \max_{i \in [n]} \mathbb{E}[ (f_i - \hat f_i^*)^2 \mid Z ]   \leq  \bar \sigma^2 \frac{\delta_n^2 n^{-1/2} \log^{-1} n}{1+D_n}$$ 
where  $D_n = \max_{i\in [n]} | \{ k \in [n] : i \ \ \mbox{is a vertex in } G_k^z(m_k^*) \} | $ 
where $m_k^*$ is the number of hops for the $k$th node associated oracle trimming $\mathcal{K}^*$. 
\end{assumption} 
%% $\gamma^*$ is balancing bias and variance 
%%
Assumption \ref{assum:DR}(i) places mild conditions on the propensity score estimator. Because treatments are assumed to be assigned independently conditioned on observables, the estimation of the propensity score is standard.  Assumption \ref{assum:DR}(ii) is a condition on the randomization of the assignments relative to the interference patterns. It assumes the existence of (an unknown and possibly random oracle) $\mathcal{K}^*$ within $\widetilde E_n^o$ that produces a good approximation error and the associated estimator $\hat f^*$ has a minimum convergence rate to estimate the interference function (that is, slightly faster than $n^{-1/4}$). The quantity $D_n$ bounds the number of egos whose neighborhoods are affected by a single node (i.e., bounds the number of changes in $\{\hat f_j\}_{j\in [n]}$ if one treatment is changed). Although it is plausible for $D_n$ to be bounded (e.g., when $m_i^*$ is uniformly bounded and nodes have bounded degree), it is allowed to grow with the sample size. The results are with high probability on the realization of the treatment assignments ($Z$ and $\widetilde Z$). Thus our condition (and analysis) allows for $\mathcal{K}^*$ to depend on the treatment assignments. 

We show how these conditions can be used with the orthogonal moment condition implied by the definition of (\ref{def:tauDR}). The analysis has new features to deal with the dependence due to interference, and the feature engineering of the proposed approach. We note that as in other context for high dimensional causal inference \citep{belloni2014inference,chernozhukov2017double}, sample-splitting can be quite helpful in reducing overfitting. However, attempting a sample splitting approach is more complicated due to the network dependence which is also being estimated. We bypass this issue by proposing Algorithm 2 that 
% provides a way to 
helps
achieve independence from the feature engineering process and the interference estimation. 

Next we establish the main result of this section for the estimation of the ADTT. The asymptotic normality result is obtained by relating our proposed AIPW estimator $\hat\tau^{AIPW}$  in (\ref{def:tauDR})  to the (unfeasible) estimator  $\hat\tau^{AIPW*}$ that knows the true interference function and the propensity score function, i.e., defined as in (\ref{def:tauDR}) with $\hat f_i=f_i$ and $\hat e_i = e_i$, $i\in [n]$.
\begin{theorem}
\label{thm:dra_tau_inf}
Under Assumptions \ref{cond_unconf}, \ref{ass:cond_ind}, \ref{assump:outc_mod} and \ref{assum:DR},  with probability at least $1-\delta-\Delta_n$ we have that $|\hat\tau^{AIPW} - \hat\tau^{AIPW*}| \leq C\bar\sigma \delta_n n^{-1/2}$ so that 
$$\mbox{$\sqrt{\sum_{i=1}^nZ_i}$} \ \  \Sigma_n^{-1/2}(\hat{\tau}^{AIPW}-\tau) \to N(0,1)$$
where  $\Sigma_n := \frac{1}{\sum_{i=1}^nZ_i}\sum_{i=1}^n \mathbb{E}\left[\{ Z_i(Y_i-\tau_i -f_i) - (Y_i-f_i)(1-Z_i)e_i/(1-e_i)\}^2 \mid Z_i,X_i \right]$. 
\end{theorem}

Theorem \ref{thm:dra_tau_inf} establishes the first-order equivalence between the proposed estimator $\hat\tau^{AIPW}$ and the infeasible AIPW estimator that knows the true interference function.  In contrast to the outcome regression estimator, the AIPW estimator explicitly exploits the independence of the assignment of treatment between units. This allows us to obtain the so-called Neyman orthogonality and reduce the first-order bias. The analysis further takes into account model selection and feature engineering. This relies on the derivation of sharp bounds conditioned on observing the neighborhood patterns. This is needed to enhance the adaptivity of the approach that does not require the convergence to a deterministic set of patterns.  The variance $\Sigma_n$ is needed for the construction of confidence regions. Corollary \ref{cor:var} provides guidance on its estimation.

\begin{corollary}\label{cor:var}
Suppose that the conditions of Theorem \ref{thm:dra_tau_inf} hold and that $\tau_i=\tau(X_i)$. Then we have that confidence intervals based on
$$ \widetilde{\Sigma}_n := \frac{1}{\mbox{$\sum_{i=1}^nZ_i$}}\sum_{i=1}^n  \{ Z_i(Y_i-\hat \tau^{AIPW}-\hat f_i) - (Y_i-\hat f_i)(1-Z_i)\hat e_i/(1-\hat e_i)\}^2$$
are asymptotically conservative, namely $\Sigma_n \leq \widetilde{\Sigma}_n + o(1)$. Moreover, the result of Theorem \ref{thm:dra_tau_inf} continues to hold with a plug-in estimate $$\widehat  \Sigma_n := \frac{1}{\mbox{$\sum_{i=1}^nZ_i$}}\sum_{i=1}^n  \{ Z_i(Y_i-\hat \tau_i -\hat f_i) - (Y_i-\hat f_i)(1-Z_i)\hat e_i/(1-\hat e_i)\}^2  $$ where $\hat \tau_i=\hat\tau(X_i)$ is an estimator of $\tau_i=\tau(X_i)$, such that $\frac{1}{\sum_{i=1}^nZ_i}\sum_{i=1}^n Z_i(\hat\tau_i-\tau_i)^2 \leq \delta_n^2$ with probability $1-o(1)$.
\end{corollary}

Without additional assumptions on the shocks and on the individual treatment effects we can still construct conservative confidence regions based on $\widetilde{\Sigma}_n$. Under the condition that the individual treatment effects can be approximated by an estimable function of the covariates, we can obtain consistent estimates for $\Sigma_n$ and obtain asymptotically correct coverage. In practice, we recommend to pursue the estimation of $\tau_i=\tau(X_i)=\mathbb{E}[Y_i-f_i\mid X_i]$ using the ``data'' $\{(Y_i-\hat f_i, X_i)\}_{\{i \in [n]:Z_i=1\}}$ by a predictive machine learning model, e.g., random forest. Since only a slow rate of convergence is needed, the condition on the estimators $\hat \tau_i$ in Corollary \ref{cor:var} is plausible despite of the use of (consistent) estimates $\hat f_i$ of the interference function.

Next we consider randomization over the treatment which alters the estimand  
% \(\mathbb{E}_{Z}(\tau)\) 
and the construction of the confidence interval. 
It follows that the proof of Theorem \ref{thm:dra_tau_inf} can be used to establish an asymptotic normality result considering the treatment assignment as random.

\begin{proposition}
\label{clt_delta}
Under Assumptions \ref{cond_unconf}, \ref{ass:cond_ind}, \ref{assump:outc_mod}, and \ref{assum:DR} %, and \ref{ass:bounded_DT}, 
we have:
\begin{equation*}
    \sqrt{n} \ (\Sigma^D_n)^{-1/2} ( \hat{\tau}^{AIPW} - \mathbb{E}_Z (\tau) ) \xrightarrow{d} N(0,1),
\end{equation*}
where $\bar e := \frac{1}{n} \sum_{i=1}^n e_i $, \(\Gamma_n  := \sum_{i=1}^n \left[ Z_i(Y_i - \tau_i - f_i) - \frac{e_i}{1 - e_i} (1 - Z_i) (Y_i - f_i) \right]^2\), and 
\[
\Sigma^D_n = \frac{1}{n\bar e^2}\mathbb{E}_{Z,Y} [\Gamma_n] + \frac{1}{n^3}\sum_{i=1}^n \sum_{j=1}^n \sum_{h=1}^n \frac{e_i (1 - e_i) e_j e_h }{ \bar e^4}(\tau_i - \tau_j)(\tau_i - \tau_h).
\]
 \end{proposition}

Proposition \ref{clt_delta} establishes the limiting distribution with the new centering and a correction on the standard error due to the randomness of the treatment. The next proposition discusses the estimation of the relevant variance. 

\begin{proposition}
\label{cons_var_epsilon_Z}
Let \(\widetilde{\Sigma}^C_n = \left( \frac{1}{n} \sum_{i=1}^n Z_i\right)^{-1}  \widetilde{\Sigma}_n\), where \(\widetilde{\Sigma}_n\) is defined in Corollary \ref{cor:var}. Under Assumptions \ref{cond_unconf}, \ref{ass:cond_ind}, \ref{assump:outc_mod}, and \ref{assum:DR}, confidence intervals based on $\widetilde{\Sigma}^C_n$ are asymptotically conservative, namely with probability $1-o(1)$ we have $$\Sigma^D_n \leq \widetilde{\Sigma}^C_n  + o(1).$$
Moreover, for \(\delta_n\) as defined in Assumption \ref{assum:DR}, suppose that \(\frac{1}{n} \sum_{i=1}^n (\hat{\tau}_i - \tau_i)^2 \leq \delta_n^2\) 
%and \(\frac{1}{n} \sum_{i=1}^n (\hat{e}_i - e_i)^2 \leq \delta_n^2\) 
holds with probability $1-o(1)$. Then Proposition \ref{clt_delta} continues to hold with a plug-in estimator for \(\Sigma^D_n\), with probability $1-o(1)$, given by:
\begin{equation*}
    \begin{split}
       \hat{\Sigma}^D_n & = \widehat \Sigma_n \frac{\mbox{$\frac{1}{n}\sum_{i=1}^nZ_i$}}{ \mbox{$ \left( \frac{1}{n}\sum_{i=1}^n \hat{e}_i \right)^2$}}  +\frac{1}{n^3}\sum_{i=1}^n \sum_{j=1}^n \sum_{h=1}^n  \frac{\hat{e}_i (1 - \hat{e}_i) \hat{e}_j \hat{e}_h }{ \mbox{$\left( \frac{1}{n}\sum_{i=1}^n \hat{e}_i \right)^4$}}(\hat{\tau}_i - \hat{\tau}_j) (\hat{\tau}_i - \hat{\tau}_h).
    \end{split}
\end{equation*}
\end{proposition}

Proposition \ref{cons_var_epsilon_Z} shows that $\widetilde{\Sigma}_n^C$ can be used as a conservative estimate of the variance even if we consider the randomness of the assignment of treatment. Moreover, under consistent estimates of the individual treatment effects and propensity scores, we can construct a consistent estimate for $\Sigma_n^D$.  

\subsection{Rates for the Outcome Regression Estimator}\label{sec:OutcomeRegression}

Next, we consider the outcome regression (OR) estimator \ref{def:tauOR} for (\ref{ATT}).  This estimator can be consistent without the assumption of independent treatment assignment assumption. For the OR estimator, we apply Algorithm 1 and use the resulting estimator of the interference function in (\ref{def:tauOR}). The next result leverages Theorem \ref{thm:single_node} to provide a guarantee for $\hat \tau^{OR}$. 

\begin{theorem}
\label{ora_tau_inf}
Let $\hat\tau^{OR}$ be defined as in (\ref{def:tauOR}) and the interference estimator be constructed based on Algorithm 1. %Under Assumptions \ref{additivity}, \ref{cond_unconf}, \ref{Assump:NoDirectImpact}, \ref{Neigh_inter}, and \ref{assump:outc_mod}, 
Under Assumptions  \ref{cond_unconf} and \ref{assump:outc_mod},  with probability at least $1-2\delta-\exp(-n\tilde c)$, for some positive constants $\tilde c$ and $C$,  we have 
\begin{equation}
\label{ora_tau}
\begin{split}
 | \hat{\tau}^{OR}-\tau | \leq &  \frac{C\lambda}{n}  \sum_{i=1}^n Z_i \! \left[ \min_{ k \in [n] } \frac{r^Z(\gamma_0(G_i^Z(k)))}{\lambda-1}+\bar \sigma \sqrt{\frac{\log(n^2/\delta)}{|V_{\gamma_0(G_i^Z(k))}|}} 
 \ \right] + \bar\sigma \sqrt{\frac{4\log( 2/\delta)}{cn}} \\
\end{split}
\end{equation}
\end{theorem}

Theorem \ref{ora_tau_inf} is based on the oracle inequality of Theorem \ref{thm:single_node}. A consequence of the proof of Theorem \ref{ora_tau_inf} is that (for $\lambda = 2$) with probability $1-3\delta$
$$|\hat{\tau}^{OR}-\tau | \leq  C\max_{i\in [n]} \min_{k \in [n], g_k=\gamma_0(G_i^Z(k))} r^Z(g_k)+\bar \sigma \sqrt{\frac{\log(n/\delta)}{|V_{g_k}|}}  $$
which has the maximum rate for the interference function (moreover, note that the last term in Theorem \ref{ora_tau_inf} cannot determine the rate). However, Theorem \ref{ora_tau_inf} averages the adaptive rates across units. Thus, it allows some units to have a small misspecification while other units have a higher level of misspecification. The estimator $\hat\tau^{OR}$ is in general non-regular due to the regularization bias introduced by selecting the relevant $\hat m_i$-hop by balancing (an estimate of) bias and variance. The performance guarantees of $\hat \tau^{OR}$ are robust with respect to the joint dependence on the assignment of treatment between units.

\section{Experimental Performance of the AIPW Estimator}\label{sec:experiments}

We evaluated the performance of the estimator \(\hat{\tau}^{AIPW}\) under varying graph structures and interference scenarios.
% , we conduct simulations across eight distinct settings. These scenarios are designed to capture a range of configurations involving graph types and interference functions. 
Each simulation is based on a fixed sample size of \(n = 2000\) and are repeated \(M = 10,000\) times.
%
% with noise \(\epsilon \overset{i.i.d.}{\sim} N(0, 1.5^2)\). The simulations and we set the maximum depth of $\widehat{E}_n$ as \(m_{max}\).
%
Two graph models are considered: Erd\H{o}s R\'{e}nyi (ER) model with independent edges, each having probability \(p = \frac{4}{n}\) and Barabási\-Albert (BA) model which follows a preferential attachment mechanism (with parameter \(\alpha = 1/2\)) where new nodes are more likely to connect to high degree existing nodes. The BA model produces sparser graphs with heavier tailed degree distributions than the ER model.
% The graphs used in the simulations are generated using two network models. First, the Erd\H{o}s R\'{e}nyi (ER) model, denoted \(G \sim ER(n, p)\), is employed to construct graphs where each edge exists independently with probability \(p = \frac{4}{n}\). Second, the Barabási\-Albert (BA) model, denoted \(G \sim BA(n, \alpha)\), generates graphs based on a preferential attachment mechanism. In this model, new nodes are more likely to connect to existing nodes with higher degrees. By setting the attachment parameter \(\alpha = 1/2\), we produce graphs with sparse centers and stretched exponential degree distributions.

We consider a single confounder \(X_i = s_i(1)\), where \(s_i(1)\) is the size of the immediate neighborhood of individual \(i\). The treatment assignment mechanism is \(P(Z_i = 1 \mid X_i) = \frac{X_i}{C_0}\), where \(C_0 = (3/2) \max_{i \in [n]} s_i(1)\). This setup ensures that treatment probabilities are proportional to the size of an individual’s neighborhood. The direct treatment effect is $\tau(X_i) = X_i/2$ (meaning that there is greater heterogeneity in the direct treatment effect for BA graphs due to their degree distributions) and we consider two interference functions throughout the simulation: % Additionally, the direct treatment effect depends on the confounder $X_i$ via \(\tau(X_i) = X_i/2\).
\begin{equation}
\label{int_exp_dec} 
\text{exponential decay (exp$(k)$): } f(X, \gamma(G^{Z}_{i})) = \mbox{$\sum_{l=1}^{k} (k-1)^{-l} \left| s^{Z}_{i}(l) - s^{Z}_{i}(l-1) \right|$},
\end{equation}
which decays exponentially on the differences between successive neighborhood sizes, and  
\begin{equation}
\label{int_long_dist}
    % \begin{split}
\text{inverse proportion (inv): }        f\left(X, \gamma\left(G^{{Z}}_{i}\right)\right) = \Lambda(X, {Z})/\min\limits_{l \in \{1, \cdots, m_i\}} \{l : s_i^{{Z}}(l) > 0\}
% \frac{\Lambda(X, {Z})}{\min\limits_{l \in \{1, \cdots, m_i\}} \{l : s_i^{{Z}}(l) > 0\}},
    % \end{split}
\end{equation}
where interference is larger the closer unit $i$ is to a treated unit. Lastly, the additive noise in the simulations is \(\epsilon \overset{i.i.d.}{\sim} N(0, 1.5^2)\) and we set the maximum depth of $\widehat{E}_n$ as \(m_{max}\).

% Interference effects are modeled using two distinct functions. The first is an interference which decays exponentially on the differences between successive neighborhood sizes:
% \begin{equation}
% \label{int_exp_dec} 
%  f_{exp(n)}(\gamma(G^{{Z}}_{i})) = \sum_{l=1}^{n} (k-1)^{-l} \left| s^{{Z}}_{i}(l) - s^{{Z}}_{i}(l-1) \right|.   
% \end{equation}
% % The second is a truncated version of function \eqref{int_long_dist}, restricting the summation to the first three layers of the graph:
% % \begin{equation}
% %     \label{int_cut_exp_dec}
% %     f(\gamma(G^{{Z}}_{i})) = \sum_{l=1}^{3} (k-1)^{-l} \left| s^{{Z}}_{i}(l) - s^{{Z}}_{i}(l-1) \right|.
% % \end{equation}
% The second interference function, referred to as the inverse proportion interference function, quantifies the magnitude of interference as a function of the minimum distance to the nearest treated neighbors:
% \begin{equation}
% \label{int_long_dist}
%     \begin{split}
%         f_{inv}\left(X, \gamma\left(G^{{Z}}_{i}\right)\right) = \frac{\Lambda(X, {Z})}{\min\limits_{l \in \{1, \cdots, m_i\}} \{l : s_i^{{Z}}(l) > 0\}},
%     \end{split}
% \end{equation}
% where \(\{1, \cdots, m_i\}\) represents the distances of neighboring units from unit \(i\). This formulation effectively models the dependency of interference on proximity, with smaller distances resulting in larger contributions to the interference experienced by unit \(i\).

Eight scenarios are constructed to examine the interplay between graph structures and interference types. The combinations are labeled via (Graph Type, Interference Function, $m_{max}$), so for example (ER, exp$(n)$, $n$) means the underlying graph is an ER graph, interference has the exponential decay form of depth $n$, and $m_{max} = n$. When $m_{max} = n$ or is greater than the depth of the exponential decay, we apply the pruning rule from in Algorithm $1$. 
% Here we present results for \(\lambda = 0.4\) using the synthetic treatments of Algorithm 2. Results for $\lambda = 1.01$ and for non-synthetic treatments are substantively similar (though the bias of the estimators observed for the non-synthetic treatments is generally larger than for synthetic treatments) and are presented in the Supplement.
We let $\lambda = 1.01$ and use the synthetic treatments of Algorithm 2.

We report bias, as well as standard errors and coverage for both $\tau$ and $\tau_{ADET}$ in Table~\ref{tab_perf_AIPW_cov_synthetic_lambda_1_01}.  Standard errors for \(\hat{\tau}^{AIPW}(\cdot)\) are constructed using three formulations, reflecting varying degrees of randomness and levels of conservativeness: (1) \(\hat{se} = ( \widehat{\Sigma}_n/\sum_{i=1}^n Z_i)^{1/2}\), where \(\widehat{\Sigma}_n\) is a plug-in estimator of \(\Sigma_n\) from Theorem \ref{thm:dra_tau_inf}, as detailed in Corollary \ref{cor:var}; (2) \(\hat{se}_D := (\hat{\Sigma}^D_n/n)^{1/2}\), accounts for the randomness arising from \(\epsilon\) and \(Z\) by employing the plug-in estimator \(\hat{\Sigma}^D_n\), as defined in Proposition \ref{cons_var_epsilon_Z}, to estimate \({\Sigma}^D_n\) from Proposition \ref{clt_delta}; (3) 
% This standard error explicitly accounts for randomness arising from both \(\epsilon\) and \(Z\). 
% In both \(\hat{se}\) and \(\hat{se}_D\), \(\{\hat{\tau}_i\}_{i=1}^n\) represent predictions derived from a linear regression model of \(\{Y_i - \hat{f}_i : Z_i = 1, i \in [n]\}\) on \(\{X_i : Z_i = 1, i \in [n]\}\). 
\(\hat{se}_C := (\widetilde{\Sigma}_n/\sum_{i=1}^n Z_i)^{1/2}\), incorporates \(\widetilde{\Sigma}_n\) from Corollary \ref{cor:var} and Proposition \ref{cons_var_epsilon_Z} to construct a conservative estimate of the uncertainty for both \(\tau\) and \(\tau_{ADET}\). Throughout the simulation, \(\hat{se}\) and \(\hat{se}_D\) rely on regression estimates of the individual treatment effects.
% , as established in Corollary \ref{cor:var} and Proposition \ref{cons_var_epsilon_Z}.

\begin{small}
\begin{table}[t]
\centering
\caption{Performance of the estimator \(\hat{\tau}^{AIPW} \) using \textbf{synthetic} treatments ($\lambda = 1.01$). 
%\(\hat{\tau}\) represents the Monte Carlo average of \(\hat{\tau}^{AIPW}\) across repetitions.
}
\label{tab_perf_AIPW_cov_synthetic_lambda_1_01} % label after caption.
\begin{tabular}{c|cc|ccc|cc|cc}
  \hline
 & & 
 % & 
 &  & &  & \multicolumn{2}{p{2 cm}}{Coverage of $\tau$}   & \multicolumn{2}{|p{3 cm}}{Coverage of $\tau_{ADET}$}  \\ 
  \hline
 & $\tau_{ADET}$ & 
 % $\mathbb{E}_M(\hat{\tau}^{AIPW}_m)$ & 
 bias & $\hat{se}$  & $\hat{se}_D$ & $\hat{se}_C$ & $\hat{se}$ & $\hat{se}_C$  & $\hat{se}_D$ & $\hat{se}_C$ \\ 
  \hline
(ER,exp$(n)$,$n$) & 2.480 
% & 2.462 
& -0.018 & 0.105 & 0.126 &  0.133 &  
0.953  & 0.987 & 0.970 & 0.978 \\ 
(BA,inv,$n$) & 1.679 
% & 1.661 
& -0.018 & 0.153  & 0.176  & 0.187 & 
0.958 & 0.988 & 0.955  & 0.968\\
% 0.953 & 0.988 & 0.955  & 0.968\\ 
  (ER,inv,$n$) &  2.480 
  % & 2.482 
  & 0.002 & 0.091            & 0.102  & 0.105 &   
  0.950  & 0.974  & 0.955  &  0.961 \\ 
  (BA,exp$(n)$,$n$) & 1.679 
  % & 1.644 
  & -0.035 & 0.145 &  0.172 & 0.184 & 
  0.949 & 0.988 &  0.958 & 0.972\\ 
  (ER,exp$(3)$,$10$) & 2.480 
  % & 2.460 
  & -0.021 & 0.100 & 0.117 & 0.123 &  
  0.949 & 0.982  & 0.963 & 0.971 \\ 
  % 0.948 & 0.982  & 0.963 & 0.971 \\ 
  (ER,exp$(3)$,$3$) &  2.480 
  % & 2.469 
  & -0.011 & 0.097 & 0.111 & 0.115 & 
  0.949 & 0.978 & 0.956 & 0.963 \\ 
  % 0.948 & 0.978 & 0.956 & 0.963 \\ 
  (BA,exp$(3)$,$10$) &  1.679 
  % & 1.625 
  & -0.053  & 0.152 &   0.181  &  0.194  &    
  0.949    & 0.989 &  0.960  & 0.972 \\
  (BA,exp$(3)$,$3$) & 1.679 
  % & 1.658 
  & -0.021  & 0.145 & 0.170 &   0.181  & 
  0.949 & 0.987 &  0.957  & 0.967 \\  
   \hline
\end{tabular}
\end{table}
\end{small}

The simulations empirically demonstrate the relationship between the standard errors, with \(\hat{se} < \hat{se}_C\) and \(\hat{se}_D < \hat{se}_C\). Throughout the simulations, the greater heterogeneity in the degrees of the BA graphs leads to greater heterogeneity in the individual treatment effects, which in turn makes  $\hat{se}_C$ be a substantially more conservative estimator (this is seen through the substantial overcoverage for those graphs). As established in the theory, $\hat{se}$ is the correct standard error for $\tau$, as demonstrated in proper nominal coverage across simulation settings, and similarly $\hat{se}_D$ that incorporates the randomness in $Z$ is the correct standard error for $\tau_{ADET}$ producing proper nominal coverage across those simulations.

\section{Real data application}
\label{sec_sin_out_pattern_0_1_2}
We reanalize the results of the field experiment from 
% We consider the field experiment analyzed in 
\cite{aronow2017estimating} and \cite{paluck2016changing}, which investigated the effects of anti-conflict interventions on school conflict climates. The experiment involved 28 schools and in 
% of 56 schools selected through block randomization, stratified by school-level covariates. In 
each school, 15\% of students (capped at 64) were identified as eligible for a bi-weekly anti-conflict training intervention based on gender and grade. Half of eligible students were randomly assigned to this
% receive bi-weekly anti-conflict training as an 
intervention. 
% which encourages these students to address common conflicts and publicly promote anti-conflict norms. Thus, the treatment was assigned hierarchically: first at the school level, then at the individual level. 
At the end of the school year, various anti-conflict norms were measured, including the self-reported wearing of orange wristbands, which were rewards for anti-conflict behaviors. 
Previous work identified this measure as particularly salient for assessing the intervention's efficacy. Social networks were constructed based on students’ nominations of up to 10 peers they spent time with
% in the past weeks, 
forming an undirected graph where an edge exists if either student nominated the other.

We utilize several estimators, each accompanied by valid variance estimators. For \(\hat{\tau}^{AIPW}\), propensity scores are estimated using the covariates ``grade" and ``gender," consistent with the experimental design in \cite{paluck2016changing}. The patterns we considered are the number of extra treated units equal to $0$, $1$ or larger than $1$ by increasing the radius of unit's neighbors. For example, $(0,1, \geq 2)$ refers to the set of units where the number of treated units at $0$-th distance neighbor is $0$, the number of treated units at $1$-st distance neighbor is equal to $1$, the number of treated units at $2$-nd distance neighbor is larger or equal to $2$. We estimate $\hat{\tau}^{AIPW}$ under penalty parameter \(\lambda = 1.01\) respectively.\footnote{Since \(\bar{\sigma}\) in Algorithm 1 is unknown, we estimate the variance of each pattern \(g\) by retaining nodes in \(\widehat{E}_n\) with at least five observations. The upper bound in Step 2 of Algorithm 1 is modified to incorporate these variance as:  
% \[
% |\hat{f}(g^\prime) - \hat{f}(g^{\prime \prime})| \geq \lambda \sqrt{2\log(2n^2/\delta)} \left(\frac{\sigma(g^\prime)}{\sqrt{|V_{g^\prime}|}} + \frac{\sigma(g^{\prime \prime})}{\sqrt{|V_{g^{\prime \prime}}|}}\right),
% \]
\(
|\hat{f}(g^\prime) - \hat{f}(g^{\prime \prime})| \geq \lambda \sqrt{2\log(2n^2/\delta)} \left({\sigma(g^\prime)}/{\sqrt{|V_{g^\prime}|}} + {\sigma(g^{\prime \prime})}/{\sqrt{|V_{g^{\prime \prime}}|}}\right),
\)
where \(\sigma(g)\) is estimated by sample variance for units with pattern \(g\), defined as \(\sigma_g =({|V_g| - 1})^{-\frac{1}{2}} \cdot \left( \sum_{i \in V_g} (Y_i - \bar{Y}_g)^2\right)^{\frac{1}{2}} \).} The maximum tree depth before pruning is set to \(m_{max} = 4\). 

We also consider the H\'{a}jek estimators for the treated population which are compatible with the exposure mapping introduced in \cite{aronow2017estimating}: $D_i = d_{ab}$ where $a$ is 0 if unit $i$ is not treated and 1 otherwise and $b$ is 0 if unit $i$ does not have any treated neighbors, and 1 otherwise. We let $\hat{\tau}^{HJ}(d_{1j}, d_{0j})$ be the H\'{a}jek estimator contrasting exposure $d_{1j}$ with $d_{0j}$ for $j\in\{0,1\}$.
% as follows:  
% \begin{equation}
% \label{instance_exp_mapping}
% D_i =
% \begin{cases}
% d_{11}, & \text{if } z_i \cdot 1\{ s_i^{z}(1) > 0 \} > 0, \\
% d_{10}, & \text{if } z_i \cdot 1\{ s_i^{z}(1) = 0 \} > 0, \\
% d_{01}, & \text{if } (1 - z_i) \cdot 1\{ s_i^{z}(1) > 0 \} > 0, \\
% d_{00}, & \text{if } (1 - z_i) \cdot 1\{ s_i^{z}(1) = 0 \} > 0.
% \end{cases}
% % \end{equation}
% The Hájek estimator for \(\tau\) and \(\tau_{ADET}\) are defined as:  
% \[
% \hat{\tau}^{HJ}(d_{1j}, d_{0j}) = \frac{\sum_{i} w_i(d_{1j}) \cdot 1\{D_i = d_{1j}\} \cdot Y_i}{\sum_{i} w_i(d_{1j}) \cdot 1\{D_i = d_{1j}\}} - \frac{\sum_{i} w_i(d_{0j}) \cdot 1\{D_i = d_{0j}\} \cdot Y_i}{\sum_{i} w_i(d_{0j}) \cdot 1\{D_i = d_{0j}\}},
% \]
% where \(j \in \{0, 1\}\). 
For comparison, we also include \(\hat{\tau}^{NV}\), which does not adjust for confounding or interference. 
We derive standard error estimators under the two scenarios that the estimand is $\tau$ (written as \(\hat{se}(rs)\)) and $\tau_{ADET}$ (written as \(\hat{se}(rs)\)). When estimating $\tau_{ADET}$ the standard error estimators are constructed by modifying the HAC variance estimators utilized in \cite{wang2024designbasedinferencespatialexperiments} and applying the law of total variance to account for the joint distribution of treatments and shocks.
% The standard error estimators for \(\hat{\tau}^{HJ}(d_{1j}, d_{0j})\) (\(j \in \{0, 1\}\)) and \(\hat{\tau}^{NV}\) are straightforward to derive when randomness arises solely from shocks. Under randomness in both \(Z\) and shocks, the standard error estimators are constructed by modifying the HAC variance estimators utilized in \cite{wang2024designbasedinferencespatialexperiments} and applying the law of total variance to account for the joint distribution of treatments and shocks. 
% We denote the standard error estimator accounting for randomness in shocks only as \(\hat{se}(rs)\), and the estimator accounting for randomness in both shocks and treatments as \(\hat{se}(rs)\). 
For \(\hat{\tau}^{AIPW}\), these are given by $\hat{se}(rs) = (\widehat{\Sigma}_n/\sum_{i=1}^n Z_i)^{1/2}$ and $\hat{se}(rst) = (  \widehat{\Sigma}^D_n / n)^{1/2}$. %:
%\[
%\hat{se}(rs) = \left( \frac{1}{\sum_{i=1}^n Z_i} \widehat{\Sigma}_n \right)^{1/2}, \quad \hat{se}(rst) = \left( \frac{1}{\sum_{i=1}^n n} \widehat{\Sigma}^D_n \right)^{1/2}.
%\]

\begin{minipage}{0.45\textwidth}
\begin{table}[H] \centering 
 \caption{Comparison of \(\hat{\tau}^{AIPW}\) 
 (with 
 % \(\lambda \in \{0.4, 0.7, 1.01\}\) and 
 Patterns \(0, 1, \geq 2\)), H\'{a}jek and naive estimators}
%  \caption{Comparison of \(\hat{\tau}^{AIPW}\) (with \(\lambda = 1.01\) and Patterns \(0, 1, \geq 2\)), H\'{a}jek and naive estimators for real data}
\label{tab:num_treat_units_pattern_0_1_2_all_lambdas}
\resizebox{\textwidth}{!}{
\begin{tabular}{@{\extracolsep{5pt}} cccc} 
\\[-1.8ex]\hline 
 & point est & $\hat{se}(rs)$  & $\hat{se}(rst)$   \\ 
\hline \\[-1.8ex] 
% \multirow{3}{*}{$\hat{\tau}^{AIPW}$} & $\lambda=0.4$ & $0.1853$ & $0.0203$ & $0.0204$  \\ 
%$\hat{\tau}^{HJ}$ & $0.2199$ & $0.0263$ &    &  &\\
% & $\lambda=0.7$ & $0.1925$  & $0.0202$ & $0.0204$  \\ 
$\hat{\tau}^{AIPW}$  
% $\lambda=1.01$ 
& $0.1871$ & $0.0203$ & $0.0204$  \\ 
$\hat{\tau}^{HJ}(d_{11},d_{01})$ & $0.1570$  & $0.0362$ & $0.0509$   \\
$\hat{\tau}^{HJ}(d_{10},d_{00})$ & $0.2496$ & $0.0348$ & $0.0355$   \\
$\hat{\tau}^{NV}$ & $0.2484$ & $0.0203$ & $0.0283$   \\
\hline \\[-1.8ex] 
\end{tabular}
}
%\vspace{0.5em} % Add vertical space
%    \parbox{\textwidth}{\footnotesize 
%\(\hat{se}(rs)\): the estimator of standard error accounts for randomness in shocks only; \(\hat{se}(rst)\): the estimator of standard error accounts for randomness in both shocks and treatments. 
%   }
\end{table}
\end{minipage}%
% \hfill
\begin{minipage}{0.5\textwidth}
\begin{figure}[H]
\resizebox{.7\textwidth}{!}{
       \begin{tikzpicture}
    % 0-th layer
    \node[align=left] at (5,7) {$0$};
    %\node[align=left] at (4,6.45) {$1$};
    \draw[thick] (5,7) circle (0.4cm);
    \node[align=left] at (5,6.4) {$0.044$};
    \node[align=left] at (5,6.0) {$(1670)$};
    
    \draw[thick] (2,5.4) -- (5,5.8); 
    \draw[thick] (5,5.4) -- (5,5.8);
    \draw[thick] (8,5.4) -- (5,5.8); 

    % 1-st layer
    \node[align=left] at (2,5.0) {$0$};
    \node[align=left] at (5,5.0) {$1$};
    \node[align=left] at (8,5.0) {$\geq 2$};
    
    \draw[thick] (2,5.0) circle (0.4cm);
    \draw[thick] (5,5.0) circle (0.4cm);
    \draw[thick] (8,5.0) circle (0.4cm);

    % estimator 
    \node[align=left] at (2,4.4) {$0.009$};
    \node[align=left] at (2,4.0) {$(1296)$};
    \node[align=left] at (5,4.4) {$0.153$};
    \node[align=left] at (5,4.0) {$(261)$};
    \node[align=left] at (8,4.4) {$0.186$};
    \node[align=left] at (8,4.0) {$(113)$};

    % 2-nd layer 
    \node[align=left] at (3.5,3) {$0$};
    \node[align=left] at (5,3) {$1$};
    \node[align=left] at (6.5,3) {$\geq 2$};
    
    \node[align=left] at (3.5,2.4) {$0.148$};
    \node[align=left] at (3.5,2.0) {$(128)$};
    \node[align=left] at (5,2.4) {$0.134$};
    \node[align=left] at (5,2) {$(82)$};
    \node[align=left] at (6.5,2.4) {$0.196$};
    \node[align=left] at (6.5,2.0) {$(51)$};

    \draw[thick] (3.5,3) circle (0.4cm);
    \draw[thick] (5,3) circle (0.4cm);
    \draw[thick] (6.5,3) circle (0.4cm);

    \draw[thick] (3.5,3.4) -- (5,3.8); 
    \draw[thick] (5,3.4) -- (5,3.8); 
    \draw[thick] (6.5,3.4) -- (5,3.8); 

    %\node[align=left] at (8.7,4.6) {$4$};
    %\node[align=left] at (9.4,4.6) {$\cdots$};
    %\node[align=left] at (10.1,4.6) {$8$};

    %\draw[thick] (8.7,4.6) circle (0.25cm);
    %\draw[thick] (10.1,4.6) circle (0.25cm);

    %\draw[thick] (8.7,4.85) -- (9.5,5.4); 
    %\draw[thick] (10.1,4.85) -- (9.5,5.4); 

    % 3-rd layer 

\end{tikzpicture}
}\vspace{-.25cm}
        \caption{Structure of \(\widehat{\mathcal{K}}\), interference estimators, and control unit counts at tree nodes generated by algorithm $2$ (\
        % (\lambda = 1.01\), 
        node size \(\geq 5\)) 
        }
\label{Fig:real_data_pattern_0_1_2_lambda_1_01_est}
\end{figure}
\end{minipage}

% \begin{table}[H] \centering 
%   \caption{Comparison of \(\hat{\tau}^{AIPW}\) (with \(\lambda \in \{0.4, 0.7, 1.01\}\) and Patterns \(0, 1, \geq 2\)), H\'{a}jek and naive estimators for real data} 
% \label{tab:num_treat_units_pattern_0_1_2_all_lambdas}
% \begin{tabular}{@{\extracolsep{5pt}} ccccc} 
% \\[-1.8ex]\hline 
%  & & point est & $\hat{se}(rs)$  & $\hat{se}(rst)$   \\ 
% \hline \\[-1.8ex] 
% \multirow{3}{*}{$\hat{\tau}^{AIPW}$} & $\lambda=0.4$ & $0.1853$ & $0.0203$ & $0.0204$  \\ 
% %$\hat{\tau}^{HJ}$ & $0.2199$ & $0.0263$ &    &  &\\
% & $\lambda=0.7$ & $0.1925$  & $0.0202$ & $0.0204$  \\ 
%  & $\lambda=1.01$ & $0.1871$ & $0.0203$ & $0.0204$  \\ 
% $\hat{\tau}^{HJ}(d_{11},d_{01})$ & & $0.1570$  & $0.0362$ & $0.0509$   \\
% $\hat{\tau}^{HJ}(d_{10},d_{00})$ & & $0.2496$ & $0.0337$ & $0.0355$   \\
% $\hat{\tau}^{NV}$ & & $0.2484$ & $0.0203$ & $0.0283$   \\
% \hline \\[-1.8ex] 
% \end{tabular}
% %\vspace{0.5em} % Add vertical space
% %    \parbox{\textwidth}{\footnotesize 
% %\(\hat{se}(rs)\): the estimator of standard error accounts for randomness in shocks only; \(\hat{se}(rst)\): the estimator of standard error accounts for randomness in both shocks and treatments. 
% %   }
% \end{table}

Table \ref{tab:num_treat_units_pattern_0_1_2_all_lambdas} reports the estimated direct treatment effects using the original exposure mappings and our adaptive approach. We see that failing to adapt to general confounding and interference patterns leads to smalled direct effect estimates. To further explore the underlying structures for the heterogeneity of interference, we investigate \(\widehat{\mathcal{K}}\) under the less conservative  penalty parameters \(\lambda = 0.4, 0.7,\) (presented in Appendix \ref{sec:additional_figs_analysis}) and \(1.01\).
If the correct exposure mapping is defined by \(d_{11}, d_{10}, d_{01},\) and \(d_{00}\), it would be expected that \(\widehat{\mathcal{K}}\) retains only the first layer. However, as shown in Figure \ref{Fig:real_data_pattern_0_1_2_lambda_1_01_est}, and Figures \ref{Fig:real_data_pattern_0_1_2_lambda_0_4_est} and \ref{Fig:real_data_pattern_0_1_2_lambda_0_7_est} in Appendix \ref{sec:additional_figs_analysis}, nodes in deeper layers are consistently retained. This observation suggests the presence of richer structures governing the heterogeneity of interference in this data application.

\section{Additional Discussion and Extensions}\label{sec:Additional}

There are many possible extensions of the current baseline model that are of interest. In this section, we point out some extensions that are relatively direct to accommodate (possibly increasing the data requirements), while we mention others that would require fundamentally new results. 

Throughout the paper we focus on the estimation of the average direct treatment effect on the treated (ADTT) under a simple specification that allows us to clearly communicate the main insights and key contributions (i.e., adaptivity, oracle inequalities, handle feature engineering, etc). Nonetheless the ideas here can be considered in other settings.

For example, the estimation of the average direct treatment effect (ADTE) is a key estimand in causal inference. The ideas proposed in this work extends naturally to this case. In order to pursue this estimand we consider the potential outcomes model described in Assumption \ref{assump:outc_mod} with the additional condition that $\tau_i=\tau(X_i)$. The additional assumption allows us to exploit the idea that $\tau(X_i)=\mathbb{E}[Y_i-f_i\mid X_i]$ when the $i$th node is treated.

Based on Algorithm 1, we obtain estimates of the interference, $\hat f_i, i\in[n]$. Therefore we can estimate the function $\tau(\cdot)$ by estimating a conditional mean model for the data $(Y_i-\hat f_i,X_i)_{\{i\in[n]:Z_i=1\}}$ pertaining to the treated nodes. Letting $\hat\tau_i:=\hat\tau(X_i)$, we have 
 different estimators for ADTE, e.g., $\widehat{\tau}_{ADTE}^{AIPW} = \frac{1}{n} \sum_{i=1}^n  \hat\tau(X_i) + Z_i\frac{(Y_i- \hat\tau(X_i) - \hat f_i)}{\hat e_i} -  (1-Z_i)\frac{(Y_i - \hat f_i)}{1-\hat e_i}$.

The proposed estimator for $\tau(\cdot)$ relies on an estimated target due to the use of the interference estimator. Despite of the bias that this introduces, the rates of convergence for the interference can be used to establish rates of convergence for the proposed estimator. 

The numerical experiments in Section~\ref{sec:simulation-with-x} illustrates a natural extension of our Algorithm 1 to the setting of discrete covariates (i.e. finite support). Importantly, the theoretical development in the paper can be immediately adapted to this setting. The approximation error in Definition \ref{bias_term} is now defined over patterns of the form $(x,g)$, requiring us to track $V_{x,g}$ throughout. This necessarily has an effect on the effective sample size for estimating $\hat f(x,g)$ in Algorithm 1 but the results of Theorem~\ref{thm:single_node} (based on Proposition~\ref{Good_event} in the Appendix) follow, mutatis mutandis. The discrete covariate setting is a special case of the following formulation that can handle continuous and other covariates. We note that stratifying on a discrete covariate and then growing the trees in Algorithm 1 has a direct parallel to first growing a single tree based on graph patterns only, and then enforcing a split on the discrete covariate in the leafs. To see the extension to continuous covariates we define $\hat f(g,\cdot) = \arg\min_{h\in\mathcal{H}}\sum_{i\in V_g} (Y_i-h(X_i))^2$ where $\mathcal{H}$ is a class of function that describes the influence of covariates on the interference patterns. If we grow the tree based on $\hat f(g,\cdot)$ that corresponds to splitting on the $X$'s in the leafs. Unlike the discrete case, however, this more general formulation has a more profound effect on the rates of convergence which impacts the pruning of the tree (i.e. the thresholds that are estimated).

We close by noting that we view the current work as an initial contribution to blend machine learning tools within causal inference under network interference for observational studies. Although outside the scope of the current paper, there are several different variations and extensions. Among the extensions that would require fundamentally new results, we would include to handle unobserved graphs, relax the partial linear setting, and to provide a data-driven estimation for the function of patterns $\gamma$. We view the development of tools to allow proper inference under these alternative models as interesting directions for future research.

\begin{appendix}

\section{Proofs for section \ref{The_Res}}
Let $\mathcal{W}:= \{ w=(\gamma_0(G_i^z(1)),\ldots,\gamma_0(G_i^z(m))) : \ m, i \in [n], z\in\{0,1\}^n \}$
denote a sequence of possible patterns $g=\gamma_0(G_i^z(k))$ that is obtained by increasing $k$. For each element in  $g \in \mathcal{E}$, there is a unique element in $\mathcal{W}$ that corresponds to a path $\gamma_0(G_i^z(1)),\ldots,g$.  We note that $\widehat E_n \subset \mathcal{E}$.

\begin{proposition}
\label{Good_event}
Let $f(g)=\frac{1}{|V_g|}\sum_{i\in V_g} f(\gamma(G_i^Z))$, $\hat f(g) = \frac{1}{|V_g|}\sum_{i\in V_g}  Y_i$, and $\alpha(g)
=   \bar\sigma \sqrt{\frac{2\log( 2n^2/\delta)}{|V_g|}}
$ if $|V_g|>0$. Suppose Assumption \ref{assump:outc_mod} holds and we observe $(Y_i,Z_i)_{i=1}^n$. Then for any $\delta \in (0,1)$ we have 
\begin{equation*}
\mathbb{P}\Big( \  \bigcap_{ g \in \mathcal{W}, |V_g|>0}  \left\lbrace  \left\lvert f(g)- \hat{f}(g) \right\rvert   \leq  \alpha(g)  \right\rbrace   \    \Big) \geq 1-\delta 
\end{equation*}
\end{proposition}
\begin{proof}[Proof for Proposition \ref{Good_event}]
Let  $g \in \mathcal{W}$ and recall that $V_g = V_g(Z)$. Recall that under Assumption  \ref{assump:outc_mod}, since for any $i\in V_g$ we have $Z_i=0$ so that $Y_i=\widetilde Y_i(0,Z_{-i}) = f(\gamma(G_i^Z))+\epsilon_i$ so that
\begin{equation*}
\begin{split}
| f(g)-\hat{f}(g) | & \leq  \frac{1}{|V_g|}\Big| \sum_{i \in V_g} \epsilon_i \Big|   % \leq \bar \sigma \sqrt{\frac{2\log(2n^2/\delta)}{|V_g|}}
\end{split}
\end{equation*}
Then we have
$$
\begin{array}{rl}
\mathbb{P}\left(   \frac{\lvert \sum_{i \in V_g} \epsilon_i \rvert }{\sqrt{|V_g|}} > \bar \sigma \sqrt{2\log(\frac{2n^2}{\delta})}\Bigg|  |V_g|>0 \right) & = \mathbb{E}\left[\mathbb{P}\left(   \frac{\lvert \sum_{i \in V_g} \epsilon_i \rvert }{\sqrt{|V_g|}} > \bar \sigma \sqrt{2\log(\frac{2n^2}{\delta})}\mid |V_g|>0, Z \right) \right] \\
& \leq \delta/ n^2\\
\end{array}
$$
where we used that $\epsilon_i$'s are independent, zero mean  sub-Gaussian random variables with parameter $\bar \sigma^2$ conditionally on $Z$ by Assumption \ref{assump:outc_mod}. 

Thus we have
{\small $$
\begin{array}{rl}
&\mathbb{P}\left( \  \bigcap_{ g \in \mathcal{W} : |V_g|>0}  \left\lbrace  \left\lvert f(g)- \hat{f}(g) \right\rvert   \leq  \alpha(g)  \right\rbrace   \    \right) \\
&= 1-\mathbb{P}\left( \  \bigcup_{ g \in \mathcal{W} : |V_g|>0}  \left\lbrace  \left\lvert f(g)- \hat{f}(g) \right\rvert   >  \alpha(g)  \right\rbrace   \    \right) \\
& \geq 1 -  {\displaystyle \sum_{g \in \mathcal{W}}} \mathbb{P}\left(   \frac{\lvert \sum_{i \in V_g} \epsilon_i \rvert }{|V_g|} > \bar \sigma \sqrt{\frac{2\log(2n^2/\delta)}{|V_g|}}\mid |V_g|>0 \right)\mathbb{P}\left( |V_g|>0 \right)\\
& \geq 1-\frac{\delta}{n^2}\sum_{g \in \mathcal{W}}\mathbb{P}\left( |V_g|>0 \right)\\
& \geq 1-\delta.
\end{array}
$$}
where the last line  is based on $\sum_{g\in \mathcal{W}}  1\{ |V_g|>0 \} \leq n^2$ for every realization of $Z$.  
\end{proof}

\subsection{Proofs for Section \ref{sec:OracleInterference}}
\begin{proof}[Proof of Theorem \ref{thm:single_node}]
By Proposition \ref{Good_event}, for any $g = \gamma_0(G_i^Z(m)) \in \mathcal{W}$ such that $|V_g|>0$,  with probability at least $1-\delta$, 
$$|f(\gamma(G_i^Z)) - \hat{f}(g)| \leq |f(\gamma(G_i^Z)) - f(g)| + | f(g)- \hat{f}(g) | \leq r^Z(g) + \alpha(g)$$
where we used the triangle inequality and the fact that 
$$\begin{array}{rl}
|f(\gamma(G_i^Z)) - f(g)| & =\left| \frac{1}{|V_g|}\sum_{j\in V_g} f(\gamma(G_i^Z)) - f(\gamma(G_j^Z))\right| \\
& \leq \frac{1}{|V_g|}\sum_{j\in V_g} |f(\gamma(G_i^Z)) - f(\gamma(G_j^Z))| \leq r^Z(g) \end{array}$$
since $V_g \cup \{i\} \subseteq \overline{V}_g$  by definition of $\overline{V}_g$ (see Definition \ref{bias_term}).

Fix $i\in [n]$, let $\hat f_i = \hat f( \gamma_0(G_i^Z(\hat m_i))$ denote the estimate of the interference for $i$, let $k^*_i$ denote the minimizer of the RHS in the statement,  and $g_k=\gamma_0(G_i^Z(k))$ denote the pattern at the $k$-hop from $i$. We divide the analysis into three cases. 

\noindent Case 1: $k^*_i=\hat m_i$. In that case we have with probability at least $1-\delta$,
\begin{equation}
\label{oracle_leaf_sce_1}
\begin{split}
 | f(\gamma(G_i^Z)) - \hat f_i | & = | f(\gamma(G_i^Z)) - \hat f(g_{\hat m_i})| \\
&  \leq | f(\gamma(G_i^Z)) - f(g_{\hat m_i})| + |f(g_{\hat m_i}) - \hat f(g_{\hat m_i})|\\
&  = | f(\gamma(G_i^Z)) - f(g_{k_i^*})| + |f(g_{k_i^*}) - \hat f(g_{k_i^*})| \leq  r^Z(g_{k^*_i}) + \alpha(g_{k^*_i})\\
\end{split}   
\end{equation}

\noindent Case 2: $k^*_i > \hat m_i$. In that case with probability at least $1-\delta$, 
%\begin{small}
\begin{equation}
\label{oracle_leaf_sce_2}
\begin{array}{rl}
& | f(\gamma(G_i^Z))-\hat{f}_i | = |f(\gamma(G_i^Z))- \hat f(g_{\hat m_i})| \\
& \leq_{(1)}  |f(\gamma(G_i^Z))- f(g_{k_i^*})| + |f(g_{k_i^*}) - \hat f(g_{k_i^*})| + |\hat f(g_{k_i^*}) - \hat f(g_{\hat m_i})|\\
& \leq_{(2)} r^Z(g_{k^*_i}) + \alpha(g_{k^*_i}) + |\hat f(g_{k_i^*}) - \hat f(g_{\hat m_i})|\\
& \leq_{(3)} r^Z(g_{k^*_i}) + \alpha(g_{k^*_i}) + \lambda\{\alpha(g_{k^*_i})+\alpha(g_{\hat m_i})\} \leq_{(4)} r^Z(g_{k^*_i}) + (1+2\lambda)\alpha(g_{k^*_i})\\
\end{array}
\end{equation}
%\end{small}
\noindent \noindent \!\!where (1) holds by the triangle inequality, (2) holds with probability $1-\delta$ by definition of the approximation error and Proposition \ref{Good_event}, (3) holds because $g_{\hat m_i+1}$ was pruned in Algorithm 1. Since $g_{\hat m_i}=par(g_{\hat m_i+1})$ and $g_{k^*_i}\succeq g_{\hat m_i+1}$, this pruning implies that $|\hat f(g_{\hat m_i}) - \hat f(g_{k^*_i})| \leq \lambda\{\alpha(g_{\hat m_i}) + \alpha(g_{k^*_i})\}$. Finally (4) follows from the fact that $\alpha(g_{\hat m_i}) \leq \alpha(g_{k^*_i})$ since $V_{g_{k^*_i}} \subset V_{g_{\hat m_i}}$ as $\hat m_i < k^*_i$.

\noindent Case 3: $k_i^* < \hat m_i$. We have that with probability $1-\delta$,
\begin{equation}
\label{oracle_leaf_int_1}
\begin{split}
|f(\gamma(G_i^Z)) - \hat{f}_i| &  = |f(\gamma(G_i^Z)) - \hat{f}(g_{\hat m_i})| \\
& \leq_{(1)} |f(\gamma(G_i^Z)) - {f}(g_{\hat m_i})|+|{f}(g_{\hat m_i})- \hat {f}(g_{\hat m_i})| \\
& \leq_{(2)} r^Z(g_{\hat m_i})+ \alpha(g_{\hat m_i})\leq_{(3)} r^Z(g_{k_i^*})+ \alpha(g_{\hat m_i})\\ 
\end{split}
\end{equation}
where (1) follows by the triangle inequality, (2) holds with (uniformly over $i \in [n]$) probability at least $1-\delta$ by Proposition \ref{Good_event}, and (3) holds by $k_i^* < \hat m_i$ and the definition of the approximation error.

We proceed to bound $\alpha(g_{\hat m_i})$. Since $g_{\hat m_i} \in \widehat{\mathcal{K}}$ by Step 2 of Algorithm 1, for some $g'\succeq g_{\hat m_i-1}$, $g'' \succeq g_{\hat m_i}$ we have
\begin{equation} 
\label{oracle_leaf_int_1b} |\hat f(g') - \hat f(g'') | > \lambda \{ \alpha(g')+\alpha(g'')\}  \end{equation}
Moreover, again by the triangle inequality and Proposition \ref{Good_event} (on the same events), with probability at least $1-\delta$
\begin{equation} 
\label{oracle_leaf_int_2}
\begin{split} |\hat f(g') - \hat f(g'') | & \leq |\hat f(g') - f(g')| + |f(g')-f(g'')|+|f(g'')- \hat f(g'') | \\
&   \leq \alpha(g^\prime)+ r^Z(g_{k^*_i})+\alpha(g^{\prime\prime})  \\
\end{split}
\end{equation}
Then, since $\lambda>1$, $\alpha(g_{\hat m_i-1}) \leq \alpha(g')$, and $\alpha(g_{\hat m_i}) \leq \alpha(g'')$, by (\ref{oracle_leaf_int_1b}) and (\ref{oracle_leaf_int_2}) we have
 $$(\lambda-1) \{ \alpha(g_{\hat m_i-1})+\alpha(g_{\hat m_i})\} \leq r^Z(g_{k^*_i})$$

Therefore we have 
$$ |f_i-\hat f_i| \leq r^Z(g_{k^*_i}) + \alpha(g_{\hat m_i}) \leq r^Z(g_{k^*_i}) + \frac{r^Z(g_{k^*_i})}{\lambda-1}.$$

Combining Cases 1, 2, and 3, uniformly over $i\in[n]$, we have that
$$ |f(\gamma(G_i^Z)) - \hat{f}_i| \leq \max\left\{  r^Z(g_{k^*_i}) + (1+2\lambda)\alpha(g_{k^*_i}),  \frac{\lambda}{\lambda -1}r^Z(g_{k_i^*})\right\}$$
holds  with probability at least $1-\delta$.

\end{proof}

 \subsection{Proofs of Section \ref{sec:MainATT}}

\begin{proof}[Proof of Theorem \ref{thm:dra_tau_inf}]
For the $i$th node let $f_i:=f(\gamma(G_i^Z))$ denote the interference, $\hat f_i$ denote its estimate defined in Algorithm 1, and $N_1:=\sum_{i=1}^nZ_i$ denote the number of treated nodes. We have that
\begin{equation}\label{eq:dr_main}
\begin{array}{rl}
|\hat\tau^{AIPW}-\hat\tau^{AIPW*}| & = \left| \frac{1}{N_1}\sum_{i=1}^n Z_i (f_i - \hat f_i) + \frac{(Y_i - f_i)(1-Z_i)e_i}{1-e_i} - \frac{(Y_i - \hat f_i)(1-Z_i)\hat e_i}{1-\hat e_i} \right| \\
& \leq  \left|\frac{1}{N_1}\sum_{i=1}^n (f_i - \hat f_i)\left\{Z_i-\frac{(1-Z_i)e_i}{1-e_i}\right\}\right|\\
&+ \left|\frac{1}{N_1}\sum_{i=1}^n(1-Z_i)(Y_i-f_i)\left\{\frac{e_i}{1-e_i}-\frac{\hat e_i}{1-\hat e_i}\right\} \right|\\
& + \left|\frac{1}{N_1}\sum_{i=1}^n(1-Z_i)(f_i-\hat f_i)\left\{\frac{e_i}{1-e_i}-\frac{\hat e_i}{1-\hat e_i}\right\} \right|\\
& =:  (I) + (II) + (III)
\end{array}
\end{equation}

Since by Assumption \ref{assump:outc_mod} we have that $(Y'_i,Z_i,X_i)$, for $Y'_i := Y_i - f_i$, satisfies SUTVA, unconfoundness and overlap, it follows that conditional on $Z,X$, we have $$\sqrt{N_1}\Sigma_n^{-1/2}(\hat \tau^{AIPW*}- \tau) \rightarrow N(0,1)$$ 
where $\Sigma_n := \frac{1}{\sum_{i=1}^nZ_i}\sum_{i=1}^n \mathbb{E}\left[\{ Z_i(Y_i-\tau_i -f_i) - (Y_i-f_i)(1-Z_i)e_i/(1-e_i)\}^2 \mid Z_i,X_i \right]$.

Thus, to show the asymptotic normality of $\sqrt{n}(\hat \tau^{AIPW}-\tau)$ it suffices to establish that with high probability $1-o(1)$ we have $$(I)+(II)+(III) \leq C\delta_n^{1/2} n^{-1/2} \Sigma_n^{-1/2}.$$  
Step 2 provides the bound for (III), Step 3 for (II), Step 4 for (I). (Step 5 below contains a technical derivation for a concentration inequality.)

{\it Step 2: Bound on (III).}  To bound (III) in (\ref{eq:dr_main}) we apply  Cauchy-Schwartz
$$
\begin{array}{rl}
(III)& :=\left|\frac{1}{N_1}\sum_{i=1}^n(1-Z_i)(f_i-\hat f_i)\left\{\frac{e_i}{1-e_i}-\frac{\hat e_i}{1-\hat e_i}\right\} \right| \\
& \leq \frac{n}{N_1} \sqrt{\frac{1}{n}\sum_{i=1}^n (1-Z_i)(f_i - \hat f_i)^2} 
 \sqrt{\frac{1}{n}\sum_{i=1}^n(1-Z_i)\left(\frac{e_i}{1-e_i}-\frac{\hat e_i}{1-\hat e_i}\right)^2}\\
\end{array}
$$
By Assumption \ref{assum:DR}(i), we have with probability $1-\Delta_n$ $$\sqrt{\frac{1}{n}\sum_{i=1}^n\left(\frac{e_i}{1-e_i}-\frac{\hat e_i}{1-\hat e_i}\right)^2} \leq C\sqrt{\frac{\bar v\log(n/\Delta_n)}{n}} =: r^e_n$$ 

By Lemma \ref{lemma:sparsity_gamma_star},  with probability $1-\delta-\Delta_n$
$$ \begin{array}{rl}
\frac{1}{n}\sum_{i=1}^n (1-Z_i)(f_i - \hat f_i)^2 %&  \{\leq \frac{\lambda}{\lambda-1} \bar \sigma \delta_n n^{-1/2} + 3\lambda \frac{\bar\sigma}{\underline{\sigma}}\sqrt{2\log(2n^2/\delta)} \frac{\delta_n \bar \sigma n^{-1/4}}{\log n} \}^2 \\
& \leq C_\lambda (\bar\sigma/\underline \sigma)^2  \bar\sigma^2 \sqrt{\log(n/\delta)}\mathbb{E}[ \frac{1}{n}\sum_{i=1}^n(f_i-\hat f_i^*)^2\mid Z]\\
& \leq C_\lambda (\bar\sigma/\underline \sigma)^2  \bar\sigma^2 \sqrt{\log(n/\delta)} \delta_n^2 n^{-1/2} / \log n
\end{array}$$
where we used $\delta \geq n^{-c}$, Assumption \ref{assum:DR}(ii) and $D_n \geq 0$.

 Thus, with probability $1-\delta-2\Delta_n-\exp(-nc^2/8)$ we have 
$$(III) \leq \frac{n}{N_1} \cdot C_\lambda \delta_n  (\bar\sigma/\underline\sigma)\bar\sigma\frac{\log^{\frac{1}{2}}(2n^2/\delta)}{\log(n)}  n^{-1/4} C\sqrt{\frac{\bar v  \log(n/\Delta_n)}{n}} \leq C_\lambda' (\bar\sigma/\underline\sigma)\bar\sigma \delta_n n^{-1/2}$$ 
where we used that $\frac{n}{N_1} \leq C$ with probability $1-\exp(-nc^2/8)$.

{\it Step 3: Bound on (II).} To bound (II) we observe that $$\begin{array}{rl}
\mathbb{E}[(1-Z_i)(Y_i-f_i)\mid X_i] & = \mathbb{E}[(1-Z_i)(Y_i(0,Z_{-i})-f_i)\mid X_i] \\
&= \mathbb{E}[(1-Z_{i})\mid X_i] \mathbb{E}[(Y_i(0,Z_{-i})-f_i)\mid X_i] =0
\end{array}
$$ 
where we used Assumption \ref{assump:outc_mod}   and Assumption \ref{cond_unconf}(i-ii) (Unconfoundedness).  By construction we have $(1-Z_i)(Y_i-f_i) = (1-Z_i)(\widetilde Y_i(0,Z_{-i})-f_i) = (1-Z_i)\epsilon_i$ which are independent across $i$ and zero mean conditional on $X$. Let $\widetilde{\mathcal{E}}:=\{ f= (1-Z_i)\epsilon_i\{ \frac{e_i}{1-e_i}-\frac{\tilde e_i}{1-\tilde e_i}\}, e_i \in \mathcal{E}, \frac{1}{n}\sum_{i=1}^n(\frac{e_i}{1-e_i}-\frac{\tilde e_i}{1-\tilde e_i})^2 \leq C \bar v \log n / n \}$. Then with probability $1-\Delta_n- n^{-2}$, we have
{\small $$ 
\begin{array}{rl}
(II) & = \left|\frac{1}{n}\sum_{i=1}^n(1-Z_i)(Y_i-f_i)\left\{\frac{e_i}{1-e_i}-\frac{\tilde e_i}{1-\tilde e_i}\right\} \right|\\
& \leq_{(1)} \max_{\tilde h \in \widetilde{\mathcal{E}}} \left|\frac{1}{n}\sum_{i=1}^n h(Z_i,X_i,\epsilon_i) \right| \\
&\leq_{(2)}   C\sqrt{ \frac{\bar v \log(n)}{n}}  \sup_{\tilde e \in \mathcal{F}} \sqrt{\frac{1}{n}\sum_{i=1}^n (\frac{e_i}{1-e_i}-\frac{\tilde e_i}{1-\tilde e_i})^2}\\
& \leq_{(3)} Cn^{-1} \bar v \log n \leq_{(4)} \delta_n \bar \sigma \delta_n n^{-1/2}
\end{array}$$}
where (1) holds by $\hat e \in \mathcal{E}$ with probability $1-\Delta_n$  by Assumption \ref{assum:DR}(i), (2) holds with probability $1-n^{-2}$ by Lemma 16 in \cite{belloni2011l1} with $D= \bar\sigma C$, $n(\epsilon,\widetilde {\mathcal{E}},P) \leq (A/\epsilon)^{\bar v}$, $F(Y,X,Z) \leq (1-Z_i)\epsilon_i \max e_i/(1-e_i)$, $\|F\|_{P,2} \leq C \bar \sigma$, $\rho = C\sqrt{v\log n / n}$, (3) holds  since $\sup_{\tilde e \in \mathcal{E}} \frac{1}{n}\sum_{i=1}^n (\frac{e_i}{1-e_i}-\frac{\tilde e_i}{1-\tilde e_i})^2 \leq \frac{C\bar v \log(n)}{n}$, and (4) holds by $\bar v^2 \log n \leq \bar\sigma \delta_n^2 n$.
~\\
{\it Step 4: Bounding (I).} Next we proceed to bound (I).  We will use the notation $g \in \gamma$ to denote an element of the partition (i.e. corresponding to a node of the associated context tree). For an $s$-valued function $\tilde f$, it is  consistent with a context tree with at most $s$  nodes denoted by  $\mathcal{T}_s$. The function $\tilde\gamma$ maps a labelled graph $G_i^z$ to one of its $s$ nodes. (That is the function $\tilde\gamma$ defines some function $\tilde m$ that determines the hops, namely $\tilde \gamma (G_i^z) = \gamma_0( G_i^Z(\tilde m))$ where $\tilde m= \tilde m(G_i^Z)$.)

We start by bounding the number of partitions $\hat s$ associated with $\hat \gamma$. By Lemma \ref{lemma:inclusion} we have that $\gamma^* \succeq \hat \gamma$ with probability $1-\delta$  so that $\hat s \leq s$ with the same probability. Furthermore, under Assumptions \ref{assump:outc_mod} and \ref{assum:DR}, Lemma \ref{lemma:sparsity_gamma_star} shows $s\leq (n/\underline \sigma^2) \mathbb{E}[ \frac{1}{n}\sum_{i=1}^n (f_i-\hat f_i^*)^2 \mid Z]$ with probability $1-\delta-\Delta_n$. Therefore, with probability $1-\delta-\Delta_n$
\begin{equation}\label{bound:hat-s}\hat s \leq s \leq (n/\underline \sigma^2) \mbox{$\mathbb{E}\left[ \frac{1}{n}\sum_{i=1}^n (f_i-\hat f_i^*)^2 \mid Z\right] $}
%C\delta_n^2 n^{1/2}/\log^2 n
\end{equation}
Furthermore, with probability at least $1-\delta-\Delta_n$, uniformly over $i\in[n]$, we have that our estimator $\hat f$ satisfies \begin{equation}\label{eq:bound:fstart-hatf}\begin{array}{rl}
|f_i^* - \hat f_i | & \leq |f_i^* -  f_i | + |f_i - \hat f_i | \\
& \leq |f_i^* -  f_i | + \frac{\lambda}{\lambda-1} \bar \sigma \delta_n n^{-1/2} + 3\lambda \frac{\bar\sigma}{\underline{\sigma}}  \sqrt{2\log(2n^2/\delta) \mathbb{E}[ |f_i-\hat f_i^*|^2 \mid Z]}\\
& \leq C_{\lambda,\sigma}  \sqrt{2\log(2n^2/\delta) \mathbb{E}[ |f_i-\hat f_i^*|^2 \mid Z]} \\
\end{array}\end{equation}%& \leq |f_i^* -  f_i | + \frac{\lambda}{\lambda-1} \bar \sigma \delta_n %n^{-1/2} + 3\lambda \frac{\bar\sigma}{\underline{\sigma}}  %\frac{\sqrt{2\log(2n^2/\delta)}}{\log n} \bar{\sigma} \delta_n  n^{-1/4}\\
%& \leq C_{\lambda,\sigma}  \bar \sigma \delta_n n^{-1/4}/\sqrt{\log n}\\ 
%\end{array}\end{equation}
where the first inequality holds by the triangle inequality, the second with probability $1-\delta-\Delta_n$ by Lemma \ref{lemma:sparsity_gamma_star}, and the third by Assumption \ref{assum:DR}(ii). We define $r$ such that with probability $1-\delta - \Delta_n$ we have $ r \geq  C_{\lambda,\sigma}  \sqrt{2\log(2n^2/\delta) \max_{i\in [n]}\mathbb{E}[ |f_i-\hat f_i^*|^2 \mid Z]} $, which is given by Assumption \ref{assum:DR}(ii).

It will be convenient to stress a few observations and notation that will be used. In what follows we will write $\tilde f \sim \tilde \gamma$ if $\tilde f$ is function associated with nodes in $\tilde\gamma$, we will let $\mathcal{F}_{\tilde \gamma} := \{ \tilde f: {\rm{Im}}(\tilde \gamma) \to [-C,C] : \tilde f \ \mbox{is $s$-valued}\}$, and $\mathcal{T}_s:=\{ \tilde\gamma : {\rm{Im}}(\tilde \gamma) \ \mbox{is $s$-valued}\}$. Both $f^*$ and $\hat f$ are (random) $s$-valued with probability $1-\delta-\Delta_n$. Thus we say that $f^*\sim \gamma^*, \hat f \sim \hat \gamma$ where $\gamma^*, \hat \gamma \in \mathcal{T}_s$ with probability $1-\delta-\Delta_n$. Note that if a pattern $g \in \hat \gamma$ does not correspond to any $\hat \gamma(G_i^Z)$, the bound (\ref{eq:bound:fstart-hatf}) does on impose restrictions on $f^*(g)-\hat f(g)$. Nonetheless, we can extend $\hat f$ to $\hat f^e$ so that $\hat f^e(g) = f^*(g)$ if $g \not \in \{\hat \gamma(G_i^Z), i\in [n]\}$ with $\hat f^e$ being $s$-valued and without changing $(f^*_i-\hat f_i)_{i\in [n]}$. Moreover, by construction, both $\gamma^*, \hat\gamma \subset \widetilde {E}^o_n$. We will define $\widetilde{\mathcal{F}}_{\tilde \gamma,\tilde \gamma^*, r} = \{ (\tilde f, \tilde f^*) \in \mathcal{F}_{\tilde \gamma} \times  \mathcal{F}_{\tilde \gamma^*} : \|\tilde f - \tilde f^* \|_\infty \leq r \}$.

Letting $U_i:= Z_i - (1-Z_i)e_i/(1-e_i)$, we proceed to bound (I) as follows
{\small \begin{equation}\label{eq:aux:I:initial} \begin{array}{rl}
&  \mathbb{P}\left(  \left| \frac{1}{n} \sum_{i=1}^n(f_i-\hat f_i)U_i  \right|  > t  \right) \\
&  \leq_{(1)} \mathbb{P}\left(  \left| \frac{1}{n} \sum_{i=1}^n(f_i^*-\hat f_i)U_i  \right|  > t - C\delta_n \bar\sigma n^{-1/2}  \right) \\
& \leq_{(2)} \mathbb{P}\left( \max_{\tilde \gamma, \tilde \gamma^* \in \mathcal{T}_{s} \cap \widetilde{E}_n^o, (\tilde f, \tilde f^*) \in  \widetilde{\mathcal{F}}_{\tilde \gamma,\tilde \gamma^*,r} } \left| \frac{1}{n} \sum_{i=1}^n(\tilde f_i^*-\tilde f_i)U_i  \right|  > t - C\delta_n \bar\sigma n^{-1/2} \right) + \delta + \Delta_n\\
& \leq_{(3)} \displaystyle \sum_{\tilde \gamma, \tilde \gamma^* \in \mathcal{T}_{s}} \mathbb{P}\left(  \max_{(\tilde f, \tilde f^*) \in  \widetilde{\mathcal{F}}_{\tilde \gamma,\tilde \gamma^*,r}} \left| \frac{1}{n}\sum_{i=1}^n(f_i^*-\tilde f_i)U_i  \right|  > t - C\delta_n \bar\sigma n^{-1/2},  \tilde \gamma, \tilde\gamma^*\in \widetilde{E}_n^o \right) + \delta + \Delta_n\\
& \leq_{(4)} \displaystyle \sum_{\tilde \gamma,\tilde\gamma^* \in \mathcal{T}_{s}} \mathbb{P}\left(  \max_{(\tilde f, \tilde f^*) \in  \widetilde{\mathcal{F}}_{\tilde \gamma,\tilde \gamma^*,r}^\varepsilon} \left| \frac{1}{n}\sum_{i=1}^n(\tilde f_i^*-\tilde f_i)U_i  \right|  > t - 2C\delta_n \bar\sigma n^{-1/2},  \tilde \gamma, \tilde\gamma^* \in \widetilde{E}_n^o \right) + \delta + \Delta_n\\
%& \leq_{(5)} \displaystyle \sum_{\tilde \gamma,\tilde\gamma^* \in \mathcal{T}_{s}} |\mathcal{F}_{\tilde{\gamma}}^\varepsilon|\cdot |\mathcal{F}_{\tilde{\gamma}^*}^\varepsilon| \max_{(\tilde f, \tilde f^*) \in  \widetilde{\mathcal{F}}_{\tilde \gamma,\tilde \gamma^*,r}}\mathbb{P}\left(   \left| \frac{1}{n}\sum_{i=1}^n(\tilde f_i^*-\tilde f_i)U_i  \right|  > t - 2C\delta_n \bar\sigma n^{-1/2},  \tilde \gamma, \tilde\gamma^* \in \widetilde{E}_n^o \right) + \delta + \Delta_n\\
%& \leq_{(6)} \displaystyle \sum_{\tilde \gamma,\tilde\gamma^* \in \mathcal{T}_{s}} (Cn)^{2s} \max_{(\tilde f, \tilde f^*) \in  \widetilde{\mathcal{F}}_{\tilde \gamma,\tilde \gamma^*,r}}\mathbb{P}\left(   \left| \frac{1}{n}\sum_{i=1}^n(\tilde f_i^*-\tilde f_i)U_i  \right|  > t - 2C\delta_n \bar\sigma n^{-1/2},  \tilde \gamma, \tilde\gamma^* \in \widetilde{E}_n^o \right) + \delta + \Delta_n\\
%& \displaystyle =_{(7)}  \displaystyle \sum_{\tilde \gamma,\tilde\gamma^* \in \mathcal{T}_{s}} (Cn)^{2s} \max_{(\tilde f, \tilde f^*) \in  \widetilde{\mathcal{F}}_{\tilde \gamma,\tilde \gamma^*,r}}\mathbb{P}\left(   \left| \frac{1}{n}\sum_{i=1}^n(\tilde f_i^*-\tilde f_i)U_i  \right|  > t - 2C\delta_n \bar\sigma n^{-1/2}\right) \mathbb{P}\left( \tilde \gamma, \tilde\gamma^* \in \widetilde{E}_n^o \right) + \delta + \Delta_n\\
\end{array}
\end{equation}}
where (1) follows from Assumption \ref{assum:DR}(ii) and $\frac{1}{n}\sum_{i=1}^n(f_i-f_i^*)U_i \leq \frac{1}{n}\sum_{i=1}^nr^Z(g_i^*)|U_i| \leq C\delta_n \bar\sigma n^{-1/2}$,
(2) follows from $f^* \sim \gamma^*, \hat f \in \hat\gamma $ such that  $\gamma^*, \hat \gamma \in \mathcal{T}_{s}$ with probability at least $1-\delta-\Delta_n$, and (3) follows from the union bound. Inequality (4) follows from setting $\varepsilon= \bar \sigma \delta_n n^{-1/2}$ so that for any $\tilde f \in \mathcal{F}_{\tilde \gamma}$ we find a  $\tilde f' \in \mathcal{F}_{\tilde \gamma}^\varepsilon$ such that $|\tilde f_i-\tilde f_i'| \leq \bar \sigma \delta_n n^{-1/2} $ and recall that $|U_i|\leq C$. Then note that the main term in the RHS of (\ref{eq:aux:I:initial}) above can be bounded as
{\small \begin{equation}\label{eq:aux:I:initial-II} \begin{array}{rl}
& \displaystyle \sum_{\tilde \gamma,\tilde\gamma^* \in \mathcal{T}_{s}} \mathbb{P}\left(  \max_{(\tilde f, \tilde f^*) \in  \widetilde{\mathcal{F}}_{\tilde \gamma,\tilde \gamma^*,r}^\varepsilon} \left| \frac{1}{n}\sum_{i=1}^n(\tilde f_i^*-\tilde f_i)U_i  \right|  > t - 2C\delta_n \bar\sigma n^{-1/2},  \tilde \gamma, \tilde\gamma^* \in \widetilde{E}_n^o \right)\\
& \leq_{(5)} \displaystyle \sum_{\tilde \gamma,\tilde\gamma^* \in \mathcal{T}_{s}} |\mathcal{F}_{\tilde{\gamma}}^\varepsilon|\cdot |\mathcal{F}_{\tilde{\gamma}^*}^\varepsilon| \max_{(\tilde f, \tilde f^*) \in  \widetilde{\mathcal{F}}_{\tilde \gamma,\tilde \gamma^*,r}}\mathbb{P}\left(   \left| \frac{1}{n}\sum_{i=1}^n(\tilde f_i^*-\tilde f_i)U_i  \right|  > t - 2C\delta_n \bar\sigma n^{-1/2},  \tilde \gamma, \tilde\gamma^* \in \widetilde{E}_n^o \right)\\
& \leq_{(6)} \displaystyle \sum_{\tilde \gamma,\tilde\gamma^* \in \mathcal{T}_{s}} (Cn)^{2s} \max_{(\tilde f, \tilde f^*) \in  \widetilde{\mathcal{F}}_{\tilde \gamma,\tilde \gamma^*,r}}\mathbb{P}\left(   \left| \frac{1}{n}\sum_{i=1}^n(\tilde f_i^*-\tilde f_i)U_i  \right|  > t - 2C\delta_n \bar\sigma n^{-1/2},  \tilde \gamma, \tilde\gamma^* \in \widetilde{E}_n^o \right)\\
& \displaystyle =_{(7)}  \displaystyle \sum_{\tilde \gamma,\tilde\gamma^* \in \mathcal{T}_{s}} (Cn)^{2s} \max_{(\tilde f, \tilde f^*) \in  \widetilde{\mathcal{F}}_{\tilde \gamma,\tilde \gamma^*,r}}\mathbb{P}\left(   \left| \frac{1}{n}\sum_{i=1}^n(\tilde f_i^*-\tilde f_i)U_i  \right|  > t - 2C\delta_n \bar\sigma n^{-1/2}\right) \mathbb{P}\left( \tilde \gamma, \tilde\gamma^* \in \widetilde{E}_n^o \right)\\
\end{array}
\end{equation}}
where relation (5) holds by the union bound and $\widetilde{\mathcal{F}}_{\tilde \gamma,\tilde \gamma^*,r}  \subset \mathcal{F}_{\tilde \gamma} \times \mathcal{F}_{\tilde \gamma^*}$ (6) holds since $\tilde f \sim \tilde \gamma$ is $s$-valued so that $|\mathcal{F}_{\tilde\gamma}^{\varepsilon}| \leq (2C/\varepsilon)^{s}$ since $\tilde \gamma$ is fixed.
Finally, recall that $\tilde\gamma$ and $\tilde \gamma^*$ are fixed (and non-random). Thus (7) holds by independence between $\tilde Z$, that determines $\tilde\gamma, \tilde\gamma^* \in \widetilde E_n^o$, and $Z$, that determines $(f^*,\tilde f, U)$.

By Step 2, we have that for fixed $\tilde{f}^*, \tilde {f} \sim \tilde\gamma \in \mathcal{T}_{s}$ such that $\|\tilde f - \tilde f^*\|_\infty \leq r$, we have
%\begin{equation}\label{eq:Step2result}\mathbb{P}\left(   \left| %\frac{1}{n}\sum_{i=1}^n(\tilde f_i^*-\tilde f_i)U_i  \right|  > t  \right) %\leq \exp\left( -\frac{1}{2}t^2 / \{C n^{-3/2}((1+2D_n)^2\log^{-2} n\} %\right)\end{equation}
%{\color{red} 
{\small \begin{equation}
\begin{split}
\label{eq:Step2result}
\mathbb{P}\left(   \left| \frac{1}{n}\sum_{i=1}^n(f_i^*-\tilde f_i)U_i  \right| > \left( t - 2C\delta_n \bar\sigma n^{-1/2}    \right) \right) &
\leq \exp\Bigg( \frac{-\mbox{$\frac{1}{2}$}\left( t - 2C\delta_n \bar\sigma n^{-1/2} \right)^2n}{ (1+2D_n)^2\|\tilde f^* - \tilde f\|_\infty^2} \Bigg) \\ 
&\leq \exp\Bigg( \frac{-\mbox{$\frac{1}{2}$}\left( t - 2C\delta_n \bar\sigma n^{-1/2} \right)^2n}{  (1+2D_n)^2 r^2} \Bigg) \\ 
\end{split}
\end{equation}
}
Therefore by (\ref{eq:aux:I:initial}), (\ref{eq:aux:I:initial-II}) and (\ref{eq:Step2result}) with $$t = \sqrt{2} n^{-1/2} \left( r (1+2D_n) \sqrt{Cs\log n} \right) + 2C\bar\sigma \delta_n n^{-1/2}$$ for $C$ large enough we have
$$ \begin{array}{rl}
  \mathbb{P}\left(  \left| \frac{1}{n} \sum_{i=1}^n(f_i-\hat f_i)U_i  \right|  > t  \right)
&  \displaystyle  \leq (Cn)^{2s}\exp( -Cs\log n ) \sum_{\tilde \gamma^*,\tilde \gamma \in \mathcal{T}_{s}}  \mathbb{P}\left(\tilde \gamma, \tilde\gamma^* \in  \widetilde{E}^0_n \right) + \delta + \Delta_n\\
&  \displaystyle \leq \exp( -Cs\log(n) + 2s\log(Cn) + 4s\log n) + \delta + \Delta_n\\
& \leq  Cn^{-1}  + \delta + \Delta_n\\
\end{array}
$$
where the second inequality follows by  $$\sum_{\tilde \gamma \in \mathcal{T}_{s}} \sum_{\tilde \gamma^* \in \mathcal{T}_{s}} \mathbb{P}\left( \tilde \gamma, \tilde\gamma^* \in \widetilde E_n^o \right)= \mathbb{E}\Big[ \sum_{\tilde \gamma \in \mathcal{T}_{s}} \sum_{\tilde \gamma^* \in \mathcal{T}_{s}}1\{\tilde \gamma, \tilde\gamma^* \in \widetilde E_n^o \}\Big]\leq \mbox{$\binom{n^2}{s}^2$} \leq n^{4s}$$ 
as any realization of $\widetilde E_n^o$ has at most $n^2$ partitions, $\tilde\gamma, \tilde\gamma^*$ chooses at most $\binom{n^2}{s}$ of them. 

Finally, we have that $t \leq C\bar\sigma \delta_n n^{-1/2}$ since  $$ r (1+2D_n) \sqrt{Cs\log (n)} \leq \frac{r^2}{\sqrt{\log(n/\delta)}}(1+2D_n) \sqrt{C n \log(n)} / \underline{\sigma}  \leq C'_{\lambda,\sigma}\delta_n$$
where the first inequality holds by the bound on $s$ in (\ref{bound:hat-s}) and under the condition in Assumption \ref{assum:DR}(ii) as $r^2 \leq C_{\lambda,\sigma}^2 \max_{i\in[n]} \mathbb{E}[|f_i-f_i^*|^2\mid Z] \log(n/\delta)$.
~\\

\underline{Step 2}. In this step we establish (\ref{eq:Step2result}).
Fix $\tilde \gamma, \tilde \gamma^* \in \mathcal{T}_{s}$, $\tilde f^*\sim \gamma^*, \tilde f \sim \tilde\gamma$, where $$\|\tilde f^*-\tilde f\|_\infty := \max_{z\in\{0,1\}^n, i\in[n]} | \tilde f^*(\tilde \gamma^*(G_i^z)) - \tilde f( \tilde \gamma(G_i^z))| $$ % \leq C\delta_n n^{-1/4}/\log n$$ 
where $\tilde \gamma^*(G_i^z)=\gamma_0(G_i^z(\tilde m_i^*))$ and $\tilde \gamma(G_i^z)=\gamma_0(G_i^z(\tilde m_i))$ for some $\tilde m_i^*$ and $\tilde m_i$ (possibly dependent on $z$). Recall that $U_i= Z_i - (1-Z_i)e_i/(1-e_i)$, $i\in[n]$, are zero mean and independent across $i$, so that $\mathbb{E}[\{\tilde f^*( \tilde \gamma^*(G_i^Z))-\tilde f(\tilde  \gamma(G_i^Z)) \}U_i \mid Z_{-i} ] = 0$. We define $h(Z):=  \frac{1}{n}\sum_{i=1}^n\{\tilde f^*( \tilde \gamma^*(G_i^Z))-\tilde f(\tilde  \gamma(G_i^Z)) \}U_i$. 
Define $\tilde I^*_k(z) := \{ \ell \in [n] : \ell \ \mbox{is a vertex in} \  G_k^z(\tilde m^*_k) \}$ and $\tilde I_k(z) := \{ \ell \in [n] : \ell \ \mbox{is a vertex in } $ $\ G_k^z(\tilde m_k) \}$.
We have that for any $z, z'$ that differs only in the $i$th component satisfies 
$$\begin{array}{rl}
|h(z)-h(z') | & \leq \left|\frac{1}{n} \sum_{k : i \in \tilde I^*_k(z) \cup \tilde I_k(z) } \{\tilde f^*( \tilde \gamma^*(G_k^z))  -\tilde f(\tilde \gamma(G_k^z)) \} U_k \right|\\
& +\left|\frac{1}{n} \sum_{k : i \in \tilde I^*_k(z') \cup \tilde I_k(z') } \{\tilde f^*( \tilde \gamma^*(G_k^{z'}))-\tilde f(\tilde \gamma^*(G_k^{z'})) \} U_k \right| \\
& + |\frac{1}{n}  \{\tilde f^*( \tilde \gamma^*(G_i^z))  -\tilde f(\tilde \gamma(G_i^z)) \} (U_i - U_i') | \\
& \leq  \frac{1}{n}\bar C |\{ k : i \in \tilde I^*_k(z) \cup \tilde I_k(z) \}| \max_{k \in [n]} | \tilde f^*( \tilde \gamma^*(G_k^{z}))-\tilde f(\tilde \gamma^*(G_k^{z})) |  \\
&+  \frac{1}{n}\bar C |\{ k : i \in \tilde I^*_k(z') \cup \tilde I_k(z') \}| \max_{k \in [n]} | \tilde f^*( \tilde \gamma^*(G_k^{z'}))-\tilde f(\tilde \gamma^*(G_k^{z'})) |  \\
& + \frac{1}{n}\bar C | \tilde f^*( \tilde \gamma^*(G_i^{z}))-\tilde f(\tilde \gamma^*(G_i^{z})) |\\
& \leq \frac{1}{n}(1+2D_n) \| \tilde f^* - \tilde f\|_\infty
\end{array}
$$ 
where the last inequality is based on the definition of $D_n$ as in Assumption \ref{assum:DR} (ii) and $\tilde\gamma \preceq  \gamma^*$. Since $Z_i$'s are independent, by McDiarmid's inequality (\cite{mcdiarmid1989method}) we have
\begin{equation}\label{eq:aux:I:part3:McDiamard}
 \mathbb{P}\left( |h(Z)|> t \right) \leq 2\exp\left( -\frac{1}{2}t^2n/\{ (1+2D_n)^2 \|\tilde f^* - \tilde f\|_\infty^2 \}\right)
 \end{equation}

%The result follows from $\tilde f^*,\tilde f \in \mathcal{F}_{2s}$ such that $\|\tilde f^* - \tilde f\|_\infty \leq C\delta_n\bar \sigma C n^{-1/4}/\log(n)$ 
\end{proof}

\begin{lemma}\label{lemma:inclusion}
Under Assumption \ref{assum:DR}(ii), suppose that $ \sqrt{2\log(2n^2/\delta)} \geq \delta_n/(\lambda - 1)$. Then we have that $\mathbb{P}(  \gamma^* \succeq \hat \gamma ) \geq 1-\delta$.
\end{lemma}
\begin{proof}
With probability $1-\delta$, we have
that $|f(g)-\hat f(g)| \leq \bar{\sigma}\sqrt{ 2\log(2n^2/\delta)/|V_g|}$ uniformly over all $g\in \widehat E_n$. 

Suppose that $g \not\in \gamma^*$. Note that for any $\tilde g $ such that 
$\tilde g \succeq  par(g)$, we have that $r^Z(\tilde g) \leq r^Z(par(g)) \leq \max_{g'' \in \gamma^*} r^Z(g'') \leq \bar{\sigma} \delta_n n^{-1/2}$ by Assumption \ref{assum:DR}(ii). Thus for any  
$g',g'' \in \widehat E_n : g'\succeq \tilde g, g'' \succeq par(\tilde g) $ we have by the triangle inequality 
$$\begin{array}{rl}
|\hat f(g') - \hat f(g'')| & \leq |\hat f(g') - f(g')|+|f(g') - f(g'')|+|f(g'') - \hat f(g'')|\\
& \leq_{(1)}  \bar\sigma\sqrt{2\log(2n^2/\delta)}\left(\frac{1}{\sqrt{|V_{g'}|}} + \frac{1}{\sqrt{|V_{g''}|}}\right) + r^Z(g)\\
& \leq_{(2)}  \bar\sigma\sqrt{2\log(2n^2/\delta)}\left(\frac{1}{\sqrt{|V_{g'}|}} + \frac{1}{\sqrt{|V_{g''}|}}\right) + \delta_n\bar \sigma n^{-1/2}\\
\end{array}$$
%\lambda  \alpha(g')+\lambda \alpha(g'')   $$
where (1) holds with probability $1-\delta$ by Proposition \ref{Good_event}, and (2) holds by Assumption \ref{assum:DR}. Then provided that 
 $\mbox{$2\delta_n\bar \sigma n^{-1/2} \leq (\lambda - 1)\bar\sigma\sqrt{2\log(2n^2/\delta)}\left(\frac{1}{\sqrt{|V_{g'}|}} + \frac{1}{\sqrt{|V_{g''}|}}\right)$} $, 
which is implied by $ (\lambda - 1)\sqrt{2\log(2n^2/\delta)} \geq \delta_n$, we have that $\tilde g$ is pruned based on Algorithm $1$ and thus $\tilde{g}$ is not included in $\hat \gamma$.
\end{proof}

\begin{lemma}\label{lemma:sparsity_gamma_star}
Suppose Assumption \ref{assump:outc_mod} and Assumption \ref{assum:DR} hold. Let $s$ denote the number of partitions induced by $\hat f^* \sim \gamma^*$ defined in Assumption \ref{assum:DR}. We have that 
$$ s \leq (n/\underline{\sigma}^2)\mathbb{E}\left[ \mbox{$\frac{1}{n}\sum_{i=1}^n(f_i-\hat f_i^*)^2$} \mid Z \right]$$
Moreover, with probability $1-\delta$, the estimates from Algorithm 1 satisfy 
$$ \frac{1}{n}\sum_{i=1}^n (f_i-\hat f_i)^2 \leq  \frac{\lambda^2}{(\lambda-1)^2}  \frac{\bar \sigma^2\delta_n^2}{n} + C\lambda^2  \frac{\bar\sigma^2}{\underline{\sigma}^2} \log(2n^2/\delta) \mathbb{E}\left[ \frac{1}{n}\sum_{i=1}^n (\hat f_i^* - f_i)^2 \mid Z\right]  $$ and with probability $1-\delta - \Delta_n$ 
$$  
|f_i - \hat f_i |  \leq \frac{\lambda}{\lambda-1} \bar \sigma \delta_n n^{-1/2} + 3\lambda \frac{\bar\sigma}{\underline{\sigma}}  \sqrt{2\log(2n^2/\delta) \ \mathbb{E}[ | \hat f_i^* - f_i |^2 \mid Z ]} 
%|f_i - \hat f_i |  \leq \frac{\lambda}{\lambda-1} \bar \sigma \delta_n n^{-1/2} + 3\lambda \frac{\bar\sigma}{\underline{\sigma}}  \frac{\sqrt{2\log(2n^2/\delta)}}{\log n} \bar{\sigma} \delta_n  n^{-1/4} 
 $$

\end{lemma}
\begin{proof}
We recall some notation for convenience. In what follows we denote $\bar V_g := \{ i \in [n] : g = \gamma^*(G_i^Z)\}$ and $ V_g := \{ i \in [n] : g = \gamma^*(G_i^Z), Z_i=0\}$. By definition, for any $g \in \widehat E_n$ we have $\hat f(g) = \frac{1}{|V_g|}\sum_{i \in V_g}Y_i$. The estimator defined by Algorithm 1 is $\hat f_i = \hat f(g)$ where $g=\gamma_0(G_i^Z(\hat m_i))$ and the (near oracle) estimator defined in Assumption \ref{assum:DR}.(ii) is defined as $\hat f_i^* = \hat f(g)$ where $g=\gamma^*(G_i^Z) = \gamma_0(G_i^Z(m_i^*))$. The true interference function is denoted by $f_i = f(\gamma(G_i^Z))$, and we denote $f_g := \frac{1}{|\bar V_g|}\sum_{i\in\bar V_g}f_i$ and $\bar f_g := \frac{1}{|\bar V_g|}\sum_{i\in\bar V_g}f_i$ and $\hat f^*_g = \hat f(g)$.

We have that 
\begin{equation}\label{eq:main_sparse:1}  \frac{1}{n}\sum_{i=1}^n(f_i-\hat f_i^*)^2 =  \frac{1}{n}\sum_{g \in \gamma^*}\sum_{i \in \bar V_g}  (f_i-\hat f_g^*)^2    
\geq \frac{1}{n}\sum_{ g \in \gamma^*} |\bar V_g|  ( \bar f_g^*-\hat f_g^*)^2  \end{equation} 
where the second inequality follows from convexity and $\bar f_g^*= \frac{1}{|\bar V_g|}\sum_{i \in \bar V_g} f_i$.

Next note that for any $g \in \widehat{E}_n$ such that $g=\gamma^*(G_i^Z)$, we have that 
$$\begin{array}{rl}
( \bar f_g -\hat f_g^*)^2 & = ( \frac{1}{|\bar V_g|}\sum_{i \in \bar V_g} f_i - \frac{1}{|V_g|}\sum_{i \in V_g} Y_i)^2\\
& = (\bar f_g - f_g)^2+ 2(\bar f_g - f_g)\frac{\sum_{i \in V_g}\epsilon_i}{|V_g|} + \frac{(\sum_{i \in V_g}\epsilon_i)^2}{| V_g|^2} \\
\end{array}
$$

So that 
\begin{equation}\label{eq:lb}\begin{array}{rl}
\mathbb{E}[( \bar f_g-\hat f_g^*)^2 \mid Z]  & \geq \underline{\sigma}^2/| V_g| \\
\end{array}\end{equation}
%\mathbb{E}[( \bar f_g-\hat f_g)^4 \mid Z] &\leq C (\mathbb{E}[|V_g|( \bar f_g-\hat f_g)^2 \mid Z])^2
where the last bound follows from $\mathbb{E}[\epsilon_i^2\mid Z,X] \geq \underline{\sigma}^2$ and $\mathbb{E}[\epsilon_i\epsilon_j\mid Z]=0$ by Assumption \ref{assump:outc_mod}.

Thus we have
$$\begin{array}{rl}
\mathbb{E}\left[ \frac{1}{n}\sum_{i=1}^n(f_i-\hat f_i^*)^2 \mid Z \right] & 
\geq_{(1)}  \mathbb{E}\left[\frac{1}{n}\sum_{g \in \gamma^*} |\bar V_g|(\bar f_g-\hat f_g^*)^2 \mid Z \right] \\
&\geq_{(2)}   \frac{1}{n}\sum_{g \in \gamma^*} |\bar V_g|\mathbb{E}\left[(\bar f_g-\hat f_g^*)^2 \mid Z \right] \\
& \geq_{(3)} \underline{\sigma}^2 \frac{1}{n}\sum_{g \in \gamma^*} \frac{|\bar V_g|}{|V_g|} \geq_{(4)} \underline{\sigma}^2 s / n \\
\end{array}
$$ 
where (1) by  (\ref{eq:main_sparse:1}), (2) follows since $|\bar V_g|$ is known given $Z$, (3) by the lower bound (\ref{eq:lb}), and (4) by $|\bar V_g| \geq |V_g|$ by definition. Therefore we have
\begin{equation}\label{eq:refref} \mbox{$\underline{\sigma}^2 \frac{s}{n} \leq   \underline{\sigma}^2 \frac{1}{n}\sum_{g \in \gamma^*} \frac{|\bar V_g|}{|V_g|} \leq  
\mathbb{E}\left[ \frac{1}{n}\sum_{i=1}^n(f_i-\hat f_i^*)^2 \mid Z\right]$}  \end{equation} %\leq   \delta_n n^{-1/2}/\log n  $$
%so that $ s \leq  \underline{\sigma}^{-2} \delta_n n^{1/2}/\log n $. %(Indeed, to maximize the number of leaves, we need to make each leaf with nearly the same number of observations.) Moreover, since for every realization of the data there are at most $n^2$ nodes, the number of possibly functions on each realization of the data is bounded by $n^{\frac{4\delta_nn^{1/2}}{\log n}}$

To show the second part, for each node $i \in [n]$ let $g_i=\gamma^*(G_i^Z)$, the pattern associated with $\gamma^*$. Then with probability $1-\delta-\Delta_n$, for all $i\in[n]$, we have
\begin{equation}\label{eq:uniform:cor6:DR}\begin{array}{rl}
\ |f_i - \hat f_i |  & \leq_{(1)} \frac{\lambda}{\lambda-1} r^Z(g_i) + 3\lambda \bar\sigma\sqrt{\frac{2\log(2n^2/\delta)}{|V_{g_i}|}}\\
& \leq_{(2)} \frac{\lambda}{\lambda-1} \bar \sigma \delta_n n^{-1/2} + 3\lambda \bar\sigma\sqrt{\frac{2\log(2n^2/\delta)}{|V_{g_i}|}} \\
\end{array}
\end{equation}
where (1) holds by Corollary \ref{cor:second_context_tree}, and (2) holds by Assumption \ref{assum:DR} since $g_i\in \gamma^*$.

Thus we have that 
$$\begin{array}{rl}
\frac{1}{n}\sum_{i=1}^n|f_i - \hat f_i |^2 & \leq C_\lambda \frac{\delta_n^2\bar \sigma^2}{n} + C\lambda^2 \bar\sigma^2 \frac{2}{n}\sum_{i=1}^n\frac{\log(2n^2/\delta)}{|V_{g_i}|}\\
& = C_\lambda \frac{\delta_n^2\bar \sigma^2}{n} + C\lambda^2 \bar\sigma^2 \frac{2\log(2n^2/\delta)}{n}\sum_{g \in \gamma^*}\frac{|\bar{V}_{g}|}{|V_{g}|}\\
& \leq C_\lambda \delta_n^2\bar \sigma^2 n^{-1} + 2C\lambda^2 \frac{\bar\sigma^2}{\underline{\sigma}^2} \log(2n^2/\delta)\mathbb{E}\left[\frac{1}{n}\sum_{j=1}^n (\hat f_i^* - f_i)^2 \mid Z\right] \\
\end{array}$$
where the last step follows from the second inequality in (\ref{eq:refref}). Similarly, because
\begin{equation}\label{eq:aux:uniform:cor6} \mathbb{E}[ (f_i - \hat f_i^*)^2 \mid Z] \geq \underline{\sigma}^2 / |V_{g_i}|\end{equation}
we have with probability $1-\delta -\Delta_n$
$$\begin{array}{rl} 
|f_i - \hat f_i | & \leq \frac{\lambda}{\lambda-1} \bar \sigma  \delta_n n^{-1/2} + 3\lambda \bar\sigma\sqrt{\frac{2\log(2n^2/\delta)}{|V_{g_i}|}}\\ 
& \leq \frac{\lambda}{\lambda-1} \bar \delta_n \sigma n^{-1/2} + 3\lambda \frac{\bar\sigma}{\underline{\sigma}}\sqrt{2\log(2n^2/\delta) \mathbb{E}[(f_i-\hat f_i^*)^2\mid Z]}\\
& \leq \frac{\lambda+2}{\lambda-1} \delta_n\bar \sigma n^{-1/2} + 3\lambda \frac{\bar\sigma^2}{\underline{\sigma}}\sqrt{2\log(2n^2/\delta)} \frac{ \delta_n n^{-1/4}}{\log(n)} 
\end{array}$$
where the first inequality holds by (\ref{eq:uniform:cor6:DR}), the second holds  by (\ref{eq:aux:uniform:cor6}), and the third inequality by Assumption \ref{assum:DR}(ii). 
%since $|V_{g_i}| \geq 1$ and $\sum_{i=1}^n |V_{g_i}| = n - N_1$ by Lemma \ref{lemma:aux_comb}
\end{proof}

\begin{proof}[Proof of Corollary \ref{cor:var}]
Let $N_1 := \sum_{i=1}^n Z_i$, 
$$
\begin{array}{l}
\tilde \psi_i = Z_i(Y_i-\hat\tau^{AIPW} - \hat f_i)-(Y_i-\hat f_i)(1-Z_i)\hat e_i/(1-\hat e_i)\\
\tilde \psi_i^* = Z_i(Y_i-\hat\tau^{AIPW} -  f_i)-(Y_i- f_i)(1-Z_i) e_i/(1-e_i), \mbox{ and } \\ 
\psi_i = Z_i(Y_i-\tau_i - f_i)-(Y_i- f_i)(1-Z_i) e_i/(1-e_i).\end{array}$$ We have that
$$
\begin{array}{rl}
&\Sigma_n^{1/2} - \widetilde{\Sigma}_n^{1/2}  = \Sigma_n^{1/2} - \sqrt{\frac{1}{N_1} \sum_{i=1}^n (\tilde \psi_i)^2} \\
& \leq \Sigma_n^{1/2} - \sqrt{\frac{1}{N_1} \sum_{i=1}^n (\tilde \psi_i^*)^2}  + \sqrt{\frac{1}{N_1} \sum_{i=1}^n (\tilde \psi_i - \tilde \psi_i^*)^2 } \\ & = \Sigma_n^{1/2} - \sqrt{\frac{1}{N_1} \sum_{i=1}^n (\psi_i+Z_i\{\tau_i-\hat\tau^{AIPW}\})^2}  + \sqrt{\frac{1}{N_1} \sum_{i=1}^n (\tilde \psi_i - \tilde \psi_i^*)^2 } \\
& \leq \Sigma_n^{1/2} - \sqrt{\frac{1}{N_1} \sum_{i=1}^n (\psi_i+Z_i\{\tau_i-\tau\})^2} + |\tau-\hat \tau^{AIPW}|  + \sqrt{\frac{1}{N_1} \sum_{i=1}^n (\tilde \psi_i - \tilde \psi_i^*)^2 } \\
& \leq \Sigma_n^{1/2} - \sqrt{\frac{1}{N_1} \sum_{i=1}^n (\psi_i+Z_i\{\tau_i-\tau\})^2} + C \delta_n
\end{array}
$$
since $|\tau-\hat \tau^{AIPW}|\leq C\delta_n$ with probability $1-o(1)$ by Theorem \ref{thm:dra_tau_inf}, and 
$$ 
\begin{array}{rl}
\sqrt{\frac{1}{N_1} \sum_{i=1}^n (\tilde \psi_i - \tilde \psi_i^*)^2 } & \leq_{(1)} C\sqrt{\frac{1}{N_1} \sum_{i=1}^n (f_i-\hat f_i)^2} \\
& +C\sqrt{\frac{1}{N_1} \sum_{i=1}^n (Y_i-f_i)^2(1-Z_i)\left( \frac{\hat e_i}{1-\hat e_i}-\frac{e_i}{1- e_i}\right)^2} \\
& \leq_{(2)}  C\delta_n + C\bar \sigma \sqrt{\log n} \sqrt{\frac{1}{N_1} \sum_{i=1}^n \left( \frac{\hat e_i}{1-\hat e_i}-\frac{e_i}{1- e_i}\right)^2} \leq_{(3)} C\delta_n \\
\end{array}
$$
where (1) holds by triangle inequality, (2) follows with probability $1-o(1)$ by Lemma %\href{lemma:sparsity_gamma_star}{10}
\ref{lemma:sparsity_gamma_star} 
and by $\max_{i\in[n]} |Y_i-f_i|(1-Z_i) \leq C\bar\sigma\sqrt{\log n}$ holding with probability $1-o(1)$ by  $\epsilon_i$, $i\in[n]$, being independent zero mean subgaussian random variables by Assumption \ref{assump:outc_mod}, and (3) from Assumption \ref{assum:DR}(i). Finally
$$\begin{array}{rl} 
\frac{1}{N_1} \sum_{i=1}^n (\psi_i+\{\tau_i-\tau\})^2 & = \frac{1}{N_1} \sum_{i=1}^n \psi_i^2+\frac{1}{N_1} \sum_{i=1}^nZ_i\{\tau_i-\tau\}^2 + \frac{2}{N_1}\sum_{i=1}^n(\tau_i-\tau)\psi_i\\
& \geq \frac{1}{N_1} \sum_{i=1}^n \psi_i^2+\frac{1}{N_1} \sum_{i=1}^nZ_i\{\tau_i-\tau\}^2 - C \delta_n\\
& \geq \frac{1}{N_1} \sum_{i=1}^n \psi_i^2 - C \delta_n = \Sigma_n + (\widehat \Sigma_n - \Sigma_n  ) - C \delta_n\\
\end{array}$$
with probability $1-o(1)$ since $\mathbb{E}[(\tau_i-\tau)\psi_i\mid X]=0$, $|\tau_i-\tau|\leq C$ and $\psi_i$ is sub-Gaussian.

Thus, the first and second results follow from $|\widehat \Sigma_n - \Sigma_n | \leq C\delta_n $. Since $\psi_i$ are sub-Gaussian and independent, the result follows.
\end{proof}

\subsection{Proofs of Section \ref{sec:OutcomeRegression}}

\begin{proof}[Proof of Theorem \ref{ora_tau_inf}]
Let $f_i := f(\gamma(G_i^Z))$. We have
\begin{equation}\label{eq:or:main_main}
\begin{split}
\left\lvert \hat{\tau}^{OR}-\tau  \right\rvert & =_{(1)}\Bigg| \frac{\mbox{$\sum_{i=1}^n Z_i ( Y_i - \hat f_i )$} }{\mbox{$\sum_{i=1}^n Z_i$} }-\tau  \Bigg|  =_{(2)}\left\lvert \frac{\mbox{$\sum_{i=1}^n Z_i ( \widetilde{Y}_i(1,Z_{-i}) - \hat f_i )$}  }{\mbox{$\sum_{i=1}^n Z_i$} }-\tau  \right\rvert \\
& =_{(3)}\left\lvert \frac{1}{\mbox{$\sum_{i=1}^n Z_i$}}\sum_{i=1}^n Z_i \left( \tau_i + f_i + \epsilon_i - \hat f_i \right) -\tau  \right\rvert \\
& \leq_{(4)}  \frac{1}{\mbox{$\sum_{i=1}^n Z_i$} }  \left\lvert  \sum_{i=1}^n Z_i   \left( f_i -  \hat{f}_i    \right) \right\rvert+ \frac{1}{\mbox{$\sum_{i=1}^n Z_i$} } \left\lvert \sum_{i=1 }^n Z_i   \epsilon_i  \right\rvert \\
\end{split}
\end{equation}
where (1) holds by definition of the estimator, (2) holds by the potential outcomes definition, (3) holds by Assumption \ref{assump:outc_mod}, and (4) by the triangle inequality. 
To bound the second term in the last line of (\ref{eq:or:main_main}), by $\epsilon_i$, $i\in[n]$, being independent zero mean subgaussian random variables by Assumption \ref{assump:outc_mod} so that

\begin{equation}\label{OR:ATT:part1}
\begin{array}{rl}
\mathbb{P}\left( \left\lvert \frac{\sum_{i=1}^n Z_i \epsilon_i }{\sum_{i=1}^n Z_i} \right\rvert >t \right) & = \mathbb{E}\left[\mathbb{P}\left( \left| \frac{\sum_{i=1}^n Z_i \cdot  \epsilon_i  }{\sum_{i=1}^n Z_i} \right| >t \mid Z \right)\right] \\
& \leq 2\mathbb{E}\left[\exp\left(-\frac{t^2 (\sum_{i=1}^n Z_i)^2}{2 \sum_{i =1 }^n Z_i^2 \bar\sigma^2 }\right)\right]\\
&\leq 2\mathbb{E}\left[\exp\left(-\frac{t^2 (\sum_{i=1}^n Z_i^2)}{2 \bar\sigma^2 } \right)\right] \leq \delta    
\end{array} 
\end{equation} 
by taking $t=t(Z)=\bar \sigma \sqrt{2\log(2/\delta)/\sum_{i=1}^n Z_i}$ since $Z_i\in\{0,1\}$ for $i \in [n]$. Moreover, \begin{equation}\label{OR:ATT:part2}\mathbb{P}\left( \sum_{i=1}^n Z_i < nc/2 \right) = \mathbb{E}_X\left[ \mathbb{P}\left( \sum_{i=1}^n Z_i < nc/2 \mid X \right) \right]  \leq \exp(- n c^2/8 )  \end{equation}
since $\mathbb{P}(Z_i=1 \mid X ) \geq c$ and the components of $Z$ being conditionally independent by Assumption \ref{cond_unconf}.(ii). %{\color{red} [If relaxing (ii) in Assumption \ref{cond_unconf}, we may refer to the concentration for dependent Bernoulli random variables in \cite{lampert2018dependency}.]}

To bound the first term in the last line of (\ref{eq:or:main_main}), by Theorem \ref{thm:single_node}, with probability $1-\delta$, for all $i\in[n]$, we have
\begin{equation}\label{OR:ATT:part3}
   |f_i - \hat f_i| \leq 3\lambda \min_{k \in [n], g_k=\gamma_0(G_i^Z(k))} \frac{r^Z(g_k)}{\lambda-1}+\bar \sigma \sqrt{\frac{2\log(2n^2/\delta)}{|V_{g_k}|}} \end{equation}

The result follows by combining (\ref{OR:ATT:part1}), (\ref{OR:ATT:part2}), and (\ref{OR:ATT:part3}) into  (\ref{eq:or:main_main}).
\end{proof}

\section{Proof of Examples}

\begin{proof}[Proof of Example 4]
		For each fixed $g_k$, $|V_{g_k}|=\mbox{$\sum_{i=1}^n$} 1\{\gamma_0(G^Z_i(k)) =g_k \}$. Based on the graph and treatment generating processes,
		\begin{equation*}
		\begin{split}
		& \mathbb{E}\left(\mbox{$\sum_{i=1}^n$} 1\{\gamma_0(G^Z_i(k)) =g_k \} \right)= \mbox{$\sum_{i=1}^n$} \mathbb{P}\left( \gamma_0(G_i^Z(k))=g_k  \right) \\
		& = \mbox{$\sum_{i=1}^n$} \left(1/2\right)^{1+\sum_{j=1}^{k} jd} = n \cdot \left( 1/2 \right)^{1+\frac{d(k+1)k}{2}} \\
		\end{split}
		\end{equation*}
		Let a realized $z=(z_1,\cdots,z_i,\cdots,z_n)$. The treatment vector of which only $u$-th element is different from $z$ is defined as $z_{(u)}:=(z_1,\cdots,1-z_u,\cdots,z_n)$. Then for $\forall u \in [n]$, 
	\begin{small}
        \begin{equation*}
		\begin{split}
		& \sup_{z \in \{0,1\}^n} \left\lvert \mbox{$\sum_{i=1}^n$} 1\{\gamma_0(G^z_i(k)) =g_k \}-\mbox{$\sum_{i=1}^n$} 1\{\gamma_0(G^{z_{(u)}}_i(k)) =g_k \} \right\rvert  \leq  \mbox{$\sum_{j=1}^k$} jd = 2^{-1} d(k+1)k   \\
		\end{split}
		\end{equation*}
            \end{small}
		where the inequality is obtained from counting the number of units whose pattern may change because of altering $u$-th component. Then 
		%for given $k$, 
		by utilizing McDiarmid's inequality in \cite{mcdiarmid1989method}, we have  
		\begin{equation*}
		\begin{split}
		\mathbb{P}\left(|V_{g_k}|-\mathbb{E}(|V_{g_k}|) \geq -t  \right) \leq  \exp\left(-{2t^2}/\mbox{$\sum_{i=1}^n$} d^2 (k+1)^2k^2/4\right)
		\end{split}
		\end{equation*}
		Let $g_{i,k}:=\gamma_0(G^\mathbf{Z}_{i}(k))$ for $i\in [n]$, i.e., $g_{i,k}$ is the pattern induced by specific $\mathbf{Z}$ and $k$-hop neighborhood from unit $i$. Then by taking $t=n\cdot \left( \frac{1}{2} \right)^{2+\frac{d(k+1)k}{2}}$,
		\begin{equation}
		\label{lower_bound_V_g}
		\begin{split}
		& \mathbb{P}\left( \cap_{i=1}^n \left\lbrace  |V_{g_{i,k}}|-\mathbb{E}(|V_{g_{i,k}}|)  \geq -t \right\rbrace  \right) \leq \sum_{i=1}^n \mathbb{P}\left(   |V_{g_{i,k}}| \geq n \cdot \left( 1/2 \right)^{2+\frac{d(k+1)k}{2}} \right) \\
		& \leq n \cdot \exp\left( - \frac{8n (1/2)^{4+d(k+1)k}}{d^2(k+1)^2k^2}     \right)= \exp \left( \log n - \frac{n}{2^{1+d(k+1)k} \cdot d^2(k+1)^2 k^2 } \right) 
		\end{split}
		\end{equation}
		Therefore, based on Theorem \ref{thm:single_node} and formula \eqref{lower_bound_V_g}, 
		%, and by taking $d=4$, 
		\begin{equation*}
		\begin{split}
		\max_{i\in [n]} |f(G^\mathbf{Z}_{i})-\hat{f}_i| & \leq \max_{i \in [n]} \  \frac{\lambda+2}{\lambda-1} r^\mathbf{Z}(g_{i,m_0})+ 3\lambda \bar \sigma \sqrt{\frac{2\log({2n^3}/{\delta})}{|V_{g_{i, m_0}}|}} \\
		& \leq \max_{i\in [n]} \ 0+3\lambda \bar{\sigma} \cdot \sqrt{ \frac{2\log ({2 n^3}/{\delta})}{|V_{g_{i,m_0}}|}} \leq  3\lambda \bar{\sigma} 2^{\frac{3}{2}+\frac{d}{4}(k+1)k} \cdot \sqrt{ \frac{\log ({2 n^3}/{\delta})}{n}}  \\
		\end{split}
		\end{equation*}
		with probability at least $1-\exp \left( \log n - \frac{n}{2^{1+d(k+1)k} \cdot d^2(k+1)^2 k^2 } \right) -\delta$. The final result is obtained by taking $k=m_0$. 
	\end{proof}

\begin{proof}[Proof of Example \ref{example_exp_decay_intf}]
Based on Algorithm $1$, we take $\lambda=2$ in the following proof. To obtain the upper bound in the statement, we have
\begin{equation*}
	\begin{split}
	& \max_{i\in [n]}  |f(G^Z_{i})-\hat{f}_i|  \leq_{(1)} \max_{i \in [n]}  \frac{\lambda}{\lambda-1} r^Z(g_{i,k})+ 3\lambda \bar \sigma \sqrt{\frac{2\log({2n^2}/{\delta})}{|V_{g_{i, k}}|}} \\
	%\ \min_{k \in [n], g_k=\gamma_0 \left( G^Z_i(k) \right) }
	& \leq_{(2)} \frac{\lambda}{\lambda-1}  C^{\prime} \exp(-4k^2)  +3\lambda \bar{\sigma} 2^{\frac{3}{2}+\frac{d}{4}(k+1)k} \cdot \sqrt{ \frac{\log ({2 n^2}/{\delta})}{n}} \\
	& \leq_{(3)} \frac{1}{2} \cdot \frac{\lambda C^\prime \bar{\sigma}^2}{\lambda-1} \cdot \frac{1}{n^{\frac{1}{4}}\log n}   + 3\lambda \bar{\sigma} 2^{\frac{3}{2}+\frac{d}{4}(k+1)k} \cdot \sqrt{ \frac{\log ({2 n^2}/{\delta})}{n}} \leq_{(4)} \frac{\lambda C^\prime \bar{\sigma}^2}{\lambda-1} \cdot \frac{1}{n^{\frac{1}{4}}\log n}
	\end{split}
\end{equation*}
$(1)$ is based on Theorem \ref{thm:single_node} and it holds for each $k \in [n]$. In steps $(2)-(4)$, we will seek for sufficient conditions of $k$ such that the bounds can reach desirable accuracy as in $(4)$. $(2)$ is based on the upper bound of approximation error in the statement and the lower bound of $|V_{g_k}|$ in Example \ref{example_threshold_intf} with probability at least $1-\delta - n\exp \left(- \frac{n}{d^2(k+1)^2 k^2 \cdot 2^{1+d(k+1)k} } \right)$. We then derive the sufficient condition of $k$ such that $\frac{n}{d^2(k+1)^2 k^2 \cdot 2^{1+d(k+1)k}} \geq 2\log n$. Note $\frac{n}{d^2(k+1)^2 k^2 \cdot 2^{1+d(k+1)k}} \geq \frac{n}{2^{1+d(k+1)k+\log(6d^2)\cdot k}}$ for $k\geq 0$. Then by setting $\frac{n}{2^{1+d(k+1)k+\log(6d^2)\cdot k}} \geq 2 \log n$, we have 
\begin{equation*}
    \begin{split}
     & 0 \leq  k \leq -\frac{1}{2}-\frac{1}{2d}\log(6d^2)+\sqrt{ \left[\frac{1}{2}+\frac{1}{2d}\log(6d^2) \right]^2+\log\left(\frac{n^{1/d}}{2^{1/d}(\log n)^{1/d}}\right)-\frac{1}{d}} \\
     %& = -\frac{1}{2}-\frac{1}{8}\log(6d^2)+\sqrt{ \left[\frac{1}{2}+\frac{1}{2d}\log(6d^2) \right]^2+\log\left(\frac{n^{1/d}}{2^{1/d}(\log n)^{1/d}}\right)-\frac{1}{d}} 
    \end{split}
\end{equation*}
%by plugging in $d=4$. 
%Therefore, by setting $0 \leq k \leq \sqrt{ \log \left(  \frac{n^{1/4}}{2^{1/4} (\log n)^{1/4}   } \right)  }-\frac{11}{10}$, 
Under this condition, with probability at least $1-\delta-\frac{1}{n}$, the lower bound of $|V_{g_{i,k}}|$ holds uniformly across different units. Furthermore, the necessary condition of $n$ such that $0 \leq -\frac{1}{2}-\frac{1}{2d}\log(6d^2)+\sqrt{ \left[\frac{1}{2}+\frac{1}{2d}\log(6d^2) \right]^2+\log\left(\frac{n^{1/d}}{2^{1/d}(\log n)^{1/d}}\right)-\frac{1}{d}}$ is $\frac{n}{(\log n)}\geq 2e$. Therefore, we take $n \geq 15$.
%$$\frac{n}{2^{1+d(k+1)^2+\log(6d^2)\cdot k}}>2\log n.$$
%Correspondingly, the probability to take the lower bound of $|V_{gi,k}|$ is at least $1-\delta-\frac{1}{n}$. 

For deriving $(3)$, the sufficient condition of $k$ such that $\frac{\lambda}{\lambda-1} C^\prime \exp(-4k^2) \leq \frac{1}{2} \cdot \frac{\lambda C^\prime \bar{\sigma}^2}{\lambda-1} \cdot \frac{1}{n^{\frac{1}{4}}\log n}$ is $k \geq \sqrt{ \log \left(  \frac{2^{1/4}}{\bar{\sigma}^{1/2}} \cdot n^{1/16}\cdot (\log n)^{1/4} \right) }$. We next derive the sufficient condition of $k$ for $(4)$. By considering $3\lambda \bar{\sigma} 2^{\frac{3}{2}+\frac{d}{4}(k+1)k} \cdot \sqrt{ \frac{\log (\frac{2 n^2}{\delta})}{n}}\leq  \frac{1}{2} \cdot \frac{\lambda C^\prime \bar{\sigma}^2}{\lambda-1} \cdot \frac{1}{n^{\frac{1}{4}}\log n}$ where $\delta= \frac{1}{n}$, we have $0\leq k \leq -\frac{1}{2}+\sqrt{\frac{1}{4}+ \log\left(  \frac{C^\prime \bar{\sigma}}{ 3 (\lambda-1) 2^{5/2} }   \frac{n^{1/d}}{(\log n)^{4/d} (\log (2n^3))^{2/d} } \right)}$. Such range of $k$ leads to the necessary condition of $n$. That is, $\frac{n^{1/d}}{(\log n)^{4/d} \cdot (\log (2n^3))^{2/d}} \geq \frac{3(\lambda-1)2^{5/2}}{C^\prime \bar{\sigma}}$, i.e., $n\geq 2$ by plugging in the lower bound of $C^\prime$, $\lambda=2$ and $\bar{\sigma}^2=1$.  
% which satisfies $0 \leq \sqrt{\log \left( \frac{C^\prime \bar{\sigma}  }{ 3^{\frac{3}{2}}\cdot 2^{\frac{5}{2}} (\lambda-1) } \cdot \frac{n^{\frac{1}{4}}}{ \log^{\frac{3}{2}}(2n) } \right)}-1$ as $\frac{n^{1/4}}{\log^{3/2}(2n)} \geq \frac{2^{\frac{5}{2}}e3^2(\lambda-1)}{C^\prime \bar{\sigma}}$. Thus this leads to $n \geq 10^2$ by plugging in $\lambda=2, C^\prime \geq 3^{3/2} \cdot 2^{9/2} e$. 
Next we let $$a:= \sqrt{ \log \left(  \frac{2^{1/4}}{\bar{\sigma}^{1/2}} \cdot n^{1/16}\cdot (\log n)^{1/4} \right)}$$ $$b:=-\frac{1}{2}-\frac{1}{2d}\log(6d^2)+\sqrt{ \left[\frac{1}{2}+\frac{1}{2d}\log(6d^2) \right]^2+\log\left(\frac{n^{1/d}}{2^{1/d}(\log n)^{1/d}}\right)-\frac{1}{d}}$$ and 
$$c:=-\frac{1}{2}+\sqrt{\frac{1}{4}+ \log\left(  \frac{C^\prime \bar{\sigma}}{ 3 (\lambda-1) 2^{5/2} }   \frac{n^{1/d}}{(\log n)^{4/d} (\log (2n^3))^{2/d} } \right)}.$$ 
Then we have, by taking $a\leq k \leq \min(b,c)$, with probability at least $1-\frac{2}{n}$,
\begin{equation*}
	\begin{split}
	& \max_{i\in [n]} |f(G^Z_{i})-\hat{f}_i| \leq \frac{\lambda C^\prime \bar{\sigma}^2}{\lambda-1} \cdot \frac{1}{n^{\frac{1}{4}}\log n}. 
	\end{split}
\end{equation*}
Now we obtain the sufficient conditions of $n$ such that $a \leq b$. Note $b^2$ can be lower bounded by 
\begin{small}
\begin{equation*}
    \begin{split}
       b^2 & = \frac{1}{d} \log \left( \frac{n}{ 2 \log n} \right) - \frac{1}{d}+  \frac{1}{2} \left(1+\frac{1}{d} \log (6d^2) \right)^2 \\
       & - \left(1+\frac{1}{d} \log(6d^2) \right)\cdot \sqrt{ \frac{1}{4}\left( 1+ \frac{1}{d} \log(6d^2) \right)^2 + \frac{1}{d} \log \left( \frac{n}{2 \log n} \right) -\frac{1}{d} } \\
       & \geq \frac{1}{d} \log \left( \frac{n }{2 \log n}  \right)- \frac{1}{d} - \left(1+\frac{1}{d} \log(6d^2) \right)\cdot \sqrt{\frac{1}{d} \log \left( \frac{n}{2 \log n}  \right) -\frac{1}{d}   }   \geq_{(1)} \frac{1}{2d} \left[  \log \left( \frac{n}{2 \log n} \right)- 1 \right]
       %& \geq \frac{1}{d} \log \left( \frac{n }{2 \log n}  \right)- \frac{1}{d} - \left(1+\frac{1}{d} \log(6d^2) \right)\cdot \sqrt{\frac{1}{d} \log \left( \frac{n}{2 \log n}  \right) -\frac{1}{d}   }   \geq_{(1)} \frac{1}{2} \left( \frac{1}{d} \log \left( \frac{n}{2 \log n} \right)- \frac{1}{d}  \right) 
    \end{split}
\end{equation*}
\end{small}
$(1)$ is by taking $n\geq 5 \cdot 10^9, \lambda =2$ and $\bar{\sigma}^2=1$ such that $$ \left(1+\frac{1}{d} \log(6d^2) \right)\cdot \sqrt{\frac{1}{d} \log \left( \frac{n}{2 \log n}  \right) -\frac{1}{d} }\leq \frac{1}{2} \left( \frac{1}{d} \log\left( \frac{n}{2 \log n}   \right) - \frac{1}{d}    \right). $$
Then by setting $a^2 \leq \frac{1}{2} \left( \frac{1}{d} \log \left( \frac{n}{2 \log n} \right)- \frac{1}{d}  \right) $, we obtain the sufficient condition of $n$ for this bound. That is, $\frac{n^{1/16}}{(\log n)^{3/8}} \geq (2^3e)^{1/8}$. Therefore, we take $n\geq 10^{12}$, which satisfies this sufficient condition.

Similarly, we derive the sufficient condition of $n$ such that $a \leq c$. Note the lower bound of $c^2$ is 
\begin{small}
\begin{equation*}
    \begin{split}
      c^2  & = \frac{1}{2} + \log\left(  \frac{C^\prime \bar{\sigma}}{ 3 (\lambda-1) 2^{5/2} }   \frac{n^{1/d}}{(\log n)^{4/d} (\log (2n^3))^{2/d} } \right) \\ & -\sqrt{\frac{1}{4}+ \log\left(  \frac{C^\prime \bar{\sigma}}{ 3 (\lambda-1) 2^{5/2} }   \frac{n^{1/d}}{(\log n)^{4/d} (\log (2n^3))^{2/d} } \right)} \\ 
      & \geq \log\left(  \frac{C^\prime \bar{\sigma}}{ 3 (\lambda-1) 2^{5/2} }   \frac{n^{1/d}}{(\log n)^{4/d} (\log (2n^3))^{2/d} } \right)\\
      & - \sqrt{\log\left(  \frac{C^\prime \bar{\sigma}}{ 3 (\lambda-1) 2^{5/2} }   \frac{n^{1/d}}{(\log n)^{4/d} (\log (2n^3))^{2/d} } \right)}\\
      & \geq_{(1)} \frac{1}{2} \log\left(  \frac{C^\prime \bar{\sigma}}{ 3 (\lambda-1) 2^{5/2} }   \frac{n^{1/d}}{(\log n)^{4/d} (\log (2n^3))^{2/d} } \right)
    \end{split}
\end{equation*}
\end{small}
$(1)$ is by taking $n \geq 1.5 \cdot 10^7$ such that, by plugging in the values of $C^\prime, \bar{\sigma}^2$ and $\lambda$
\begin{small}
$$ \sqrt{\log\left(  \frac{C^\prime \bar{\sigma}}{ 3 (\lambda-1) 2^{5/2} }   \frac{n^{1/d}}{(\log n)^{4/d} (\log (2n^3))^{2/d} } \right)}  \leq \frac{1}{2} \log\left(  \frac{C^\prime \bar{\sigma}}{ 3 (\lambda-1) 2^{5/2} }  \frac{n^{1/d}}{(\log n)^{4/d} (\log (2n^3))^{2/d} } \right).$$
\end{small}
Then by setting $a^2 \leq \frac{1}{2} \log\left(  \frac{C^\prime \bar{\sigma}}{ 3 (\lambda-1) 2^{5/2} } \cdot   \frac{n^{1/d}}{(\log n)^{4/d} (\log (2n^3))^{2/d} } \right)$, we have the sufficient condition $n \geq 2$. 
%The necessary condition for $n$ such that $a\leq c$ is $\frac{n^{1/4}}{\log^{3/2}(2n)} \geq \frac{3^{3/2}\cdot 2^{9/2}(\lambda-1)}{C^\prime \bar{\sigma}}$, i.e., $ \frac{n^{1/4}}{\log^{3/2}(2n)} \geq \frac{1}{\sqrt{2}e}$, and $\frac{n^{1/16}(\log n)^{1/4}}{(\log 2n)^{3/4}}\geq \frac{3^{3/4} \cdot 2^{3/2}(\lambda-1)}{ (C^\prime \bar{\sigma})^{1/2}}$, i.e., $\frac{n^{1/16}(\log n)^{1/4}}{(\log 2n)^{3/4}}\geq \frac{1}{4 \cdot \sqrt{e}}$. The intersection of the range of $n$ is $n \geq 10^3$. 
%Furthermore, the necessary condition for $n$ such that $c\leq b$ is $\log\left( \frac{C^\prime \bar{\sigma}}{3^{3/2}\cdot 2^{9/4} (\lambda-1)} \cdot \frac{(\log n)^{1/4}}{\log^{3/2}(2n)}    \right) \leq -\frac{1}{5} \sqrt{\log \left(\frac{n^{1/4}}{2^{1/4}(\log n)^{1/4} } \right) }$. Let $n_2$ be the smallest $n$ satisfying this inequality. 
Therefore, combining all conditions of $n$, by setting $n \geq n_0$ where $n_0\geq 10^{12}$, we have with probability at least $1-\frac{2}{n}$, $$\max_{i\in[n]} |f(G_i^Z)- \hat f_i |  \leq   \frac{\lambda C^\prime \bar{\sigma}^2}{\lambda-1} \cdot \frac{1}{n^{\frac{1}{4}}\log n}.$$
%Then by setting $n_0=\max(n_1, n_2, e^{8.2},n_3)$, we have when $n>n_0$, with probability at least $1-\frac{2}{n}$, the conclusion holds.
\end{proof}

\section{Technical Results for Randomized Treatments}
The proof of Proposition \ref{clt_delta} relies on  Lemmas \ref{exp_ratio_ratio_exp}, \ref{clt_numerator}, \ref{clt_denominator} and \ref{multi_clt}.
\begin{lemma}
\label{exp_ratio_ratio_exp}
Under Assumptions \ref{ass:cond_ind} and \ref{assump:outc_mod} we have 
%the difference between the expectation of $\tau$ and the ratio of its expectations is at most of order $n^{-1}$, namely, 
   \begin{equation*}
        \begin{split}
         \left\lvert   \mathbb{E}\left( \frac{ \mbox{$\sum_{i=1}^n$} \tau_i Z_i}{ \mbox{$\sum_{i=1}^n$} Z_i }  \right) - \frac{\mathbb{E}(\mbox{$\sum_{i=1}^n$} \tau_i Z_i)}{ \mathbb{E} (\mbox{$\sum_{i=1}^n$} Z_i) } \right\rvert \leq  C {n}^{-1}. 
        \end{split}
    \end{equation*}
\end{lemma}

\begin{lemma}
\label{clt_numerator}
Under Assumptions \ref{ass:cond_ind} and \ref{assump:outc_mod} we have
    \begin{equation*}
        \begin{split}
            \sqrt{n} \sigma^{-1}_{11} \left[ \frac{1}{n} \mbox{$\sum_{i=1}^n$} \left(Z_i(Y_i-f_i)- \frac{e_i}{1-e_i} (1-Z_i)(Y_i-f_i)  \right)- \frac{1}{n} \mathbb{E}\left(\mbox{$\sum_{i=1}^n$}  Z_i\tau_i\right)  \right] \rightarrow N(0,1)
        \end{split}
    \end{equation*}
where $\sigma^{2}_{11}=\frac{1}{n}\mathbb{E}_{Z,Y} [\Gamma_n]+\frac{1}{n} \sum_{i=1}^n \tau^2_i e_i (1-e_i)$.
\end{lemma}

\begin{lemma}
\label{clt_denominator}
Under Assumptions \ref{ass:cond_ind} and \ref{assump:outc_mod} we have
  \begin{equation*}
        \begin{split}
            \sqrt{n} \sigma^{-1}_{22} \left( \frac{1}{n} \mbox{$\sum_{i=1}^n$} Z_i- \frac{1}{n} \mbox{$\sum_{i=1}^n$} e_i  \right) \rightarrow N(0,1)
        \end{split}
    \end{equation*}
where $\sigma^{2}_{22}=\frac{1}{n} \sum_{i=1}^n e_i (1-e_i)$. 
\end{lemma}

\begin{lemma}
\label{multi_clt}
Let $ \Sigma^{M}_{n}= \begin{pmatrix}
\sigma^2_{11} & \sigma_{12}\\
\sigma_{21} & \sigma^2_{22}
\end{pmatrix}$ where $\sigma_{12}=\sigma_{21}=\frac{1}{n }\sum_{i=1}^n e_i(1-e_i)\tau_i$. The entries $\sigma^2_{11}$ and $\sigma^2_{22}$ are defined as in Lemmas \ref{clt_numerator} and \ref{clt_denominator} respectively. Under Assumptions \ref{ass:cond_ind} and \ref{assump:outc_mod}, we have  
 \[ \frac{\sqrt{n}}{\Sigma^{M \ \frac{1}{2}}_{n}} \left[ 
\begin{pmatrix}
\frac{1}{n} \sum_{i=1}^n \left[ Z_i (Y_i-f_i)- \frac{e_i}{1-e_i} (1-Z_i)(Y_i-f_i)  \right] \\
\frac{1}{n} \sum_{i=1}^n Z_i  
\end{pmatrix}
- 
\begin{pmatrix}
\frac{1}{n} \sum_{i=1}^n e_i \tau_i\\
\frac{1}{n} \sum_{i=1}^n e_i 
\end{pmatrix} 
\right] \rightarrow \mathcal{N}(0, I)
\] where $I$ is a two by two identity matrix. 
\end{lemma}
%We are now prepared to prove Proposition \ref{clt_delta}. 
\begin{proof}[Proof of Proposition \ref{clt_delta}]
Let $U=\frac{1}{n}\sum_{i=1}^n \left[ Z_i (Y_i-f_i)-  \frac{e_i}{1-e_i} (1-Z_i)(Y_i-f_i)\right]$, $W=\frac{1}{n}\sum_{i=1}^n Z_i$ and $f((U,W)^T)=\frac{U}{W}$. Then $\mathbb{E}(U)= \frac{1}{n} \sum_{i=1}^n \tau_i e_i$ and $\mathbb{E}(W)=\frac{1}{n}\sum_{i=1}^n e_i$. By applying Delta method on $f((U,W)^T)$, we have 
\begin{small}
\begin{equation*}
    \begin{split}
       & \sqrt{n} \ (\Sigma^D_n)^{-1/2} \left( \hat{\tau}^{AIPW}- \mathbb{E}(\tau) \right)=_{(1)} \sqrt{n} \ \left(\Sigma^{D}_n\right)^{-\frac{1}{2}} \left( \hat{\tau}^{AIPW*}- \mathbb{E}(\tau) \right) + O\left((\Sigma^D_n)^{-1/2}   \delta^{\frac{1}{2}}_n \Sigma^{-1/2}_n \right) \\
       &=_{(2)} \frac{\sqrt{n}}{(\Sigma^D_{n})^{1/2}} \Bigg( \sum_{i=1}^n \frac{Z_i(Y_i-f_i)- \frac{e_i(1-Z_i)}{1-e_i} (Y_i-f_i)}{ \mbox{$\sum_{i=1}^n Z_i$}  } - \frac{\mbox{$\sum_{i=1}^n e_i \tau_i$} }{ \mbox{$\sum_{i=1}^n e_i$}}  \Bigg) + O(\delta^{1/2}_n)  \rightarrow_{(3)} N\left(0,  1  \right)
    \end{split}
\end{equation*}
\end{small}
where $\Sigma^D_n =\left( \nabla f  \left( \left(\mathbb{E}(U), \mathbb{E}(W) \right)^T \right) \right)^T \Sigma^M_n \nabla f \left( ( \mathbb{E}(U),\mathbb{E}(W) )^T \right)$. $(1)$ is based on equation \eqref{eq:dr_main} in the proof of Theorem \ref{thm:dra_tau_inf} in %\cite{belloni2022neighborhood}
the main paper, with probability $1-5\Delta_n-3\delta-\exp(-nc^2/8)-n^{-2}-Cn^{-1}$. $(2)$ is based on Lemma \ref{exp_ratio_ratio_exp} and $(3)$ is based on Lemmas \ref{clt_numerator}, \ref{clt_denominator} and \ref{multi_clt}. Then the variance of $\sqrt{n} \hat{\tau}^{AIPW*}$ is equal to 
\begin{small}
\begin{equation*}
\begin{split}
  & \left( \nabla f  \left( \left(\mathbb{E}(U), \mathbb{E}(W) \right)^T \right) \right)^T \Sigma^M_n \nabla f \left( ( \mathbb{E}(U),\mathbb{E}(W) )^T \right)  \\
  & = \begin{pmatrix}
      \frac{1}{n^{-1}\mbox{$\sum_{i=1}^n$} e_i } & - \frac{\mbox{$\sum_{i=1}^n$} \tau_i e_i}{ n^{-1} \left( \mbox{$\sum_{i=1}^n$} e_i \right)^2  } 
  \end{pmatrix}  \begin{pmatrix}
      \sigma^2_{11} & \sigma_{12} \\
      \sigma_{21} & \sigma^2_{22}
  \end{pmatrix}  \begin{pmatrix}
      \frac{1}{n^{-1}\mbox{$\sum_{i=1}^n$} e_i } & - \frac{\mbox{$\sum_{i=1}^n$} \tau_i e_i}{ n^{-1} \left( \mbox{$\sum_{i=1}^n$} e_i \right)^2  } 
  \end{pmatrix}^T \\
  & = \frac{\sigma^2_{11}}{ ( n^{-1} \mbox{$\sum_{i=1}^n$} e_i)^2 } - \frac{ 2 \sigma_{21} \mbox{$\sum_{i=1}^n$} \tau_i e_i }{ n^{-2} (\mbox{$\sum_{i=1}^n$} e_i)^3 } + \frac{ \left( \mbox{$\sum_{i=1}^n$} \tau_i e_i \right)^2 \sigma^2_{22}}{  n^{-2}  (\mbox{$\sum_{i=1}^n$} e_i)^4 } \\
  & = \frac{\mathbb{E}_{Z,Y} [\Gamma_n] }{ n^{-1}(\mbox{$\sum_{i=1}^n$} e_i )^2 }+ \frac{\mbox{$\sum_{i=1}^n$} \tau^2_i e_i (1-e_i) }{ n^{-1} (\mbox{$\sum_{i=1}^n$} e_i)^2 } - \frac{ \left( \mbox{$\sum_{i=1}^n$} \tau_i e_i \right) \cdot \left( \mbox{$\sum_{i=1}^n$} e_i(1-e_i) \tau_i \right) }{n^{-1} \left( \mbox{$\sum_{i=1}^n$} e_i \right)^3 } \\ 
  & + \frac{ \left(\mbox{$\sum_{i=1}^n$} \tau_i e_i \right)^2 \cdot \left( \mbox{$\sum_{i=1}^n$} e_i(1-e_i) \right) }{ n^{-1} \left(  \mbox{$\sum_{i=1}^n$} e_i \right)^4  }-\frac{ \left( \mbox{$\sum_{i=1}^n$} \tau_i e_i \right) \cdot \left( \mbox{$\sum_{i=1}^n$} e_i(1-e_i) \tau_i \right) }{n^{-1} \left( \mbox{$\sum_{i=1}^n$} e_i \right)^3 }\\
  & = \frac{\mathbb{E}_{Z,Y} [\Gamma_n] }{ n^{-1}(\mbox{$\sum_{i=1}^n$} e_i )^2 }+ \frac{ \mbox{$\sum_{i=1}^n$} \mbox{$\sum_{j=1}^n$} \tau_i e_i (1-e_i) e_j (\tau_i-\tau_j) }{ n^{-1} \left( \mbox{$\sum_{i=1}^n$} e_i \right)^3 } \\
  & + \frac{ \mbox{$\sum_{i=1}^n$} \mbox{$\sum_{j=1}^n$} \mbox{$\sum_{h=1}^n$} \tau_j e_i (1-e_i) e_j e_h (\tau_h-\tau_i) }{ n^{-1} \left( \mbox{$\sum_{i=1}^n$} e_i \right)^4 } \\
  & =  \frac{\mathbb{E}_{Z,Y} [\Gamma_n] }{ n^{-1}(\mbox{$\sum_{i=1}^n$} e_i )^2 } + \frac{ \mbox{$\sum_{i=1}^n$} \mbox{$\sum_{j=1}^n$}\mbox{$\sum_{h=1}^n$} e_i (1-e_i) e_j e_h  (\tau_i-\tau_j)(\tau_i-\tau_h) }{ n^{-1} \left( \mbox{$\sum_{i=1}^n$} e_i \right)^4  }
\end{split}    
\end{equation*}
\end{small}
\end{proof}

The proof of Proposition \ref{cons_var_epsilon_Z} relies on Lemmas \ref{cons_exp}, \ref{int_upper_bound} and \ref{dif_e_i_hat_e_i}. 
\begin{lemma}
\label{cons_exp}
Let \(\widetilde{\Sigma}_n\) be as defined in Corollary \ref{cor:var} and 
$$S:=\frac{1}{(\mbox{$\sum_{i=1}^n$} e_i)^{4} } \left[ \sum_{i,j,h: j\neq h} e_i e_j e_h (\tau_i-\tau_j)(\tau_i-\tau_h) + \sum_{i,j,h: j= h} e_i e_j  (\tau_i-\tau_j)^2 \right] +\frac{\mathbb{E}_{Z,Y}(\Gamma_n)}{(\mbox{$\sum_{i=1}^n$} e_i)^{2}}$$
%$\widetilde {\Sigma}^*_n := \frac{1}{\mbox{$\sum_{i=1}^n$} Z_i }\sum_{i=1}^n  \{ Z_i(Y_i- \tau-\hat f_i) - (Y_i-\hat f_i)(1-Z_i)\hat e_i/(1-\hat e_i)\}^2$. 
where \(\Gamma_n\) is as defined in Proposition \ref{clt_delta}. Under Assumptions \ref{cond_unconf}, 
    \ref{ass:cond_ind}, \ref{assump:outc_mod} and \ref{assum:DR}, there exists a constant \( C > 0 \) such that, with probability at least \(1 - 3\delta - \Delta_n - \exp(-nc^2/8) - 6n^{-1}\),
\begin{equation*}
   \left\lvert {(\mbox{$\sum_{i=1}^n$} Z_i)^{-1}} \widetilde{\Sigma}_n - S \right\rvert \leq  C n^{-\frac{5}{4}} \log^{\frac{1}{2}}(2n^2/\delta)\delta_n\log^{-\frac{1}{2}}(n).
   %C \sqrt{\log \left( 2n^2/\delta  \right) } \cdot n^{-\frac{3}{2}}  
\end{equation*}
%where $\widetilde{\Sigma}_n$ is defined in Corollary \ref{cor:var} and $C>0$ which is a constant. 
\end{lemma}

\begin{lemma}
    \label{int_upper_bound}
    Under Assumptions \ref{ass:cond_ind} and \ref{assump:outc_mod} we have %, for each $i \in [n]$,  
    \begin{equation*}
        \begin{split}
          &  \mbox{$\sum_{i=1}^n$} e^2_i \mbox{$\sum_{j=1}^n$} \mbox{$\sum_{h=1}^n$}  e_j e_h (\tau_i-\tau_j) (\tau_i-\tau_h)\geq 0. \\
        \end{split}
    \end{equation*}
\end{lemma}

\begin{lemma}
    \label{dif_e_i_hat_e_i}
    Let $\delta_n$ be as defined in Assumption \ref{assum:DR}. Under Assumptions \ref{ass:cond_ind} and  \ref{assum:DR} (ii), there exists a constant $C>0$ such that, with probability at least $1-\Delta_n$   
    \begin{equation*}
        \begin{split}
        \frac{1}{n} \sum_{i=1}^n (e_i-\hat{e}_i)^2 \leq C\delta^2_n.
        \end{split}
    \end{equation*}
\end{lemma}

\begin{proof}[Proof of Proposition \ref{cons_var_epsilon_Z}]
Based on Lemma \ref{cons_exp} and the formula for $\Sigma^D_n$ in Proposition \ref{clt_delta}, we have, with probability $1-3\delta-\Delta_n-\exp(-nc^2/8)-6n^{-1}$, 
\begin{equation*}
    \begin{split}
       &  \frac{1}{n} \widetilde{\Sigma}^C_n - \frac{1}{n}\Sigma^D_n =  S - \frac{1}{n} \Sigma^D_n+ \frac{1}{\mbox{$\sum_{i=1}^n$} Z_i }\widetilde{\Sigma}_{n}- S \\
       %& \geq \mathbb{E} \left(  \frac{1}{\mbox{$\sum_{i=1}^n$} Z_i } \widetilde{\Sigma}^*_{n} \right) - \frac{1}{n} \Sigma^D_n+ \frac{1}{\mbox{$\sum_{i=1}^n$} Z_i }\widetilde{\Sigma}_{n}- \mathbb{E} \left( \frac{1}{\mbox{$\sum_{i=1}^n$} Z_i } \widetilde{\Sigma}^*_{n} \right) \\
       %& \geq  \frac{ \sum_{i,j,h: j\neq h} e_i e_j e_h (\tau_i-\tau_j)(\tau_i-\tau_h) + \sum_{i,j,h: j= h} e_i e_j  (\tau_i-\tau_j)^2 }{ (\mbox{$\sum_{i=1}^n$} e_i)^4 }  \\
       & \geq { (\mbox{$\sum_{i=1}^n$} e_i)^{-4} } \ \ [{ \sum_{i,j,h: j\neq h} e_i e_j e_h (\tau_i-\tau_j)(\tau_i-\tau_h) + \sum_{i,j,h: j= h} e_i e_j  (\tau_i-\tau_j)^2 } ] \\
       & - {(\mbox{$\sum_{i=1}^n$} e_i)^{-4} }{ \mbox{$\sum_{i=1}^n$} \mbox{$\sum_{j=1}^n$} \mbox{$\sum_{h=1}^n$} e_i (1-e_i)e_j e_h (\tau_i-\tau_j)(\tau_i-\tau_h)  } - C { \log^{\frac{1}{2}} \left( {2n^2}/{\delta} \right)} \cdot n^{-\frac{5}{4}} \\
       & \geq { (\mbox{$\sum_{i=1}^n$} e_i)^{-4} }{ [ \sum_{i,j,h: j\neq h} e_i e_j e_h (\tau_i-\tau_j)(\tau_i-\tau_h) + \sum_{i,j,h: j= h} e_i e_j e_{h}  (\tau_i-\tau_j)^2 ]} \\
       & - {(\mbox{$\sum_{i=1}^n$} e_i)^{-4} }{ [\mbox{$\sum_{i=1}^n$} \mbox{$\sum_{j=1}^n$} \mbox{$\sum_{h=1}^n$} e_i (1-e_i)e_j e_h (\tau_i-\tau_j)(\tau_i-\tau_h) ] } - C { \log^{\frac{1}{2}} \left( {2n^2}/{\delta} \right)} \cdot n^{-\frac{5}{4}} \\
       & = {(\mbox{$\sum_{i=1}^n$} e_i)^{-4}}{\mbox{$\sum_{i=1}^n$} \left[e_i-e_i (1-e_i)\right] \mbox{$\sum_{j=1}^n$} \mbox{$\sum_{h=1}^n$} e_j e_h (\tau_i-\tau_j) (\tau_i-\tau_h ) }\\
       &- C { \log^{\frac{1}{2}} \left( {2n^2}/{\delta} \right)} \cdot n^{-\frac{5}{4}} \geq  - C \cdot \sqrt{ \log \left( 2n^2/\delta \right)} \cdot n^{-\frac{5}{4}}  %n^{-\frac{3}{2}}
    \end{split}
\end{equation*}
where the last inequality holds by Lemma \ref{int_upper_bound} that shows the first term is non-negative. 
%\end{proof}

%\begin{proof}[Proof of Proposition \ref{var_est_E_tau}]

To prove the second part of the proposition, we show that \( |\hat{\Sigma}^D_n - {\Sigma}^D_n| \leq C \delta_n \) with high probability, where \( C \) is a positive constant. Consequently, Proposition \ref{clt_delta} remains valid after replacing \( {\Sigma}^D_n \) with \( \hat{\Sigma}^D_n \). Let $\hat{a}:=\mbox{$\sum_{i=1}^n$} \left( Z_i(Y_i-\hat{\tau}_i-\hat{f}_i) - \frac{\hat{e}_i}{1-\hat{e}_i} (1-Z_i) (Y_i-\hat{f}_i) \right)^2$, $\hat{b}:=n^{-1} \left( \mbox{$\sum_{i=1}^n$} \hat{e}_i   \right)^2$, 
%\begin{equation*}
%    \begin{split}
%      & \frac{\hat{a}}{\hat{b}}:= \frac{\mbox{$\sum_{i=1}^n$} \left( Z_i(Y_i-\hat{\tau}_i-\hat{f}_i) - \frac{\hat{e}_i}{1-\hat{e}_i} (1-Z_i) (Y_i-\hat{f}_i) \right)^2}{n^{-1} \left( \mbox{$\sum_{i=1}^n$} \hat{e}_i   \right)^2 } 
%    \end{split}
%\end{equation*}
$\hat{c}:=\mbox{$\sum_{i=1}^n$} \mbox{$\sum_{j=1}^n$} \mbox{$\mbox{$\sum_{h=1}^n$}$} \hat{e}_i (1-\hat{e}_i) \hat{e}_j \hat{e}_h (\hat{\tau}_i-\hat{\tau}_j)(\hat{\tau}_i- \hat{\tau}_h)$ and $\hat{d}:=n^{-1} (\mbox{$\sum_{i=1}^n$} \hat{e}_i)^4$  
%\begin{small}
%\begin{equation*}
%    \begin{split}
%& \frac{\hat{c}}{\hat{d}}:=  \frac{ \mbox{$\sum_{i=1}^n$} \mbox{$\sum_{j=1}^n$} \mbox{$\mbox{$\sum_{h=1}^n$}$} \hat{e}_i (1-\hat{e}_i) \hat{e}_j \hat{e}_h (\hat{\tau}_i-\hat{\tau}_j)(\hat{\tau}_i- \hat{\tau}_h)  }{n^{-1} (\mbox{$\sum_{i=1}^n$} \hat{e}_i)^4 }\\
%    \end{split}
%\end{equation*}
%\end{small}
where $a,b,c,d$ denote the oracle quantities for $\hat{a},\hat{b},\hat{c},\hat{d}$ respectively. Then $\hat{\Sigma}^D_n= \frac{\hat{a} }{ \hat{b} }+\frac{\hat{c} }{\hat{d} }$ and $\Sigma^D_n= \frac{a}{b}+\frac{c}{d}$. Thus, the difference between these two quantities is 
\begin{small}
\begin{equation}
\label{decomp_hat_Sigma_D_Sigma_D}
    \begin{split}
        \hat{\Sigma}^D_n- {\Sigma}^D_n=  \frac{(\hat{a}- a)}{\hat{b}}+ \left( \frac{1}{\hat{b}}-\frac{1}{b} \right) a + \frac{(\hat{c}-c)}{\hat{d}}+ \left( \frac{1}{\hat{d}}-\frac{1}{d} \right) c. 
    \end{split}
\end{equation}
\end{small}
\noindent Let $C_i:=  Z_i(Y_i-\hat{\tau}_i-\hat{f}_i) - \frac{\hat{e}_i}{1-\hat{e}_i} (1-Z_i) (Y_i-\hat{f}_i) +Z_i(Y_i-{\tau}_i-{f}_i) - \frac{{e}_i}{1-{e}_i} (1-Z_i) (Y_i-{f}_i)$ and $d(\hat{e}_i,e_i):=\frac{\hat{e}_i}{1-\hat{e}_i}-\frac{{e}_i}{1-{e}_i}$ for $i\in [n]$. We begin by bounding \( |\hat{a} - a| \).  
\begin{small}
\begin{equation*}
    \begin{split}
       % |\hat{a}-a| & =  \left\lvert \mbox{$\sum_{i=1}^n$}  Z_i(\tau_i-\hat{\tau}_i) C_i + \mbox{$\sum_{i=1}^n$}  Z_i(f_i-\hat{f}_i) C_i + \mbox{$\sum_{i=1}^n$} \left( \frac{e_i}{1-e_i}- \frac{\hat{e}_i}{1-\hat{e}_i} \right) Y_i (1-Z_i) C_i \right.\\
       & |\hat{a}-a|  =  \left\lvert \mbox{$\sum_{i=1}^n$}  Z_i(\tau_i-\hat{\tau}_i) C_i + \mbox{$\sum_{i=1}^n$}  Z_i(f_i-\hat{f}_i) C_i + \mbox{$\sum_{i=1}^n$} \left( \frac{e_i}{1-e_i}- \frac{\hat{e}_i}{1-\hat{e}_i} \right) (f_i+\epsilon_i) (1-Z_i) C_i \right.\\
        & \left. + \mbox{$\sum_{i=1}^n$} \left[\left( \frac{\hat{e}_i}{1-\hat{e}_i} \hat{f}_i- \frac{{e}_i}{1-{e}_i} \hat{f}_i\right) + \left( \frac{{e}_i}{1-{e}_i} \hat{f}_i - \frac{{e}_i}{1-{e}_i} f_i \right) \right] (1-Z_i) C_i \right\rvert \\
        & \leq_{(1)} C \left\lbrace [\sum_{i=1}^n  (\tau_i - \hat{\tau}_i)^2]^{\frac{1}{2}}  +  [ \sum_{i=1}^n  (f_i - \hat{f}_i)^2 ]^{\frac{1}{2}} +  [\sum_{i=1}^n   d^2(\hat{e}_i,e_i) ]^{\frac{1}{2}} +  \left[\mbox{$\sum_{i=1}^n $}  d^2(\hat{e}_i,e_i) \epsilon^2_i \right]^{\frac{1}{2}} \right\rbrace  (\sum_{i=1}^n  C^2_i )^{\frac{1}{2}}    \\
        %& \leq_{(1)} \left[ \sqrt{\mbox{$\sum_{i=1}^n $} (\tau_i - \hat{\tau}_i)^2 } + C \sqrt{\mbox{$\sum_{i=1}^n $} (f_i - \hat{f}_i)^2 }+ C \sqrt{\sum_{i=1}^n  \left(\frac{e_i}{1-e_i} - \frac{\hat{e}_i}{1-\hat{e}_i}\right)^2 }  \right] \sqrt{ \mbox{$\sum_{i=1}^n $} C^2_i  }   \\
        &:= C \left[(I) +(II)+(III)+(IV)\right] (\sum_{i=1}^n  C^2_i )^{\frac{1}{2}} \\
        % & \leq_{(2)} C \left[ n^{1/2} \delta_n + \delta_n + \frac{\log^{1/2}(2n^3)n^{1/4} \delta_n \log^{-1/2}(n) }{(1+D_n)^{1/2}} + \frac{n^{1/2}\delta_n}{\log^{1/2} n}   \right] n^{1/2} \leq C n \delta_n  \\
       % & =_{1} o_p(n)
    \end{split}
\end{equation*}
\end{small}
\((1)\) is obtained by applying the Cauchy-Schwarz inequality to each term. Next, we derive an upper bound for \( \sum_{i=1}^n  C^2_i \).
\begin{small}
\begin{equation*}
    \begin{split}
      \sum_{i=1}^n C^2_i &= \sum_{i=1}^n \left[Z_i(\tau_i-\hat{\tau}_i+f_i-\hat{f}_i+\epsilon_i)-\frac{\hat{e}_i}{1-\hat{e}_i} (1-Z_i) (f_i-\hat{f}_i+\epsilon_i) + Z_i\epsilon_i-\frac{e_i}{1-e_i}(1-Z_i)\epsilon_i  \right]^2  \\
      & = \sum_{i=1}^n \left\lbrace Z_i(\tau_i-\hat{\tau}_i)+\left[Z_i-\frac{\hat{e}_i(1-Z_i)}{1-\hat{e}_i}\right](f_i-\hat{f}_i)+\left[2Z_i-(1-Z_i)(\frac{\hat{e}_i}{1-\hat{e}_i}+\frac{e_i}{1-e_i})\right] \epsilon_i  \right\rbrace^2 \\
      & \leq_{(1)} C \mbox{$\sum_{i=1}^n$}  \left( \tau_i-\hat{\tau}_i \right)^2 + C \mbox{$\sum_{i=1}^n$}  \left( f_i-\hat{f}_i \right)^2+ C \mbox{$\sum_{i=1}^n$} \epsilon_i^2 \\
       & \leq_{(2.1)} C n \delta^2_n+_{(2.2)} C\left( \delta^2_n +  \frac{\log(2n^3)n^{1/2} \delta^2_n \log^{-1}n}{1+D_n}\right)+_{(2.3)}2n \bar{\sigma}^2  \leq C n 
    \end{split}
\end{equation*}
\end{small}
\noindent where the positive constants \(C\) vary line by line. 
%\((1)\) is derived by decomposing \(Y_i(z_i,Z_{-i})\) based on Assumption \ref{assump:outc_mod}. 
\((1)\) is based on Assumptions \ref{ass:cond_ind} and \ref{assum:DR} (i). 
%For \((2)\), the first term 
$(2.1)$ follows from the upper bound for \(\sum_{i=1}^n (\tau_i - \hat{\tau}_i)^2\) in Proposition \ref{cons_var_epsilon_Z} with probability $1-o(1)$. %The second term within the parentheses 
$(2.2)$ is obtained by using Lemma %\href{lemma:sparsity_gamma_star}{10}
\ref{lemma:sparsity_gamma_star} 
and Assumption \ref{assum:DR} (ii), with probability $1-\delta-\Delta_n$. %The third term 
$(2.3)$ is derived by applying Bernstein’s inequality to the sum \(\sum_{i=1}^n \epsilon_i^2\) with probability $1-1/n$. Next, we establish bounds for \((I) - (IV)\). Based on the assumption in Proposition \ref{cons_var_epsilon_Z} and Lemma \ref{dif_e_i_hat_e_i} respectively, we have $(I)\leq n^{1/2} \delta_n$ and  $(III)\leq n^{1/2} \delta_n$ with probability $1-o(1)-\Delta_n$. Furthermore, $(II)\leq \delta_n + \frac{\log^{1/2}(2n^3)n^{1/4} \delta_n \log^{-1/2}(n) }{(1+D_n)^{1/2}}$ which follows from Lemma %\href{lemma:sparsity_gamma_star}{10}
\ref{lemma:sparsity_gamma_star} 
combined with Assumption \ref{assum:DR} (ii), with probability $1-\delta-\Delta_n$. Next, we derive an upper bound for \((IV)\). Let $d(\hat{e}_i,e_i):=\frac{\hat{e}_i}{1-\hat{e}_i}-\frac{{e}_i}{1-{e}_i}$ for $i\in [n]$.
\begin{equation*}
    \begin{split}
    & (IV) = \left[\mbox{$\sum_{i=1}^n $}  d^2(\hat{e}_i,e_i) \epsilon^2_i \right]^{\frac{1}{2}}\leq_{(1)} \left[\mbox{$\sum_{i=1}^n $}  d^4(\hat{e}_i,e_i)\right]^{\frac{1}{4}} \left[ \mbox{$\sum_{i=1}^n $} \epsilon^4_i \right]^{\frac{1}{4}}\\
    & \leq_{(2)} C \left[\mbox{$\sum_{i=1}^n $}  d^2(\hat{e}_i,e_i)\right]^{\frac{1}{4}} \left[ \mbox{$\sum_{i=1}^n $} \epsilon^4_i \right]^{\frac{1}{4}} \leq_{(3)} C \log^{\frac{1}{4}} (n) n^{\frac{1}{4}}
    \end{split}
\end{equation*}
(1) is by the Cauchy-Schwarz inequality. (2) follows from the bound $d^4(\hat{e}_i,e_i)\leq C d^2(\hat{e}_i,e_i) [e^2_i/(1-e_i)^2+\hat{e}^2_i/(1-\hat{e}_i)^2]\leq C d^2(\hat{e}_i,e_i) [e^2_i/(1-e_i)^2+\hat{e}^2_i/(1-\hat{e}_i)^2]$ which holds under Assumptions \ref{ass:cond_ind} and \ref{assum:DR} (i). $(3)$ is based on Assumption \ref{assum:DR} (i) with probability $1-\Delta_n$, along with applying Bernsteins inequality to $\sum_{i=1}^n \epsilon^4_i$, with probability $1-1/n$, where $\epsilon_i$ is subgaussian for $i\in[n]$.
%The bounds for the first, second, and third terms within the brackets in \((2)\) follow from the following: the assumption on \(\sum_{i=1}^n (\tau_i - \hat{\tau}_i)^2\) in Proposition \ref{cons_var_epsilon_Z}, Lemma \href{lemma:sparsity_gamma_star}{10} %\ref{lemma:sparsity_gamma_star} 
%combined with Assumption \ref{assum:DR} (ii), and Assumption \ref{assum:DR} (i), respectively. Then with probability $1-o(1)-n^{-1}-\Delta_n-2 \exp(-3n/16)$, 
Finally, combining the results for \(\sum_{i=1}^n C^2_i\) and \((I) - (IV)\), we obtain, with probability at least \(1 - o(1) - \delta - 2\Delta_n - n^{-1}\),
\begin{small}
\begin{equation*}
    \begin{split}
       & |\hat{a}-a|  \leq C \left[(I) +(II)+(III)+(IV)\right] (\sum_{i=1}^n  C^2_i )^{\frac{1}{2}} \\
       & \leq C \left[ n^{1/2} \delta_n + \delta_n + \frac{\log^{1/2}(2n^3)n^{1/4} \delta_n \log^{-1/2}(n) }{(1+D_n)^{1/2}} + n^{1/2}\delta_n + \log^{\frac{1}{4}} (n) n^{\frac{1}{4}}  \right] n^{1/2} \leq C n \delta_n.  \\
       % & =_{1} o_p(n)
    \end{split}
\end{equation*}
\end{small}
%and the assumptions in Proposition \ref{var_est_E_tau}. Furthermore, 
Furthermore, we have the difference $ 
      \left\lvert  \frac{a}{\hat{b}}-\frac{a}{b} \right\rvert \leq |a| \cdot \left\lvert \frac{ (\mbox{$\sum_{i=1}^n$} e_i)^2 - (\mbox{$\sum_{i=1}^n$} \hat{e}_i)^2 }{ n^{-1} (\mbox{$\sum_{i=1}^n$} \hat{e}_i)^2 \cdot (\mbox{$\sum_{i=1}^n$} {e}_i)^2 } \right\rvert \leq C \delta_n 
$
Since \(|a| = \sum_{i=1}^n \left[ Z_i - \frac{e_i}{1 - e_i}(1 - Z_i) \right]^2 \epsilon_i^2 \leq C \sum_{i=1}^n \epsilon_i^2 \leq Cn\), with probability at least \(1 - 2\exp(-3n/16)\). This result follows from Assumption \ref{ass:cond_ind} and applying Bernstein’s inequality to \(\sum_{i=1}^n \epsilon_i^2\). Additionally, with probability $1-\Delta_n$, we have
\begin{equation*}
    \begin{split}
    & \left \lvert  \left(\mbox{$\sum_{i=1}^n$} e_i \right)^2 - \left(\mbox{$\sum_{i=1}^n$} \hat{e}_i \right)^2  \right\rvert =\left\lvert \left( \mbox{$\sum_{i=1}^n$} (e_i - \hat{e}_i )  \right) \cdot \left( \mbox{$\sum_{i=1}^n$} e_i + \mbox{$\sum_{i=1}^n$} \hat{e}_i  \right) \right\rvert \\
    &\leq_{(1)} C \sqrt{n} \left[\mbox{$\sum_{i=1}^n$} (e_i - \hat{e}_i )^2 \right]^{\frac{1}{2}} n \leq_{(2)} C n^2 \delta_n 
    \end{split}
\end{equation*}
 where \((1)\) follows from the Cauchy-Schwarz inequality and Assumption \ref{assum:DR} (i), and \((2)\) is derived using %the assumption on \(\sum_{i=1}^n (e_i - \hat{e}_i)^2\) in Proposition \ref{cons_var_epsilon_Z}
 Lemma \ref{dif_e_i_hat_e_i}, with probability $1-\Delta_n$. Similarly,
\begin{equation*}
    \begin{split}
       \left\lvert \frac{c}{\hat{d}}-\frac{c}{d} \right\rvert \leq  |c| \cdot \left \lvert \frac{ (\mbox{$\sum_{i=1}^n$} e_i)^4 - (\mbox{$\sum_{i=1}^n$} \hat{e}_i)^4 }{ n^{-1} (\mbox{$\sum_{i=1}^n$} \hat{e}_i)^4 \cdot (\mbox{$\sum_{i=1}^n$} {e}_i)^4 } \right\rvert \leq C \delta_n 
    \end{split}
\end{equation*}
since $|c|\leq C n^3$ due to the boundedness of \(|\tau_i|\) and \(|e_i|\) for \(i \in [n]\) under Assumptions \ref{ass:cond_ind} and \ref{assump:outc_mod}, and 
\begin{small}
\begin{equation*}
    \begin{split}
    & \left\lvert \left(\mbox{$\sum_{i=1}^n$} e_i \right)^4 - \left(\mbox{$\sum_{i=1}^n$} \hat{e}_i \right)^4 \right\rvert =\left\lvert \left[ \mbox{$\sum_{i=1}^n$}(e_i - \hat{e}_i )  \right] \cdot \left( \mbox{$\sum_{i=1}^n$} e_i + \mbox{$\sum_{i=1}^n$} \hat{e}_i  \right)\cdot \left[ (\mbox{$\sum_{i=1}^n$} e_i)^2 + (\mbox{$\sum_{i=1}^n$}  \hat{e}_i)^2   \right] \right\rvert \\
    & \leq C \sqrt{n} \left[\mbox{$\sum_{i=1}^n$} (e_i - \hat{e}_i )^2 \right]^{\frac{1}{2}} n^3 \leq C n^4 \delta_n
    \end{split}
\end{equation*}
\end{small}
with probability $1-\Delta_n$, based on %the assumption on propensity scores in Proposition \ref{cons_var_epsilon_Z}
Lemma \ref{dif_e_i_hat_e_i}. For $|\hat{c}-c|$, define \(\mathbf{\hat{\tau}}_{(i)} = (\hat{\tau}_i - \hat{\tau}_1, \dots, \hat{\tau}_i - \hat{\tau}_n)^T\) for \(i \in [n]\) and \(\hat{\mathbf{e}} = (\hat{e}_1, \dots, \hat{e}_n)^T\). Let \(\tau_{(i)}\) and \(\mathbf{e}\) denote the corresponding oracle quantities. Then, we have 
\begin{equation*}
    \begin{split}
      & |\hat{c}-c|= \left\lvert \mbox{$\sum_{i=1}^n$} \hat{e}_i (1-\hat{e}_i) \cdot \left( \hat{\tau}^T_{(i)} \mathbf{\hat{e}} \right)^2 - \mbox{$\sum_{i=1}^n$} e_i (1- e_i) (\tau^T _{(i)} \mathbf{e} )^2 \right\rvert \\
      & = \left\lvert  \sum_{i=1}^n \left[ \hat{e}_i (1-\hat{e}_i) - e_i(1-e_i)   \right] \cdot \left( \hat{\tau}^T_{(i)} \mathbf{\hat{e}} \right)^2 + \sum_{i=1}^n e_i (1- e_i) \cdot \left[ (\hat{\tau}^T_{(i)} \mathbf{\hat{e}})^2 - (\tau^T_{(i)} \mathbf{e})^2   \right] \right\rvert \leq C n^3 \delta_n.
    \end{split}
\end{equation*}
The first term in this decomposition satisfies
\begin{small}
\begin{equation*}
    \begin{split}
     & \left\lvert \mbox{$\sum_{i=1}^n$} \left[ \hat{e}_i (1-\hat{e}_i) - e_i(1-e_i)   \right] \cdot \left( \hat{\tau}^T_{(i)} \mathbf{\hat{e}} \right)^2 \right\rvert \leq_{(1)} \sqrt{ \mbox{$\sum_{i=1}^n$} (\hat{e}_i - e_i + e^2_i-\hat{e}^2_i )^2 } \sqrt{\mbox{$\sum_{i=1}^n$} \left( \hat{\tau}^T_{(i)} \mathbf{\hat{e}} \right)^4 } \\
     & = \sqrt{ \mbox{$\sum_{i=1}^n$} \left[ (\hat{e}_i-e_i) \cdot (1+e_i+\hat{e}_i) \right]^2 } \sqrt{\mbox{$\sum_{i=1}^n$} \left( \hat{\tau}^T_{(i)} \mathbf{\hat{e}} \right)^4 } \leq_{(2)} C  \sqrt{ \mbox{$\sum_{i=1}^n$}  (\hat{e}_i-e_i)^2 } \sqrt{\mbox{$\sum_{i=1}^n$} \left( \hat{\tau}^T_{(i)} \mathbf{\hat{e}} \right)^4 } \\
     & \leq_{(3)} C \delta_n \sqrt{n} \cdot n^{\frac{5}{2}} = C n^3 \delta_n.
    \end{split}
\end{equation*}
\end{small}
\noindent with probability $1-o(1)-\Delta_n$, where \((1)\) follows from the Cauchy-Schwarz inequality. \((2)\) follows from the fact that \((1 + e_i + \hat{e}_i) \leq C\) for all \( i \in [n] \) under Assumptions \ref{ass:cond_ind} and \ref{assum:DR} (i). $(3)$ follows from Lemma \ref{dif_e_i_hat_e_i}, with probability at least \( 1 - \Delta_n \). Additionally, since \( |\hat{\tau}_{(i)}^T \hat{\mathbf{e}}| \leq Cn \) for all \( i \in [n] \), as implied by Assumption \ref{assum:DR} (i) for \(\hat{e}_i\) and the upper bound for \( \sum_{i=1}^n (\tau_i - \hat{\tau}_i)^2 \) in Assumption \ref{assump:outc_mod}, the result follows with probability at least \( 1 - o(1) \). For the second term, we have, for each $i \in [n]$,
%follows from: (a) \((1 + e_i + \hat{e}_i)\leq C\) for \(i \in [n]\) under Assumptions \ref{ass:cond_ind} and \ref{assum:DR} (i), Lemma \ref{dif_e_i_hat_e_i} with probability $1-\Delta_n$; %, (b) the assumption on \(\sum_{i=1}^n (\hat{e}_i - e_i)^2\) in Proposition \ref{cons_var_epsilon_Z}, (b) \(|\hat{\tau}_{(i)}^T \hat{\mathbf{e}}|\leq Cn\) for \(i \in [n]\), based on Assumption \ref{assum:DR} (i) for \(\hat{e}_i\) where $i\in [n]$ and the upper bound for $\sum_{i=1}^n(\tau_i-\hat{\tau}_i)^2$ in Assumption \ref{assump:outc_mod}, with probability $1-o(1)$. For the second term, we have, for each $i \in [n]$,
\begin{small}
\begin{equation*}
    \begin{split}
      &  |\hat{\tau}^T_{(i)} \hat{\mathbf{e}} - \tau^T_{(i)} \mathbf{e}| = \left\lvert\hat{\tau}^T_{(i)} (\hat{\mathbf{e}}- \mathbf{e} ) + \left( \hat{\tau}_{(i)}^T- \tau^T_{(i)} \right) \mathbf{e} \right\rvert \leq \sqrt{\mbox{$\sum_{j=1}^n$} (\hat{\tau}_i-\hat{\tau}_j)^2 } \cdot \sqrt{ \mbox{$\sum_{j=1}^n$} (\hat{e}_j-e_j)^2 }\\
      & + |\sum_{j \neq i} (\hat{\tau}_j -\tau_j ) e_j | \leq_{(1)} \sqrt{n}\cdot C \cdot \sqrt{n} \cdot \delta_n + \sqrt{ \mbox{$\sum_{j=1}^n$} (\tau_j- \hat{\tau}_j)^2 } \cdot \sqrt{ \mbox{$\sum_{j=1}^n$} e^2_j } \leq_{(2)} C n\delta_n
    \end{split}
\end{equation*}
\end{small}
with probability $1-o(1)-\Delta_n$. \((1)\) follows from Assumption \ref{assump:outc_mod} and Lemma \ref{dif_e_i_hat_e_i}, with probability at least \( 1 - \Delta_n \). \((2)\) uses the assumption on \(\sum_{j=1}^n (\hat{\tau}_j - \tau_j)^2\) in Proposition \ref{cons_var_epsilon_Z} and Assumption \ref{ass:cond_ind}, with probability $1-o(1)$. Then 
\begin{small}
\begin{equation*}
\begin{split}
& \left\lvert \mbox{$\sum_{i=1}^n$} e_i (1- e_i) \cdot \left[ \left( \hat{\tau}^T_{(i)} \mathbf{\hat{e}}\right)^2  -  \left( \tau^T_{(i)} \mathbf{e} \right)^2  \right] \right\rvert \leq \mbox{$\sum_{i=1}^n$} \mid e_i (1-e_i ) \left( \hat{\tau}^T_{(i)} \mathbf{\hat{e}}- {\tau}^T_{(i)} \mathbf{e} \right) \cdot \left( \hat{\tau}^T_{(i)} \mathbf{\hat{e}}+ {\tau}^T_{(i)} \mathbf{e} \right)   \mid \\
& \leq C \cdot n \mbox{$\sum_{i=1}^n$} \mid \hat{\tau}^T_{(i)} \mathbf{\hat{e}}- {\tau}^T_{(i)} \mathbf{e}  \mid \leq C \cdot n^3 \delta_n.
\end{split}
\end{equation*}
\end{small}
\noindent Therefore, $\left\lvert \frac{\hat{c}-c}{\hat{d} }\right\rvert \leq C\delta_n$. Consequently, using the bounds for each term in \eqref{decomp_hat_Sigma_D_Sigma_D}, we have \(\left| \hat{\Sigma}^D_n - \Sigma^D_n \right| \leq C \delta_n\) with probability $1-o(1)-\delta-3\Delta_n-n^{-1}-2\exp(-3n/16)$. %(Similar argument shows $\Sigma_n^D - \hat \Sigma_n^D \leq C\delta_n$.)
\end{proof}

Next, we present the proofs of the auxiliary lemmas. %We now prove the intermediate lemmas for proving Propositions \ref{clt_delta} and \ref{cons_var_epsilon_Z}. %and \ref{var_est_E_tau} respectively. 
\begin{proof}[Proof of Lemma \ref{exp_ratio_ratio_exp}]
Let $a=\sum_{i=1}^n Z_i \tau_i$, $b=\sum_{i=1}^n Z_i$ and $f(a,b)=\frac{a}{b}$. Then based on Taylor expansion up to second order at $(\mathbb{E}(a),\mathbb{E}(b))$, we have 
\begin{small}
\begin{equation*}
\begin{split}
    \mathbb{E}\left( f(a,b) \right) & = f( \mathbb{E}(a), \mathbb{E}(b) )+ \left. \frac{\partial f(a,b)}{\partial (a,b)^T}  \right\rvert_{a=\mathbb{E}(a), b=\mathbb{E}(b) } \mathbb{E} (a-\mathbb{E}(a), b- \mathbb{E}(b))^T \\
    & + \frac{1}{2} \mathrm{tr} \left[ H_f ( (\mathbb{E}(a), \mathbb{E} (b))^T) \ \mathrm{var} ((a,b)^T)  \right] + o \left( n^{-1} \right)
%\end{pmatrix} \cdot \mathrm{var} \begin{pmatrix}
%a \\
%b
%\end{pmatrix}  \right] + o \left( n^{-1} \right)\\
   % & + \frac{1}{2} \mathrm{tr} \left[ H_f   
%\begin{pmatrix}
%\mathbb{E}(a) \\
%\mathbb{E} (b)
%\end{pmatrix} \cdot \mathrm{var} \begin{pmatrix}
%a \\
%b
%\end{pmatrix}  \right] + o \left( n^{-1} \right)\\
    \end{split}
\end{equation*}
\end{small}
where $H_f((\mathbb{E}(a),\mathbb{E}(b))^T )$ denotes the Hessian matrix of $f$ evaluated at $(\mathbb{E}(a),\mathbb{E}(b))^T$. Then we have 
\begin{small}
\begin{equation*}
\begin{split}
    &\left\lvert  \mathbb{E}\left( f(a,b) \right)-\frac{\mbox{$\sum_{i=1}^n$} e_i \tau_i}{ \mbox{$\sum_{i=1}^n$} e_i} -0\right\rvert = \left\lvert \frac{1}{2} \mathrm{tr} \left[ H_f ( (\mathbb{E}(a), \mathbb{E} (b))^T) \cdot  \mathrm{var} ((a,b)^T)  \right] \right\rvert + o \left( n^{-1} \right)\\ 
    &  = \left\lvert  \frac{1}{2}  tr \left[ \ \begin{bmatrix}
0 & -( \mbox{$\sum_{i=1}^n$} e_i)^{-2}   \\
-( \mbox{$\sum_{i=1}^n$} e_i)^{-2}  & \frac{ 2 \mbox{$\sum_{i=1}^n$} e_i \tau_i }{\left(\mbox{$\sum_{i=1}^n$} e_i  \right)^3 }
\end{bmatrix}  \begin{bmatrix}
\sum_{i=1}^n e_i (1-e_i) \tau^2_i  & \sum_{i=1}^n e_i (1-e_i) \tau_i  \\
\mbox{$\sum_{i=1}^n$} e_i (1-e_i) \tau_i & \mbox{$\sum_{i=1}^n$} e_i (1-e_i)
\end{bmatrix} \ \right]  \right\rvert +o(\frac{1}{n})\leq  C  \frac{1}{n}
    \end{split}
\end{equation*}
\end{small}
where the second equality is based on Assumptions \ref{ass:cond_ind} and \ref{assump:outc_mod}.
\end{proof}

\begin{proof}[Proof of Lemma \ref{clt_numerator}]
We first derive the variance for $\frac{1}{n} \mbox{$\sum_{i=1}^n$} [ Z_i(Y_i-f_i)- \frac{e_i}{1-e_i} (1-Z_i)(Y_i-f_i)  ]$. Since 
\begin{equation*}
    \begin{split}
       &  var_{Y,Z} \left[ \frac{1}{n} \mbox{$\sum_{i=1}^n$} \left(Z_i(Y_i-f_i)- \frac{e_i}{1-e_i} (1-Z_i)(Y_i-f_i)  \right) \right] \\
       & = \mathbb{E}_{Z} \ var_{Y|Z} \left[ \left. \frac{1}{n} \mbox{$\sum_{i=1}^n$} \left( Z_i (Y_i-\tau_i-f_i)- \frac{e_i}{1-e_i}(1-Z_i) \cdot (Y_i-f_i) \right)  +\frac{1}{n} \mbox{$\sum_{i=1}^n$} Z_i \tau_i \right\rvert  Z \right] \\
       & + \ var_{Z} \ \mathbb{E}_{Y|Z}  \left[ \left. \frac{1}{n} \mbox{$\sum_{i=1}^n$} \left( Z_i (Y_i-\tau_i-f_i)- \frac{e_i}{1-e_i}(1-Z_i) \cdot (Y_i-f_i) \right)  +\frac{1}{n} \mbox{$\sum_{i=1}^n$} Z_i \tau_i  \right\rvert Z  \right] \\
       & = \frac{1}{n^2} \mathbb{E}_{Z,Y} (\Gamma_n) + var \left( 0+ \frac{1}{n} \mbox{$\sum_{i=1}^n$} Z_i \tau_i  \right) = \frac{1}{n^2} \mathbb{E}_{Z,Y} (\Gamma_{n})+ \frac{1}{n^2} \mbox{$\sum_{i=1}^n$} \tau^2_i e_i (1-e_i) = \frac{1}{n} \sigma^2_{11} 
    \end{split}
\end{equation*}
Furthermore, the sequence of $\left(Z_i(Y_i-f_i)- \frac{e_i}{1-e_i} (1-Z_i)(Y_i-f_i)  \right)$ satisfies Lindeberg's condition as verified using the Dominated Convergence Theorem. Consequently, the Central Limit Theorem holds.
\end{proof}

\begin{proof}[Proof of Lemma \ref{multi_clt}]
The covariance for the numerator and the denominator of $\hat{\tau}^{AIPW*}$ is 
\begin{small}
 \begin{equation*}
     \begin{split}
       &   cov \left(  \frac{1}{n} \mbox{$\sum_{i=1}^n$} \left[ Z_i \left(Y_i-f_i \right) - \frac{e_i}{1-e_i} (1- Z_i) (Y_i -f_i)-\tau_i e_i  \right], \frac{1}{n} \mbox{$\sum_{i=1}^n$} Z_i  \right) \\
       & = \frac{1}{n^2} \mathbb{E} \left\lbrace  \left[ \mbox{$\sum_{i=1}^n$}  \left(  Z_i (Y_i-f_i) - \frac{e_i}{1-e_i } (1- Z_i)(Y_i-f_i)- \tau_i e_i   \right) \right] \cdot \left( \mbox{$\sum_{i=1}^n$} Z_i \right)       \right\rbrace \\
       & - \frac{1}{n^2} \mathbb{E}_{Z} \mathbb{E}_{Y|Z} \left\lbrace  \left[ \mbox{$\sum_{i=1}^n$}  \left(  Z_i (Y_i-f_i) - \frac{e_i}{1-e_i } (1- Z_i)(Y_i-f_i)- \tau_i e_i   \right) \right] \right\rbrace \cdot  \mathbb{E} \left( \mbox{$\sum_{i=1}^n$} Z_i \right) \\
       & = \frac{1}{n^2} \mathbb{E}_{Z} \mathbb{E}_{Y|Z} \left[ \mbox{$\sum_{i=1}^n$} \mbox{$\sum_{j=1}^n$} Z_i (\tau_i+ \epsilon_i)Z_j - \mbox{$\sum_{i=1}^n$} \mbox{$\sum_{j=1}^n$} \frac{e_i}{1-e_i} (1-Z_i)\epsilon_i Z_j- \mbox{$\sum_{i=1}^n$} \mbox{$\sum_{j=1}^n$} \tau_i e_i Z_j   \right]  \\
       & = \frac{1}{n^2} \ \mathbb{E}_{Z} \left[ \mbox{$\sum_{i=1}^n$} \mbox{$\sum_{j=1}^n$} Z_i Z_j \tau_i - \mbox{$\sum_{i=1}^n$} \mbox{$\sum_{j=1}^n$} \tau_i e_i Z_j   \right] \\
       & = \frac{1}{n^2} \left[ \mbox{$\sum_{i=1}^n$} e_i \tau_i + \mbox{$\sum_{i=1}^n$} \sum_{j \neq i} e_i e_j \tau_i -  \left( \mbox{$\sum_{i=1}^n$} e^2_i \tau_i  + \mbox{$\sum_{i=1}^n$} \sum_{ j \neq i}  e_i e_j \tau_i \right) \right] \\
       & = \frac{1}{n^2} \mbox{$\sum_{i=1}^n$} e_i (1- e_i )\tau_i = \frac{1}{n} \sigma_{12} = \frac{1}{n} \sigma_{21}
     \end{split}
 \end{equation*}
 \end{small}
Based on Lemmas \ref{exp_ratio_ratio_exp} and 
\ref{clt_numerator}, we have, for any $a\in \mathbb{R}^2$, 
\begin{small}
\begin{equation*}
    \begin{split}
        & \sqrt{n} a^T (\Sigma^M_n)^{-1/2} \begin{pmatrix} 
        \frac{1}{n} \sum_{i=1}^n \left[ Z_i (Y_i-f_i)- \frac{e_i}{1-e_i}(1-Z_i)(Y_i-f_i)  \right] - \frac{1}{n}\sum_{i=1}^n e_i \tau_i \\
        \frac{1}{n} \sum_{i=1}^n Z_i- \frac{1}{n} \sum_{i=1}^n e_i 
        \end{pmatrix}  \rightarrow {N}(0, a^T a)
    \end{split}
\end{equation*}
\end{small}
Then based on Cram\'er-Wold Theorem, we have 
\begin{small}
\begin{equation*}
    \begin{split}
        & \sqrt{n} (\Sigma^M_n)^{-1/2}  \left[ \begin{pmatrix} 
        \frac{1}{n} \sum_{i=1}^n \left[ Z_i (Y_i-f_i)- \frac{e_i}{1-e_i}(1-Z_i)(Y_i-f_i)  \right] \\
        \frac{1}{n} \sum_{i=1}^n Z_i
        \end{pmatrix}  - \begin{pmatrix}
         \frac{1}{n}\sum_{i=1}^n e_i \tau_i \\  \frac{1}{n} \sum_{i=1}^n e_i 
        \end{pmatrix}   \right]  \rightarrow \mathcal{N} (0, I)
    \end{split}
\end{equation*}
\end{small}
where $I$ denotes a two-by-two matrix with diagonal entries $1$. 
\end{proof}

The proof of Lemma \ref{cons_exp} relies on Lemmas \ref{dif_tilde_Sigma_tilde_Sigma_star} and \ref{res_for_cons_var_int_1}.

\begin{lemma}
\label{dif_tilde_Sigma_tilde_Sigma_star}
Let \(\widetilde{\Sigma}_n\) be as defined in Corollary \ref{cor:var} and $\widetilde {\Sigma}^*_n := \frac{1}{\mbox{$\sum_{i=1}^n$} Z_i }\sum_{i=1}^n  [ Z_i(Y_i- \tau-\hat f_i) - (Y_i-\hat f_i)(1-Z_i)\hat e_i/(1-\hat e_i)]^2$.
Under Assumptions \ref{cond_unconf}, \ref{ass:cond_ind}, \ref{assump:outc_mod}, and \ref{assum:DR}, there exists a constant \(C > 0\) such that with probability at least \(1 - 2\delta - \exp(-nc^2/8) - \Delta_n - 2/n\),  
\[
\left\lvert \widetilde{\Sigma}_n - \widetilde{\Sigma}^*_n \right\rvert \leq C \log \left( \frac{2n^2}{\delta} \right) n^{-1}.
\]
%Under Assumptions \ref{cond_unconf}, 
%    \ref{ass:cond_ind}, \ref{assump:outc_mod} and \ref{assum:DR}, with probability at least \( 1 - 2\delta - \exp(-nc^2/8) - \Delta_n - 2/n \), the difference between $\widetilde{\Sigma}_n$, defined in Corollary \ref{cor:var}, and $\widetilde{\Sigma}^*_n$, given by $\widetilde {\Sigma}^*_n := \frac{1}{\mbox{$\sum_{i=1}^n$} Z_i }\sum_{i=1}^n  \{ Z_i(Y_i- \tau-\hat f_i) - (Y_i-\hat f_i)(1-Z_i)\hat e_i/(1-\hat e_i)\}^2$, is at most of order ${\log \left(2n^2/{\delta} \right)} \cdot n^{-1}$, namely, for some constant $C>0$, 
\end{lemma}

\begin{lemma}
\label{res_for_cons_var_int_1}
Under Assumptions \ref{ass:cond_ind} and \ref{assump:outc_mod}, with probability at least $1- 2\delta$, we have,  
\begin{small}
   \begin{equation*}
    \begin{split}
      &  \left\lvert {(\mbox{$\sum_{i=1}^n$} Z_i)^{-4} } \mbox{$\sum_{i=1}^n$} \mbox{$\sum_{j=1}^n$} \mbox{$\mbox{$\sum_{h=1}^n$}$} Z_i Z_j Z_h (\tau_i-\tau_j)(\tau_i-\tau_h) \right. \\
      & \left. - { (\mbox{$\sum_{i=1}^n$} e_i)^{-4} } \left[\sum_{i,j,h: j\neq h} e_i e_j e_h (\tau_i-\tau_j)(\tau_i-\tau_h) + \sum_{i,j,h: j= h} e_i e_j  (\tau_i-\tau_j)^2 \right] \right\rvert \leq C  \sqrt{\log \left( \frac{2}{\delta} \right)} \cdot n^{-\frac{3}{2}} \\
      \end{split}
\end{equation*}
\end{small}
\noindent where $C>0$ is a constant.
\end{lemma}

\begin{proof}[Proof of Lemma \ref{cons_exp}]
%Let $\widetilde {\Sigma}^*_n := \frac{1}{\mbox{$\sum_{i=1}^n$} Z_i }\sum_{i=1}^n  \{ Z_i(Y_i- \tau-\hat f_i) - (Y_i-\hat f_i)(1-Z_i)\hat e_i/(1-\hat e_i)\}^2$. We have
We begin by decomposing the term \(\left| (\sum_{i=1}^n Z_i)^{-1}  \widetilde{\Sigma}_n - S \right|\) as 
\begin{equation*}
\begin{split}
   & \left\lvert \frac{1}{\mbox{$\sum_{i=1}^n$} Z_i } \widetilde{\Sigma}_n - S \right\rvert \leq \left\lvert  \frac{1}{\mbox{$\sum_{i=1}^n$} Z_i } \left( \widetilde{\Sigma}_n - \widetilde{\Sigma}^*_n \right)\right\rvert + \left\lvert \frac{1}{\mbox{$\sum_{i=1}^n$} Z_i } \widetilde{\Sigma}^*_n- S \right\rvert %=:  (I) + (II) \\
   \end{split}
\end{equation*}
where $\widetilde{\Sigma}^*_n$ is defined in Lemma \ref{dif_tilde_Sigma_tilde_Sigma_star}. %We now establish an upper bound for the term \(\left| (\sum_{i=1}^n Z_i)^{-1} \widetilde{\Sigma}^*_n - S\right|\).
For term $\left\lvert (\sum_{i=1}^n Z_i)^{-1} \widetilde{\Sigma}^*_n- S\right\rvert$, we have 
\begin{small}
\begin{equation*}
%\label{proof_cons_exp_int_1}
\begin{split}
    & \left\lvert\frac{1}{\mbox{$\sum_{i=1}^n$} Z_i} \widetilde{\Sigma}^*_n - S \right\rvert= \left\lvert \frac{1}{ (\mbox{$\sum_{i=1}^n$} Z_i)^2 }  \mbox{$\sum_{i=1}^n$}\left[ Z_i(Y_i -\tau - \hat{f}_i )-(Y_i- \hat{f}_i) (1-Z_i) \frac{ \hat{e}_i }{(1-\hat{e}_i)}\right]^2-S \right\rvert=_{(1)} \\
   & \left\lvert \frac{1}{ (\mbox{$\sum_{i=1}^n$} Z_i)^2 } \sum_{i=1}^n \left[ Z^2_i \left( \tau_i- \frac{\mbox{$\sum_{j=1}^n$} Z_j \tau_j}{ \mbox{$\sum_{j=1}^n$} Z_j} +f_i -\hat{f}_i +\epsilon_i   \right)^2 + \left( (f_i-\hat{f}_i+\epsilon_i )(1-Z_i) \frac{\hat{e}_i}{1-\hat{e}_i} \right)^2    \right]-S \right\rvert \\
   & = \left\lvert \frac{1}{ (\mbox{$\sum_{i=1}^n$} Z_i)^2 }  \mbox{$\sum_{i=1}^n$} \left\lbrace Z_i \left[ \left( \frac{\mbox{$\sum_{j=1}^n$} Z_j (\tau_i-\tau_j)}{ \mbox{$\sum_{j=1}^n$} Z_j  }  \right)^2  + \left( f_i -\hat{f}_i+ \epsilon_i  \right)^2 \right.\right. \right. \\
   & \left. \left.\left. + 2 \frac{ \mbox{$\sum_{i=1}^n$} Z_j (\tau_i-\tau_j)(f_i-\hat{f}_i+\epsilon_i) }{\mbox{$\sum_{j=1}^n$} Z_j}   \right] + (1-Z_i)\left( (f_i-\hat{f}_i+ \epsilon_i) \frac{\hat{e}_i}{1-\hat{e}_i}   \right)^2     \right\rbrace-S \right\rvert \\
   &\leq_{(a)} \mid {(\mbox{$\sum_{i=1}^n$} Z_i)^{-4} } \mbox{$\sum_{i=1}^n$} \mbox{$\sum_{j=1}^n$} \mbox{$\mbox{$\sum_{h=1}^n$}$} Z_i Z_j Z_h (\tau_i-\tau_j)(\tau_i-\tau_h) \\
   &  - { (\mbox{$\sum_{i=1}^n$} e_i)^{-4} } \sum_{i,j,h: j\neq h} e_i e_j e_h (\tau_i-\tau_j)(\tau_i-\tau_h) - \sum_{i,j,h: j= h} e_i e_j  (\tau_i-\tau_j)^2 \mid \\
   & +_{(b)} \left\lvert \frac{1}{ (\mbox{$\sum_{i=1}^n$} Z_i)^2 } \mbox{$\sum_{i=1}^n$} \left( Z_i (\epsilon_i+f_i-\hat{f}_i)+(\epsilon_i+f_i-\hat{f}_i) (1-Z_i)\frac{\hat{e}_i}{1-\hat{e}_i} \right)^2 - \frac{\mathbb{E}_{Z,Y}(\Gamma_n) }{ (\mbox{$\sum_{i=1}^n$} e_i)^2 } \right\rvert\\
   & +_{(c)} \left\lvert \frac{2}{(\mbox{$\sum_{i=1}^n$} Z_i)^3} \mbox{$\sum_{i=1}^n$} \mbox{$\sum_{j=1}^n$} Z_i Z_j (\tau_i-\tau_j)(f_i-\hat{f}_i + \epsilon_i)\right\rvert. \\
   %& \leq_{(1)}    \frac{ \sum_{i,j,h: j\neq h} e_i e_j e_h (\tau_i-\tau_j)(\tau_i-\tau_h) + \sum_{i,j,h: j= h} e_i e_j  (\tau_i-\tau_j)^2 }{ (\mbox{$\sum_{i=1}^n$} e_i)^4 } +\frac{\mathbb{E}_{Y,Z}(\Gamma_n) }{ (\mbox{$\sum_{i=1}^n$} e_i)^2 }+ C  {\log^{\frac{1}{2}} \left( \frac{2n^2}{\delta} \right)}   n^{-\frac{3}{2} } \\
   %& + C \cdot \sqrt{\log \left( 2n^2/{\delta} \right)} \cdot  n^{-\frac{3}{2} }
\end{split}
\end{equation*}
\end{small}
\((1)\) follows from the fact that the interaction term \( Z_i(1 - Z_i)(Y_i - \tau - \hat{f}_i)(Y_i - \hat{f}_i) \hat{e}_i(1 - \hat{e}_i)^{-1} = 0 \) for \( i \in [n] \). We next derive bounds for terms $(a)$, $(b)$ and $(c)$ in the last equality. For term \((a)\), we have
\begin{equation}
\label{cons_exp_part3}
    \begin{split}
        (a)\leq C  \ {\log^{\frac{1}{2}} \left( {2}/{\delta} \right)} n^{-\frac{3}{2}}
    \end{split}
\end{equation}
which follows from Lemma \ref{res_for_cons_var_int_1}, with probability $1-2\delta$. Next, for term $(b)$, we define $d(\hat{e}_i,e_i):=\frac{\hat{e}_i}{1-\hat{e}_i}-\frac{{e}_i}{1-{e}_i}$ and $a_i=\epsilon_i(1-Z_i)e_i(1-e_i)^{-1}$ for $i\in[n]$. Then we obtain the following bound.
\begin{small}
\begin{equation}
\label{cons_exp_part1}
 \begin{split}
 & \left\lvert \frac{1}{ (\mbox{$\sum_{i=1}^n$} Z_i)^2 } \mbox{$\sum_{i=1}^n$} \left( Z_i (\epsilon_i+f_i-\hat{f}_i)+(\epsilon_i+f_i-\hat{f}_i) (1-Z_i)\frac{\hat{e}_i}{1-\hat{e}_i} \right)^2- \frac{\mathbb{E}_{Z,Y}(\Gamma_n) }{ (\mbox{$\sum_{i=1}^n$} e_i)^2 } \right\rvert  \\
 &=_{(1)} \left\lvert \frac{1}{ (\mbox{$\sum_{i=1}^n$} Z_i)^2 } \mbox{$\sum_{i=1}^n$} \left( Z_i (\epsilon_i+f_i-\hat{f}_i)-a_i+(\epsilon_i+f_i-\hat{f}_i) (1-Z_i)\frac{\hat{e}_i}{1-\hat{e}_i}+a_i \right)^2- \frac{\mathbb{E}_{Z,Y}(\Gamma_n) }{ (\mbox{$\sum_{i=1}^n$} e_i)^2 } \right\rvert  \\
 & = \left\lvert \frac{1}{(\mbox{$\sum_{i=1}^n$} Z_i)^2} {\sum_{i=1}^n}\left\lbrace \left[ Z_i \epsilon_i + \epsilon_i  \frac{(1- Z_i)e_i}{1-e_i} \right]^2 + \left[ ( Z_i + \frac{(1-Z_i)\hat{e}_i}{1-\hat{e}_i} ) \left( f_i -\hat{f}_i \right)+ \epsilon_i(1-Z_i) d(\hat{e}_i,e_i) \right]^2 \right.\right.  \\
 & \left. \left. + 2 \left[  Z_i \epsilon_i + \epsilon_i  \frac{(1-Z_i)e_i}{1-e_i} \right]\cdot \left[ \left(Z_i+ \frac{(1-Z_i)\hat{e}_i}{1-\hat{e}_i}  \right) (f_i-\hat{f}_i)+\epsilon_i (1-Z_i)d(\hat{e}_i,e_i) \right]\right\rbrace  - \frac{\mathbb{E}_{Z,Y}(\Gamma_n)}{(\mbox{$\sum_{i=1}^n$} e_i)^2} \right\rvert \\
 & \leq_{(2)} \left\lvert \frac{\Gamma_n}{(\mbox{$\sum_{i=1}^n$} Z_i)^2 }- \frac{\mathbb{E}_{Z,Y}(\Gamma_n)}{ (\mbox{$\sum_{i=1}^n$} e_i)^2}  \right\rvert+ \frac{C}{ \left( \mbox{$\sum_{i=1}^n$} Z_i  \right)^2} {\sum_{i=1}^n} \left[ %( Z_i +  \frac{(1-Z_i)\hat{e}_i}{1-\hat{e}_i} )^2 
 (f_i-\hat{f}_i)^2+\epsilon_i^2(1-Z_i)d^2(\hat{e}_i,e_i) \right] \\
 & + \left\lvert \frac{2}{(\mbox{$\sum_{i=1}^n$} Z_i)^2 } \mbox{$\sum_{i=1}^n$} \left[  \left(Z_i + (1-Z_i)\frac{e_i}{1-e_i}\frac{\hat{e}_i}{1-\hat{e}_i} \right) \epsilon_i (f_i-\hat{f}_i)+  \left(1-Z_i\right) \frac{e_i}{1-e_i}\epsilon^2_i \ d(\hat{e}_i,e_i) \right]  \right\rvert \\
 & \leq \left\lvert \left[ ({\sum_{i=1}^n} Z_i )^{-2}  - ( {\sum_{i=1}^n} e_i )^{-2} \right] \cdot \Gamma_n \right\rvert+  \left\lvert  ( {\sum_{i=1}^n} e_i )^{-2} \left( \Gamma_n - \mathbb{E}_{Z,Y} (\Gamma_n)  \right)\right\rvert + \frac{C}{(\mbox{$\sum_{i=1}^n$} Z_i)^2 } \mbox{$\sum_{i=1}^n$} \left( f_i-\hat{f}_i \right)^2 \\
 & + \frac{C\sqrt{\mbox{$\sum_{i=1}^n$} \epsilon^4_i}}{(\mbox{$\sum_{i=1}^n$} Z_i)^2 }  \left( \sqrt{\mbox{$\sum_{i=1}^n$} d^4(\hat{e}_i,e_i)}+\sqrt{\mbox{$\sum_{i=1}^n$} d^2(\hat{e}_i,e_i)}  \right) + \frac{C}{(\mbox{$\sum_{i=1}^n$} Z_i)^2 } \sqrt{\mbox{$\sum_{i=1}^n$} \epsilon^2_i} \sqrt{\mbox{$\sum_{i=1}^n$} (f_i-\hat{f}_i)^2}  \\
 & \leq_{(3.1)} C n^{-\frac{3}{2}}\log^{\frac{1}{2}}(2/\delta)+_{(3.2)} C n^{-\frac{3}{2}}\log^{1/2}(n) + _{(3.3)}n^{-1}\left( \delta^2_n n^{-1}+ \log(2n^2/\delta)n^{-\frac{1}{2}}\delta_n^2\log^{-1}(n) \right)  \\
 & +_{(3.4)} C \left[ n^{-2} n^{\frac{1}{2}} \log^{1/2}(n) \left( \sqrt{\mbox{$\sum_{i=1}^n$} d^2(\hat{e}_i,e_i)} + \log^{1/2} (n) \right)\right] \\
 & +_{(3.5)}C n^{-2} n^{\frac{1}{2}} \log^{1/2}(n) \left( \delta_n+ \log^{\frac{1}{2}}(2n^2/\delta)n^{\frac{1}{4}}\delta_n\log^{-\frac{1}{2}}(n)\right) \\
 & \leq C n^{-5/4} \log^{1/2}(2n^2/\delta)\delta_n\log^{-1/2}(n)
\end{split}
\end{equation}
\end{small}
\noindent with probability at least \(1 - \exp(-nc^2/8) - \delta - \Delta_n - 3n^{-1}\). %and where $a_i=\epsilon_i(1-Z_i)e_i(1-e_i)^{-1}$ for $i\in[n]$.   
\((2)\) follows from the upper bound on \(\hat{e}_i\) specified in Assumption \ref{assum:DR} (i). %The first term in $(3)$ 
$(3.1)$ is derived from the fact that \begin{equation}
\label{proof_res_for_cons_var_int_1_int_1}
    \begin{split}
  \left\lvert \mbox{$\sum_{i=1}^n$} Z_i - \mbox{$\sum_{i=1}^n$} e_i\right\rvert \leq \sqrt{ 12 \log \left(2/{\delta}\right)} \cdot n^{\frac{1}{2}}  
    \end{split}
\end{equation}
holds with probability $1-\delta$ (\cite{angluin1977fast}). %The second term in \((3)\) 
$(3.2)$ is derived by applying Bernstein inequality on $\Gamma_n-\mathbb{E}_{Z,Y}(\Gamma_n)$. Then with probability $1-2{n}^{-1}$, $|\Gamma_n-\mathbb{E}_{Z,Y}(\Gamma_n)|\leq C \log^{1/2}(n) n^{1/2}$. 
%conditional on \(Z\), to the sum \(\sum_{i=1}^n a_i \epsilon_i^2\), where \(a_i = Z_i + (1 - Z_i) e_i^2 / (1 - e_i)^2\) for \(i \in [n]\). 
% The third and fifth terms in \((3)\) 
$(3.3)$ and $(3.5)$ are derived using Lemma %\href{lemma:sparsity_gamma_star}{10}
\ref{lemma:sparsity_gamma_star}, Assumption \ref{assum:DR} (ii), and Bernstein inequality for \(\sum_{i=1}^n \epsilon_i^2\). %The fourth term in \((3)\) 
$(3.4)$ follows from Assumption \ref{assum:DR} (i) and apply Bernstein inequality to \(\sum_{i=1}^n \epsilon_i^4\). For term $(c)$, we have
\begin{small}
\begin{equation}
\label{cons_exp_part2}
    \begin{split}
    &\left\lvert \frac{2}{(\mbox{$\sum_{i=1}^n$} Z_i)^3} \mbox{$\sum_{i=1}^n$} \mbox{$\sum_{j=1}^n$} Z_i Z_j (\tau_i-\tau_j)(f_i-\hat{f}_i + \epsilon_i) \right\rvert \\
    & \leq_{(1)} \frac{2}{(\mbox{$\sum_{i=1}^n$} Z_i)^3} \left[ \sqrt{{\sum_{i=1}^n} (f_i-\hat{f}_i)^2} \sqrt{{\sum_{i=1}^n} \left( \mbox{$\sum_{j=1}^n$} Z_j (\tau_i-\tau_j) \right)^2} + \left\lvert  \mbox{$\sum_{i=1}^n$} \left[  \mbox{$\sum_{j=1}^n$} Z_i Z_j (\tau_i -\tau_j)  \right] \epsilon_i \right\rvert \right]\\
    & \leq C \frac{1}{n^3} \left[ \left( \delta_n+ \log^{\frac{1}{2}}(\frac{2n^2}{\delta})n^{\frac{1}{4}}\delta_n\log^{-\frac{1}{2}}(n)\right) n^{\frac{3}{2}}+  n^{\frac{3}{2}} \log^{\frac{1}{2}}(n) \right] \leq C n^{-\frac{5}{4}} \log^{\frac{1}{2}}(\frac{2n^2}{\delta})\delta_n\log^{-\frac{1}{2}}(n) \\ 
    \end{split}
\end{equation}
\end{small}
with probability at least \(1 - \exp(-nc^2/8) - \delta - \Delta_n - n^{-1}\). The bound for the second term in \((1)\) is obtained by applying Bernstein's inequality, conditional on \(Z\), to \(\sum_{i=1}^n b_{i,n} \epsilon_i^2\), where \(b_{i,n} = \sum_{j=1}^n Z_i Z_j (\tau_i - \tau_j)\) for \(i \in [n]\). %Let
%\begin{small}
%    $$S:=\frac{1}{(\mbox{$\sum_{i=1}^n$} e_i)^{4} } \left[ \sum_{i,j,h: j\neq h} e_i e_j e_h (\tau_i-\tau_j)(\tau_i-\tau_h) + \sum_{i,j,h: j= h} e_i e_j  (\tau_i-\tau_j)^2 \right] +\frac{\mathbb{E}_{Z,Y}(\Gamma_n)}{(\mbox{$\sum_{i=1}^n$} e_i)^{2}}.$$\end{small}
Then we have 
\begin{small}
\begin{equation}
\label{proof_cons_exp_int_1}
    \begin{split}
    & \left\lvert \frac{1}{\mbox{$\sum_{i=1}^n$} Z_i} \widetilde{\Sigma}^*_n - S\right\rvert \leq (a)+(b)+(c) \leq_{(1)}C n^{-5/4} \log^{1/2}(2n^2/\delta)\delta_n\log^{-1/2}(n) \\
     %  & \left\lvert \frac{1}{\mbox{$\sum_{i=1}^n$} Z_i} \widetilde{\Sigma}^*_n - S\right\rvert \leq_{(1)} C  {\log^{\frac{1}{2}} \left( {2n^2}/{\delta} \right)}   n^{-\frac{3}{2} }
    \end{split}
\end{equation}
\end{small}
%\begin{small}
%\begin{equation}
%\label{proof_cons_exp_int_1}
%    \begin{split}
%       & \left\lvert \frac{1}{\mbox{$\sum_{i=1}^n$} Z_i} \widetilde{\Sigma}^*_n - \frac{ \sum_{i,j,h: j\neq h} e_i e_j e_h (\tau_i-\tau_j)(\tau_i-\tau_h) + \sum_{i,j,h: j= h} e_i e_j  (\tau_i-\tau_j)^2 }{ (\mbox{$\sum_{i=1}^n$} e_i)^4 } -\frac{\mathbb{E}_{Z,Y}(\Gamma_n) }{ (\mbox{$\sum_{i=1}^n$} e_i)^2 }\right\rvert \\
%       & \leq_{(1)} C  {\log^{\frac{1}{2}} \left( \frac{2n^2}{\delta} \right)}   n^{-\frac{3}{2} }
%    \end{split}
%\end{equation}
%\end{small}
\noindent with probability at least $1-3\delta-\Delta_n-\exp(-nc^2/8)-4n^{-1}$. 
$(1)$ follows from equations \eqref{cons_exp_part3}, \eqref{cons_exp_part1} and \eqref{cons_exp_part2}.
%the fact that $\Gamma_n$ can also be expressed as $\mbox{$\sum_{i=1}^n$} \left[ Z_i \epsilon_i - \frac{e_i}{1 - e_i} (1 - Z_i) \epsilon_i \right]^2$, and the result from Lemma \ref{lemma:sparsity_gamma_star}. 
%Next, note that 
%\begin{small}
%\begin{equation}
%    \begin{split}
%        \label{proof_cons_exp_int_2}
%\left\lvert \mathbb{E}_{Z,Y} \left(\frac{1}{\mbox{$\sum_{i=1}^n$} Z_i} \widetilde{\Sigma}^*_n \right)    - S \right\rvert \leq \mathbb{E}_{Z,Y}\left[ \left\lvert \left(\frac{1}{\mbox{$\sum_{i=1}^n$} Z_i} \widetilde{\Sigma}^*_n \right)    - S \right\rvert \right] \leq_{(1)} C n^{-\frac{5}{4}} \log^{\frac{1}{2}}(2n^2/\delta)\delta_n\log^{-\frac{1}{2}}(n)
%    \end{split}
%\end{equation}
%\end{small}
%where $(1)$ follows from formula \eqref{proof_cons_exp_int_1}.
Finally, with probability at least $1-5\delta-2\Delta_n-\exp(-nc^2/8)-6n^{-1}$, we have
\begin{small}
\begin{equation*}
\begin{split}
  %& \left\lvert \frac{1}{\mbox{$\sum_{i=1}^n$} Z_i } \widetilde{\Sigma}_n - \mathbb{E} \left( \frac{1}{\mbox{$\sum_{i=1}^n$} Z_i} \widetilde{\Sigma}^*_n     \right) \right\rvert \leq \left\lvert  \frac{1}{\mbox{$\sum_{i=1}^n$} Z_i } \left( \widetilde{\Sigma}_n - \widetilde{\Sigma}^*_n \right)\right\rvert + \left\lvert \frac{1}{\mbox{$\sum_{i=1}^n$} Z_i } \widetilde{\Sigma}^*_n- \mathbb{E} \left( \frac{1}{\mbox{$\sum_{i=1}^n$} Z_i} \widetilde{\Sigma}^*_n     \right) \right\rvert \\
   & \left\lvert \frac{1}{\mbox{$\sum_{i=1}^n$} Z_i } \widetilde{\Sigma}_n - S \right\rvert \leq \left\lvert  \frac{1}{\mbox{$\sum_{i=1}^n$} Z_i } \left( \widetilde{\Sigma}_n - \widetilde{\Sigma}^*_n \right)\right\rvert + \left\lvert \frac{1}{\mbox{$\sum_{i=1}^n$} Z_i } \widetilde{\Sigma}^*_n- S \right\rvert \\
   & \leq_{(1)} C \cdot  { \log \left(2n^2/{\delta} \right)} \cdot n^{-2} + C n^{-\frac{5}{4}} \log^{\frac{1}{2}}(2n^2/\delta)\delta_n\log^{-\frac{1}{2}}(n) \leq C n^{-\frac{5}{4}} \log^{\frac{1}{2}}(2n^2/\delta)\delta_n\log^{-\frac{1}{2}}(n).
   \end{split}
\end{equation*}
\end{small}
The constant \(C\) varies for each inequality. (1) follows from Lemma \ref{dif_tilde_Sigma_tilde_Sigma_star} and by equation \eqref{proof_cons_exp_int_1}. 
%with \eqref{proof_cons_exp_int_2}. 
\end{proof}

\begin{proof}[Proof of Lemma \ref{int_upper_bound}]
   Let $\mathbf{\tau}_{(i)}=(\tau_i-\tau_1, \cdots, \tau_i-\tau_n)^T$ and $\mathbf{e}=(e_1, \cdots, e_n)^T$. Then for each $i \in [n]$, 
   $ \mbox{$\sum_{j=1}^n$} \mbox{$\sum_{h=1}^n$} e_j e_h (\tau_i-\tau_j) (\tau_i-\tau_h) = \mathbf{\tau}^T_{(i)} \cdot \mathbf{e} \cdot \mathbf{e}^T \mathbf{\tau}_{(i)}= \left( \mathbf{\tau}^T_{(i)} \mathbf{e}  \right)^2 \geq 0.$
%\begin{equation*}
%       \begin{split}
%           \mbox{$\sum_{j=1}^n$} \mbox{$\sum_{h=1}^n$} e_j e_h (\tau_i-\tau_j) (\tau_i-\tau_h) = \mathbf{\tau}^T_{(i)} \cdot \mathbf{e} \cdot \mathbf{e}^T \mathbf{\tau}_{(i)}= \left( \mathbf{\tau}^T_{(i)} \mathbf{e}  \right)^2 \geq 0.
%       \end{split}
%   \end{equation*}
%Then the corollary holds. 
\end{proof}

\begin{proof}[Proof of Lemma \ref{dif_e_i_hat_e_i}]
Let $d({e}_i,\hat{e}_i):=\frac{{e}_i}{1-{e}_i}-\frac{\hat{e}_i}{1-\hat{e}_i}$ for $i\in [n]$. For each $i\in [n]$, we have
\begin{small}
\begin{equation}
\label{proof_dif_e_i_hat_e_i_part1}
\begin{split}
 &|e_i-\hat{e}_i|= \left\lvert\frac{e_i}{1-e_i} \left(1+\frac{e_i}{1-e_i}\right)^{-1}-\frac{\hat{e}_i}{1-\hat{e}_i} \left(1+\frac{\hat{e}_i}{1-\hat{e}_i}\right)^{-1} \right\rvert \\
 &= \left\lvert \left(\frac{e_i}{1-e_i} -\frac{\hat{e}_i}{1-\hat{e}_i}\right) \left(1+\frac{{e}_i}{1-{e}_i}\right)^{-1} -  \frac{\hat{e}_i}{1-\hat{e}_i} \left[ \left(1+\frac{\hat{e}_i}{1-\hat{e}_i}\right)^{-1}-\left(1+\frac{{e}_i}{1-{e}_i}\right)^{-1} \right] \right\rvert \\
 & = \left\lvert d(e_i,\hat{e}_i) (1-e_i)- \frac{\hat{e}_i}{1-\hat{e}_i} \left(1+\frac{\hat{e}_i}{1-\hat{e}_i}\right)^{-1} \left(1+\frac{{e}_i}{1-{e}_i}\right)^{-1} \left[1+\frac{{e}_i}{1-{e}_i}- \left(1+\frac{\hat{e}_i}{1-\hat{e}_i}\right) \right] \right\rvert\\
 & \leq |(1-e_i)+\hat{e}_i(1-e_i)| \cdot |d(e_i,\hat{e}_i)|.
 \end{split}
\end{equation}
\end{small}
Then we obtain
\begin{small}
\begin{equation*}
    \begin{split}
    \frac{1}{n}\sum_{i=1}^n (e_i-\hat{e}_i)^2 \leq_{(1)} \frac{1}{n}\sum_{i=1}^n [(1-e_i)(1+\hat{e}_i)]^2 d^2(e_i,\hat{e}_i) \leq_{(2)} C \frac{1}{n}\sum_{i=1}^n d^2(e_i,\hat{e}_i) \leq_{(3)} C \delta^2_n
    \end{split}
\end{equation*}
\end{small}
\((1)\) follows from equation \eqref{proof_dif_e_i_hat_e_i_part1}. \((2)\) follows from the lower bound of \( e_i \) and the upper bound of \( \hat{e}_i \) for \( i \in [n] \), as assumed in Assumptions \ref{ass:cond_ind} and \ref{assum:DR} (ii). \((3)\) follows from Assumption \ref{assum:DR} (ii), with probability at least \( 1 - \Delta_n \).

\end{proof}

\begin{proof}[Proof of Lemma \ref{dif_tilde_Sigma_tilde_Sigma_star}]
Let 
$A_i:=Z_i (Y_i -\hat{\tau}^{AIPW}-\hat{f}_i)-(Y_i-\hat{f}_i)(1-Z_i)\frac{\hat{e}_i}{1-\hat{e}_i}$ and $B_i:=Z_i (Y_i -{\tau}-\hat{f}_i)-(Y_i-\hat{f}_i)(1-Z_i)\frac{\hat{e}_i}{1-\hat{e}_i}$ for $i\in [n]$. Then we express the difference $|\widetilde{\Sigma}_n- \widetilde{\Sigma}^*_n |$ as
%$A_i:=Z_i(-\hat{\tau}^{AIPW}-\tau)+2\cdot \left( Z_i(Y_i-\hat{f}_i)-(Y_i-\hat{f}_i)(1-Z_i) \frac{\hat{e}_i}{1-\hat{e}_i}  \right)$ for $i\in [n]$. Then we have 
\begin{small}
\begin{equation}
\label{dif_tilde_Sigma_tilde_Sigma_star_part2}
    \begin{split}
    & |  \widetilde{\Sigma}_n- \widetilde{\Sigma}^*_n |=\frac{1}{\mbox{$\sum_{i=1}^n$} Z_i} \sum_{i=1}^n (A_i-B_i)(A_i+B_i)= \left\lvert \frac{1}{\mbox{$\sum_{i=1}^n$} Z_i} \sum_{i=1}^n Z_i (\tau-\hat{\tau}^{AIPW}) (A_i+B_i) \right\rvert \\
    & \leq \left\lvert \frac{\mbox{$ \sum_{i=1}^n$ } Z_i \cdot  (A_i+B_i) }{\mbox{$\sum_{i=1}^n$} Z_i}  \right\rvert \cdot |\hat{\tau}^{AIPW}-\hat{\tau}^{AIPW*}+\hat{\tau}^{AIPW*}-\tau|
    \end{split}
\end{equation}
\end{small}
where \begin{small}$(A_i+B_i)=Z_i(-\hat{\tau}^{AIPW}-\tau)+2\cdot \left( Z_i(Y_i-\hat{f}_i)-(Y_i-\hat{f}_i)(1-Z_i) \frac{\hat{e}_i}{1-\hat{e}_i}  \right)$\end{small} for $i\in [n]$.
%\begin{small}
%\begin{equation}
%\label{dif_tilde_Sigma_tilde_Sigma_star_part2}
%    \begin{split}
%    & |  \widetilde{\Sigma}_n- \widetilde{\Sigma}^*_n |=\frac{1}{\mbox{$\sum_{i=1}^n$} Z_i} \sum_{i=1}^n \left\lbrace Z_i (Y_i -\hat{\tau}^{AIPW}-\hat{f}_i)-(Y_i-\hat{f}_i)(1-Z_i)\frac{\hat{e}_i}{1-\hat{e}_i}- \left[Z_i (Y_i -{\tau}-\hat{f}_i) \right.\right.  \\
%    &\left.\left. -(Y_i-\hat{f}_i)(1-Z_i)\frac{\hat{e}_i}{1-\hat{e}_i} \right] \right\rbrace \left\lbrace Z_i (Y_i -\hat{\tau}^{AIPW}-\hat{f}_i)-(Y_i-\hat{f}_i)(1-Z_i)\frac{\hat{e}_i}{1-\hat{e}_i}+ \left[Z_i (Y_i -{\tau}-\hat{f}_i) \right.\right.  \\
%    &\left.\left. -(Y_i-\hat{f}_i)(1-Z_i)\frac{\hat{e}_i}{1-\hat{e}_i} \right] \right\rbrace = \left\lvert \frac{1}{\mbox{$\sum_{i=1}^n$} Z_i} \sum_{i=1}^n Z_i (\tau-\hat{\tau}^{AIPW}) A_i \right\rvert \\
%    & \leq \left\lvert \frac{\mbox{$ \sum_{i=1}^n$ } Z_i \cdot  A_i }{\mbox{$\sum_{i=1}^n$} Z_i}  \right\rvert \cdot |\hat{\tau}^{AIPW}-\hat{\tau}^{AIPW*}+\hat{\tau}^{AIPW*}-\tau|.
%    \end{split}
%\end{equation}
%\end{small}
We then derive an upper bound for $|\hat{\tau}^{AIPW*}-\tau|$.
\begin{equation}
\label{dif_tilde_Sigma_tilde_Sigma_star_part3}
    \begin{split}
    & \left\lvert \hat{\tau}^{AIPW*}-\tau  \right\rvert=\left\lvert ({\mbox{$\sum_{i=1}^n$}Z_i })^{-1} \ \mbox{$\sum_{i=1}^n$} \left[Z_i-{e_i(1-e_i)^{-1}(1-Z_i)}\right]\epsilon_i \right\rvert   \\
    & \leq_{(1)} {(\mbox{$\sum_{i=1}^n$} Z_i)^{-1}} { \left\lbrace 2 \mbox{$\sum_{i=1}^n$} \left[Z_i-{e_i(1-e_i)^{-1}(1-Z_i)}\right]^2 \bar{\sigma}^2 \right\rbrace^{1/2} \log^{1/2}n } \\
    & \leq_{(2)} (\mbox{$\sum_{i=1}^n$} Z_i)^{-1}\left[ C \mbox{$\sum_{i=1}^n$} (Z^2_i+(1-Z_i)^2) \right]^{1/2}\log^{1/2}(n)  \leq_{(3)}  C n^{-\frac{1}{2}} \log^{\frac{1}{2}} (n) 
    \end{split}
\end{equation}
$(1)$ follows from the tail bound for the subgaussian random variable \( \sum_{i=1}^n [Z_i - e_i(1 - e_i)^{-1} (1 - Z_i)] \epsilon_i \), conditional on \( Z \), which holds with probability at least \( 1 - 2/n \). $(2)$ follows from the inequality $[Z_i-e_i(1-e_i)^{-1}(1-Z_i)]^2\leq C[Z^2_i+(1-Z_i)^2]$ for $i\in [n]$ under Assumption \ref{ass:cond_ind}. $(3)$ follows from inequality \eqref{OR:ATT:part2}, which holds with probability at least \( 1 - \exp(-nc^2/8) \). Next, we proceed to derive an upper bound for $|(\sum_{i=1}^n Z_i)^{-1} \sum_{i=1}^n Z_i (A_i+B_i)|$. 
\begin{small}
\begin{equation}
\label{dif_tilde_Sigma_tilde_Sigma_star_part1}
    \begin{split}
    & \left\lvert \frac{\mbox{$ \sum_{i=1}^n$ } Z_i (A_i+B_i) }{\mbox{$\sum_{i=1}^n$} Z_i}  \right\rvert= \left\lvert -(\hat{\tau}^{AIPW} -\tau )+2 \frac{ \mbox{$\sum_{i=1}^n$} Z_i \epsilon_i}{ \mbox{$\sum_{i=1}^n$} Z_i}+ 2  \frac{ \mbox{$\sum_{i=1}^n$} Z_i (f_i-\hat{f}_i)}{ \mbox{$\sum_{i=1}^n$} Z_i}\right\rvert \\
    & \leq_{(1.1)}  C  \left[ \delta^{\frac{1}{2}}_n  n^{-\frac{1}{2}} \Sigma^{-\frac{1}{2}}_n  +   C n^{-\frac{1}{2}} \log^{\frac{1}{2}}(n) \right]+_{(1.2)} C \log^{\frac{1}{2}} ({2}/{\delta}) n^{-\frac{1}{2}}\\
    & +_{(1.3)} C n^{-\frac{1}{2}} \delta_n \left(1+ \log( {2n^2}/{\delta}) n^{-\frac{1}{2}}\log^{-1}n \right)^{\frac{1}{2}} \leq C \log^{\frac{1}{2}} (2n^2/\delta) n^{-\frac{1}{2}}.
    \end{split}
\end{equation}
\end{small}
\noindent %The first term in the second line follow 
with probability at least \(1 - 2\delta - \exp(-nc^2/8) - \Delta_n-2/n\). $(1.1)$ follows from equation \eqref{dif_tilde_Sigma_tilde_Sigma_star_part3} with probability $1-\exp(-nc^2/8) - 2/n$ and equation \eqref{eq:dr_main} in the proof of Theorem \ref{thm:dra_tau_inf} in the main paper,  
%\cite{belloni2022neighborhood} 
with probability $1-\delta-\Delta_n$. $(1.2)$ follows from equations \eqref{OR:ATT:part1} and \eqref{OR:ATT:part2} with probability $1-\delta-\exp(-nc^2/8)$. %The third term 
$(1.3)$ is derived by applying the Cauchy-Schwarz inequality, followed by Lemma %\href{lemma:sparsity_gamma_star}{10} 
\ref{lemma:sparsity_gamma_star}
and Assumption \ref{assum:DR} (ii) with probability $1-\delta-\Delta_n$. Combining these results, we obtain the following bound for \( | \widetilde{\Sigma}_n - \widetilde{\Sigma}^*_n | \). With probability at least \( 1 - 2\delta - \exp(-nc^2/8) - \Delta_n - 2/n \), we have 
\begin{equation}
    \begin{split}
    & |  \widetilde{\Sigma}_n- \widetilde{\Sigma}^*_n |  % = \left\lvert \frac{1}{\mbox{$\sum_{i=1}^n$} Z_i} \mbox{$\sum_{i=1}^n$} Z_i (\tau-\hat{\tau}^{AIPW})  A_i \right\rvert \leq \left\lvert \frac{\mbox{$ \sum_{i=1}^n$ } Z_i \cdot  A_i }{\mbox{$\sum_{i=1}^n$} Z_i}  \right\rvert \cdot |\hat{\tau}^{AIPW}-\tau| \\
    \leq \left\lvert \frac{\mbox{$ \sum_{i=1}^n$ } Z_i( A_i+B_i) }{\mbox{$\sum_{i=1}^n$} Z_i}  \right\rvert \left(|\hat{\tau}^{AIPW}-\hat{\tau}^{AIPW*}|+|\hat{\tau}^{AIPW*}-\tau| \right)\\
     & \leq_{(1)} C \log^{1/2} (2n^2/\delta) n^{-1/2} \left(\delta^{\frac{1}{2}}_n  n^{-\frac{1}{2}} \Sigma^{-\frac{1}{2}}_n  +   C n^{-\frac{1}{2}} \log^{\frac{1}{2}}(n) \right)\leq C {\log \left( 2n^2/\delta \right) }  n^{-{1}}
     %O_p\left( \sqrt{\log n} \cdot  n^{-\frac{1}{2}} \right)
    \end{split}
\end{equation}
$(1)$ follows from equations \eqref{dif_tilde_Sigma_tilde_Sigma_star_part2}, \eqref{dif_tilde_Sigma_tilde_Sigma_star_part3}, and \eqref{dif_tilde_Sigma_tilde_Sigma_star_part1}, together with the bound on equation \eqref{eq:dr_main} in the main paper.%\cite{belloni2022neighborhood}.
\end{proof}

\begin{proof}[Proof of Lemma \ref{res_for_cons_var_int_1}]
We first decompose the difference into two parts
\begin{small}
\begin{equation}
\label{proof_res_for_cons_var_int_1}
\begin{split}
 & \frac{\mbox{$\sum_{i=1}^n$} \mbox{$\sum_{j=1}^n$} \mbox{$\sum_{h=1}^n$} Z_i Z_j Z_h (\tau_i-\tau_j)(\tau_i-\tau_h)}{\mbox{$(\mbox{$\sum_{i=1}^n$} Z_i)^4$} }  -  \frac{\mbox{$\sum_{i=1}^n$} \mbox{$\sum_{j=1}^n$} \mbox{$\sum_{h=1}^n$} \mathbb{E} \left( Z_i Z_j Z_h (\tau_i-\tau_j)(\tau_i-\tau_h) \right)}{\mbox{$(\mbox{$\sum_{i=1}^n$} e_i)^4$} } \\
 & =_{(a)} \frac{n^4}{ (\mbox{$\sum_{i=1}^n$} Z_i)^4 }  \mbox{$\sum_{i=1}^n$}  \mbox{$\sum_{j=1}^n$} \mbox{$\sum_{h=1}^n$} \frac{1}{n^4} \left[ Z_i Z_j Z_h (\tau_i-\tau_j)(\tau_i-\tau_h) -  \mathbb{E}  \left( Z_i Z_j Z_h (\tau_i-\tau_j)(\tau_i-\tau_h) \right) \right] \\
 & +_{(b)} \frac{(\mbox{$\sum_{i=1}^n$} e_i)^4-(\mbox{$\sum_{i=1}^n$} Z_i)^4}{(\mbox{$\sum_{i=1}^n$} Z_i)^4 \cdot (\mbox{$\sum_{i=1}^n$} e_i)^4 } \cdot  \left[ \mbox{$\sum_{i=1}^n$} \mbox{$\sum_{j=1}^n$} \mbox{$\sum_{h=1}^n$} \mathbb{E} \left( Z_i Z_j Z_h (\tau_i-\tau_j)(\tau_i-\tau_h) \right) \right]. \\
 \end{split}
\end{equation}
\end{small}
\noindent For part $(a)$, note that \(\frac{n^4}{(\sum_{i=1}^n Z_i)^4} \leq C\) with probability at least \(1 - \exp(-nc^2/8)\), as established in \eqref{OR:ATT:part2}. Furthermore, let  
\begin{equation*}
\begin{split}
   h(Z):= \mbox{$\sum_{i=1}^n$}  \mbox{$\sum_{j=1}^n$} \mbox{$\sum_{h=1}^n$} \frac{1}{n^4} \left[ Z_i Z_j Z_h (\tau_i-\tau_j)(\tau_i-\tau_h) -  \mathbb{E}  \left( Z_i Z_j Z_h (\tau_i-\tau_j)(\tau_i-\tau_h) \right) \right].  
\end{split}  
\end{equation*}
Without loss of generality, consider changing the component \(Z_1\) from \(1\) to \(0\). Then
\begin{equation*}
    \begin{split}
      &  |h((1,Z_{-1}))-h((0,Z_{-1}))| \\
      &\leq   \mbox{$\sum_{j=1}^n$} \mbox{$\sum_{h=1}^n$} \frac{1}{n^4} \left\lvert  Z_j Z_h (\tau_1-\tau_j)(\tau_1-\tau_h) -  \mathbb{E}  \left(  Z_j Z_h (\tau_1-\tau_j)(\tau_1-\tau_h) \right) -0 \right\rvert \leq_{(1)} \frac{C}{n^2} \\ 
    \end{split}
\end{equation*}
$(1)$ follows from Assumptions \ref{ass:cond_ind} and \ref{assump:outc_mod}. This bound also holds for all other $i\in[n]$. Applying McDiarmid's inequality (see, e.g., \cite{mcdiarmid1989method}) to \(h(Z)\), we obtain that, with probability at least \(1 - \delta\),
\begin{small}
\begin{equation*}
    \begin{split}
     & \left\lvert {\sum_{i=1}^n} {\sum_{j=1}^n} {\sum_{h=1}^n} \frac{1}{n^4} \left[ Z_i Z_j Z_h (\tau_i-\tau_j)(\tau_i-\tau_h) -  \mathbb{E}  \left( Z_i Z_j Z_h (\tau_i-\tau_j)(\tau_i-\tau_h) \right) \right] \right\rvert  \leq C {\log^{\frac{1}{2} } ( \frac{1}{\delta} )} n^{-\frac{3}{2}}. \\ 
    \end{split}
\end{equation*}
\end{small}
For part $(b)$, we have  
\begin{small}
\begin{equation*}
    \begin{split}
      & \left\lvert \left(\mbox{$\sum_{i=1}^n$} e_i\right)^4- \left(\mbox{$\sum_{i=1}^n$} Z_i\right)^4  \right\rvert = \left\lvert \sum_{i=1}^ne_i-\sum_{i=1}^n Z_i \right\rvert \left( \sum_{i=1}^n e_i + \sum_{i=1}^n Z_i \right) \left[ \left(\mbox{$\sum_{i=1}^n$} e_i\right)^2+ \left(\mbox{$\sum_{i=1}^n$} Z_i\right)^2 \right]  \\
      %= \left\lvert \left(\sum_{i=1}^ne_i-\sum_{i=1}^n Z_i\right)\cdot\left[(\mbox{$\sum_{i=1}^n$} e_i)^2+(\sum_{i=1}^n e_i)\cdot(\sum_{i=1}^n Z_i)+ (\mbox{$\sum_{i=1}^n$} Z_i)^2\right] \right\rvert \\
      & \leq _{(1)} C \cdot \sqrt{\log \left(2/{\delta}\right) }  \cdot n^{\frac{7}{2}}
    \end{split}
\end{equation*}
\end{small}
\noindent $(1)$ follows from Assumption \ref{ass:cond_ind} and inequality \eqref{proof_res_for_cons_var_int_1_int_1} with probability at least $1-\delta$.
%the fact that, with probability at least $1-\delta$, the inequality
%\begin{equation}
%\label{proof_res_for_cons_var_int_1_int_1}
%    \begin{split}
%  \left\lvert \mbox{$\sum_{i=1}^n$} Z_i - \mbox{$\sum_{i=1}^n$} e_i\right\rvert \leq \sqrt{ 12 \log \left(2/{\delta}\right)} \cdot n^{\frac{1}{2}}  
%    \end{split}
%\end{equation}
%holds (\cite{angluin1977fast}).
 Therefore, Part $(b)$ is on the order of $\sqrt{\log \left( 2/{\delta} \right) } \cdot {n}^{-\frac{3}{2}}$. Combining the results of part $(a)$ and $(b)$, we have, with probability at least $1-2\delta$,  
\begin{small}
\begin{equation*}
    \begin{split}
      & \left\lvert  \frac{\mbox{$\sum_{i=1}^n$} \mbox{$\sum_{j=1}^n$} \mbox{$\sum_{h=1}^n$} Z_i Z_j Z_h (\tau_i-\tau_j)(\tau_i-\tau_h)}{(\mbox{$\sum_{i=1}^n$} Z_i)^4 }  -  \frac{\mbox{$\sum_{i=1}^n$} \mbox{$\sum_{j=1}^n$} \mbox{$\sum_{h=1}^n$} \mathbb{E} \left( Z_i Z_j Z_h (\tau_i-\tau_j)(\tau_i-\tau_h) \right)}{(\mbox{$\sum_{i=1}^n$} e_i)^4 }  \right\rvert    \\
      & \leq  C \cdot \sqrt{\log \left( 2/{\delta} \right)} \cdot n^{-\frac{3}{2}} \\
    \end{split}
\end{equation*}
\end{small}

\end{proof}

\section{Additional Simulations and Analysis}
We present additional simulation results that complement and expand on the results of the main text. 
% Sections \ref{sec:sim_synth}, \ref{sec:sim_or} and \ref{sec:additional_sim_tables} use the same simulation setup in Section~\ref{sec:experiments} in the main text. The combinations are labeled via (Graph Type, Interference Function, $m_{max}$), so for example (ER, exp$(n)$, $n$) means the underlying graph is an ER graph, interference has the exponential decay form of depth $n$, and $m_{max} = n$. 
Section~\ref{sec:simulation-with-x} presents extensions of our results to settings with Assumption~\ref{Assump:NoDirectImpact} is violated and the interference function depends on underlying covariates. Section~\ref{sim_equivalence_ADTT_exposure} demonstrates the adaptivity of the our proposed estimator to very shallow exposure mappings. Lastly, Section~\ref{sec:additional_figs_analysis} provides figures of the context tree under different penalty parameters $\lambda$ for the data analysis of Section \ref{sec_sin_out_pattern_0_1_2} in the main text.

Importantly, throughout this section we study substantially more complex data generating processes and hence consider a less conservative $\lambda = 0.4$ which is consistent with sharper bounds based on bootstrap approximations to bound the effective noise as in \cite{belloni2018high}.

\subsection{Estimator performance under violations of Assumption \ref{Assump:NoDirectImpact} }\label{sec:simulation-with-x}
In this section, we explore a setting where the interference function depends on covariates, thereby violating Assumption \ref{Assump:NoDirectImpact}. This violation suggests a natural extension of Algorithm $1$ to account for covariate by constructing $\widehat{\mathcal{K}}$ under distinct values of the discrete covariate \(X\)\footnote{Applying this approach to continuous covariates is also possible by considering adaptive coarsenings of the continuous covariates. This necessarily will increase the computational complexity of the algorithm, likely require additional assumptions to demonstrate validity and so we leave this for future work.}. Then when estimating \(\hat{f}_i\), the revised algorithm first matches on \(X\) and then matches on patterns based on $\widehat{\mathcal{K}}$. Alternatively, the elements of \(\widehat{\mathcal{K}}\) in Algorithm $1$ could be adjusted as joint representations of covariates and patterns, i.e., \((X_i, \gamma(G_i^Z))\). 

In this simulation, we consider an Erd\H{o}s R\'{e}nyi graph with $n=4000$, $p=4/n$ and a graph confounder ($X_i = s_i(1)$ is the degrees of individual \(i\)). The penalty parameter is set to \(\lambda = 0.4\), and there are \(M = 10,000\) Monte Carlo repetitions. The interference function is given by:
$$f(X_i,\gamma(G_i^Z)) = \mbox{$\sum_{j=1}^{n}$}\frac{1}{k^j}(s_i^Z(j)-s_i^Z(j-1))h(X_i)/2,$$ where $k=3$, and \(s_i^Z(j)\) denotes the number of treated units within the \(j\)-th distance for individual \(i\). The dependence on covariates comes through $h(X_i)$ which depends on the \textit{fixed} empirical degree distribution in the graph: $h(X_i) = 1, 2$ or $3$ if $X_i$ is in the first, second or third tertile of the empirical degree distribution in the observed graph. Table~\ref{tab_exp_4} summarizes the results from this experiment for our estimators $\hat{\tau}^{AIPW}$ and $\hat{\tau}^{OR}$ as well as the estimators \(\hat{\tau}^{IPW}\) and \(\hat{\tau}^{NV}\) that are defined by 
\begin{equation}
\label{tau_ipw}
  \hat{\tau}^{IPW} := \frac{\mbox{$\sum_{i=1}^n$ } \hat{W}_i Z_i Y_i}{\mbox{$\sum_{i=1}^n$ } \hat{W}_i Z_i} - \frac{\mbox{$\sum_{i=1}^n$ } \hat{W}_i (1-Z_i) Y_i}{\mbox{$\sum_{i=1}^n$ }\hat{W}_i (1-Z_i)}  
\end{equation}
where \(\hat{W}_i = Z_i + (1-Z_i) \frac{\hat{e}_i(X_i)}{1-\hat{e}_i(X_i)}\) for \(i \in [n]\). Here, the propensity scores \(\hat{e}_i(X_i)\) are estimated via logistic regression of the treatment assignment \(Z\) on the covariates \(X\). The naive difference-in-means estimator, \(\hat{\tau}^{NV}\), is written as
\begin{equation}
\label{tau_NV}
    \begin{split}
       \hat{\tau}^{NV} := \frac{\mbox{$\sum_{i=1}^n$} Z_i Y_i}{\mbox{$\sum_{i=1}^n$} Z_i} - \frac{\mbox{$\sum_{i=1}^n$} (1-Z_i) Y_i}{\mbox{$\sum_{i=1}^n$} (1-Z_i)}. 
    \end{split}
\end{equation}
    \begin{table}[H]
		\begin{center}
			%\begin{tabular}{ p{1.5cm}|p{1.2cm}|p{1.2cm}| p{1.2cm}|p{1.2cm}|p{1.2cm}|p{1.2cm}|p{1.2cm}|p{1.2cm} 
			%}
			\caption{
      Comparison of estimator performance by stratifying \(X\) and constructing corresponding \(\widehat{\mathcal{K}}\): bias, standard error, and rmse (ADTT = $2.505$)
            }
			\vspace{0.2 cm}
			\label{tab_exp_4}
			\begin{tabular}[\linewidth]{c|c|c|c|c}
		     \hline 
			  & & & & \\[-1.8ex]
				% & $\mathbb{E}_M (\hat{\tau})$ 
                & bias & se & rmse$(\tau)$ &rmse$(\tau_{ADTE})$  \\%[1ex]
				\hline
				& & & & \\[-1.8ex]
				% ADTT & $1.1672$  & $0$  & $0.0085$ & $0$  \\
				% \hline\hline 
				% $\hat{\tau}^{OR}_{ns}$
                $\hat{\tau}^{OR}$ & 
                % $2.526$ &  
                $0.021$ & $0.082$ & $0.079$  
                & 0.085\\ 
				% $\hat{\tau}^{AIPW}_{ns}$ & 
    %             % $2.513$ &  
    %             $0.007$ & $0.081$ & $0.075$ 
    %             & 0.082 \\
				% $\hat{\tau}^{OR}_{s}$ & 
    %             % $2.523$ & 
    %             $0.017$ & $0.082$ & $0.078$  
    %             & 0.084 \\ 
				% $\hat{\tau}^{AIPW}_{s}$ 
                $\hat{\tau}^{AIPW}$& 
                % $2.505$ &  
                $0.000$ & $0.081$ & $0.074$ 
                & 0.081\\
                $\hat{\tau}^{IPW}$ &  
                % $2.414$ & 
                $-0.091$ & $0.090$ & $0.125$ 
                & 0.128\\
				$\hat{\tau}^{NV}$ & 
                % $3.887$ &  
                $1.381$ & $0.163$& $1.388$ 
                & 1.391\\ 
				% & & & & & & & &  \\
				\hline
			\end{tabular}
		\end{center}
        \vspace{0.1 cm}
    \parbox{\textwidth}{\footnotesize
  %  $\hat{\tau}^{IPW}$ and $\hat{\tau}^{NV}$ are defined as in \eqref{tau_ipw} and \eqref{tau_NV}, respectively. 
    }
	\end{table}
Here \(\widehat{\mathcal{K}}\) is constructed under correctly specified strata of \(X\). Under this specification, we see no reduction in the performance of either one of our estimators while we see a major degradation in the performance of $\hat{\tau}^{NV}$. 

\subsection{Demonstrating adaptivity}
\label{sim_equivalence_ADTT_exposure}
In this section, we consider a setting where potential outcomes are induced by an exposure mapping where the exposure mapping is the same as in Section $9$ of \cite{aronow2017estimating} but the interference for each unit is scaled by the unit’s covariate, leading to a violation of Assumption \ref{Assump:NoDirectImpact}.
As in Section~\ref{sec:simulation-with-x}, the graph \(G \sim ER(n, p)\), where \(n = 4000\) and \(p = 4/n\). The degree of each unit, \(X_i = s_i(1)\), serves as a confounder. The binary treatment \(Z_i\) is independently assigned to each unit with probability \(P(Z_i \mid X_i = s_i(1)) = \frac{s_i(1)}{C_0}\), where \(C_0 = \frac{3}{2} \max_{i \in [n]} s_i(1)\). The potential outcomes are defined as:  
\begin{equation}
\label{pot_out_exp_map}
\begin{split}
\widetilde {Y}_i(z_i,{z}_{-i}) & = 0.8 X_i \cdot z_i+ 0.4 X_i \cdot 1\{s^{{z}}_i(1)>0\}+ \epsilon_i : = \tau_i \cdot z_i+ f(X_i ,\gamma(G^{z}_i))+ \epsilon_i \\
\end{split}    
\end{equation}
where $\epsilon_i\overset{iid}{\sim}N(0,1.5^2)$. The interference effect depends on both the covariate \(X_i\) and the network patterns. 
% These potential outcomes align with the exposure mapping described in Section 9 of \cite{aronow2017estimating}. 
The potential outcomes can thus be represented by the following exposure mapping:
\begin{equation}
\label{instance_exp_mapping}
D_i=
\begin{cases}
d_{11} \ \ \ \mathrm{if} \  z_i \cdot 1\{ s_i^{z}(1) >0 \} >0    \\
d_{10} \ \ \ \mathrm{if} \  z_i \cdot 1\{ s_i^{z}(1) =0 \} >0 \\
d_{01} \ \ \ \mathrm{if} \  (1-z_i) \cdot 1\{ s_i^{z}(1) >0 \} >0\\
d_{00} \ \ \ \mathrm{if} \  (1-z_i) \cdot 1\{s_i^{z}(1) =0 \} >0,\\
\end{cases}
\end{equation}
and the ADTT can be expressed as:  
\[
ADTT = \frac{1}{\mbox{$\sum_{i=1}^n$} Z_i} \mbox{$\sum_{i=1}^n$} Z_i \big(Y_i(d_{11}) - Y_i(d_{01})\big) = \frac{1}{\mbox{$\sum_{i=1}^n$} Z_i} \mbox{$\sum_{i=1}^n$} Z_i \big(Y_i(d_{10}) - Y_i(d_{00})\big).
\]
Building on the idea of balancing covariates introduced in \cite{li2018balancing} and the exposure mapping above, we can write down the H\'{a}jek estimator for ADTT as:  
\begin{equation}
\begin{split}
\label{HJ_est}
\hat{\tau}^{HJ} & =  \beta \cdot \left[ \frac{\mbox{$\sum_{i=1}^n$} w_i(d_{11}) \cdot 1\{ D_i=d_{11} \} \cdot Y_i }{ \mbox{$\sum_{i=1}^n$} w_i(d_{11}) \cdot 1 \{D_i = d_{11} \} }   
- \frac{\mbox{$\sum_{i=1}^n$} w_i(d_{01}) \cdot 1\{ D_i=d_{01} \} \cdot Y_i }{ \mbox{$\sum_{i=1}^n$} w_i(d_{01}) \cdot 1\{ D_i=d_{01} \} } \right]   \\
& \quad + (1-\beta) \cdot \left[ \frac{\mbox{$\sum_{i=1}^n$} w_i(d_{10}) \cdot 1\{ D_i=d_{10} \} \cdot Y_i }{ \mbox{$\sum_{i=1}^n$} w_i(d_{10}) \cdot 1\{ D_i=d_{10} \} }   
- \frac{\mbox{$\sum_{i=1}^n$} w_i(d_{00}) \cdot 1\{ D_i=d_{00} \} \cdot Y_i }{ \mbox{$\sum_{i=1}^n$} w_i(d_{00}) \cdot 1\{ D_i=d_{00} \} } \right] \\
& := \beta \cdot \hat{\tau}^{HJ}(d_{11}, d_{01}) + (1-\beta) \cdot \hat{\tau}^{HJ}(d_{10}, d_{00}),
\end{split}
\end{equation}
where \(\beta = \frac{\mbox{$\sum_{i=1}^n$} \big(1\{D_i = d_{11}\} + 1\{D_i = d_{01}\}\big)}{n}\). The weights are defined as \(w_i(d_{11}) = \frac{\mathbb{P}(Z_i=1)}{\mathbb{P}(D_i = d_{11})}\), \(w_i(d_{01}) = \frac{\mathbb{P}(Z_i=1)}{\mathbb{P}(D_i = d_{01})}\), \(w_i(d_{10}) = \frac{\mathbb{P}(Z_i=1)}{\mathbb{P}(D_i = d_{10})}\), and \(w_i(d_{00}) = \frac{\mathbb{P}(Z_i=1)}{\mathbb{P}(D_i = d_{00})}\). Since our proposed estimators and $\hat{\tau}^{HJ}$ are estimating the same quantity, this allows for a direct comparison of the performance for these estimators. 
For all estimators we assume that the assignment mechanism is known.
% Notably, the treatment assignment mechanism is assumed to be known in \cite{aronow2017estimating}. To ensure a fair comparison, we similarly assume known propensities for \(\hat{\tau}^{AIPW}_s\). 
The comparative performance of these estimators is summarized below. 
\begin{small}
	\begin{table}[H]
		\begin{center}
			%\begin{tabular}{ p{1.5cm}|p{1.2cm}|p{1.2cm}| p{1.2cm}|p{1.2cm}|p{1.2cm}|p{1.2cm}|p{1.2cm}|p{1.2cm} 
			%}
			\caption{Comparison of estimator performance under interference induced by the exposure mapping in equation \eqref{instance_exp_mapping} (ADTT = $4.009$)}
			\vspace{0.2cm}
			\label{tab_exp_5}
			\begin{tabular}{p{1.5cm}|
            p{1.2cm}|
            p{1.2cm}| p{1.2cm}|p{1.2cm}}
            % { >{\footnotesize}c|>{\footnotesize}c|>{\footnotesize}c|>{\footnotesize}c|>{\footnotesize}c}
			\hline 
				& 
                & 
                & &   \\[-1.8ex]
				% & $\mathbb{E}_M(\hat{\tau})$ 
                & bias & se & rmse$(\tau)$ & rmse$(\tau_{ADTE})$ \\[1ex]
				\hline
				% ADTT & $3.0169$ & $0$ & $0.0439$  & $0$  \\
				% \hline\hline 
				% $\hat{\tau}^{OR}_{ns}$ 
                $\hat{\tau}^{OR}$ 
                % & $4.008$ 
                & $-0.001$ & $0.082$ & $0.066$  
                & 0.082\\ 
				% $\hat{\tau}^{AIPW}_{ns}$ 
    %             % & $4.008$ 
    %             & $-0.001$ & $0.081$ & $0.065$ 
    %             & 0.081\\
    %             $\hat{\tau}^{OR}_{s}$ 
    %             % & $4.009$ 
    %             &  $0.001$ & $0.081$ & $0.065$  
    %             & 0.081\\ 
				% $\hat{\tau}^{AIPW}_{s}$
                $\hat{\tau}^{AIPW}$
                % & $4.009$ 
                & $0.000$ & $0.081$ & $0.065$ 
                & 0.081\\
                 $\hat{\tau}^{HJ}$  
                 % & $4.008$ 
                 & $-0.001$ & $0.120$ & $0.101$ 
                 & 0.120\\ 
                $\hat{\tau}^{IPW}$ 
                % & $3.973$ 
                & $-0.036$ & $0.091$ & $0.086$ 
                & 0.098\\ 
				$\hat{\tau}^{NV}$ 
                % & $4.479$ 
                &  $0.470$ & $0.109$ & $0.477$ 
                & 0.483\\ 
				% & & & & & & & &  \\
				\hline
			\end{tabular}
		\end{center}
	\end{table}
\end{small}
Given that the exposure mapping is known, $\hat{\tau}^{HJ}$ is anticipated to perform well. Based on Table \ref{tab_exp_5}, our proposed estimators achieve comparable bias while exhibiting lower standard errors compared to \(\hat{\tau}^{HJ}\).

\subsection{Additional figures for the data analysis}\label{sec:additional_figs_analysis}
This section provides two figures of context trees for $\lambda = 0.4$ (Figure~\ref{Fig:real_data_pattern_0_1_2_lambda_0_4_est}) and $\lambda = 0.7$ (Figure~\ref{Fig:real_data_pattern_0_1_2_lambda_0_7_est}) based on the real world data described and analyzed in Section~\ref{sec_sin_out_pattern_0_1_2} of the main text.

\begin{minipage}{0.45\textwidth}
\begin{small}
\begin{figure}[H]
\resizebox{.7\textwidth}{!}{
       \begin{tikzpicture}
    % 0-th layer
    \node[align=left] at (5,7) {$0$};
    %\node[align=left] at (4,6.45) {$1$};
    \draw[thick] (5,7) circle (0.4cm);
    \node[align=left] at (5,6.4) {$0.044$};
    \node[align=left] at (5,6.0) {$(1670)$};
    
    \draw[thick] (2,5.5) -- (5,5.8); 
    \draw[thick] (5,5.5) -- (5,5.8);
    \draw[thick] (8,5.5) -- (5,5.8); 

    % 1-st layer
    \node[align=left] at (2,5.1) {$0$};
    \node[align=left] at (5,5.1) {$1$};
    \node[align=left] at (8,5.1) {$\geq 2$};
    
    \draw[thick] (2,5.1) circle (0.4cm);
    \draw[thick] (5,5.1) circle (0.4cm);
    \draw[thick] (8,5.1) circle (0.4cm);

    % estimator 
    \node[align=left] at (2,4.5) {$0.009$};
    \node[align=left] at (5,4.5) {$0.153$};
    \node[align=left] at (8,4.5) {$0.186$};
    
    \node[align=left] at (2,4.1) {$(1296)$};
    \node[align=left] at (5,4.1) {$(261)$};
    \node[align=left] at (8,4.1) {$(113)$};

    % 2-nd layer 

    \node[align=left] at (2,3.1) {$1$};
    \node[align=left] at (2,2.5) {$0.0741$};
    \node[align=left] at (2,2.1) {$(54)$};
    \draw[thick] (2,3.1) circle (0.4cm);
    \draw[thick] (2,3.5) -- (2,3.9); 
    
    \node[align=left] at (3.5,3.1) {$0$};
    \node[align=left] at (5,3.1) {$1$};
    \node[align=left] at (6.5,3.1) {$\geq 2$};
    
    \node[align=left] at (3.5,2.5) {$0.148$};
    \node[align=left] at (3.5,2.1) {$(128)$};
    \node[align=left] at (5,2.5) {$0.134$};
    \node[align=left] at (5,2.1) {$(82)$};
    \node[align=left] at (6.5,2.5) {$0.196$};
    \node[align=left] at (6.5,2.1) {$(51)$};

    \draw[thick] (3.5,3.1) circle (0.4cm);
    \draw[thick] (5,3.1) circle (0.4cm);
    \draw[thick] (6.5,3.1) circle (0.4cm);

    \draw[thick] (3.5,3.5) -- (5,3.9); 
    \draw[thick] (5,3.5) -- (5,3.9); 
    \draw[thick] (6.5,3.5) -- (5,3.9); 

    \node[align=left] at (8,3.1) {$0$};
    \node[align=left] at (9.5,3.1) {$1$};
    \node[align=left] at (11,3.1) {$\geq 2$};
    
    \node[align=left] at (8,2.5) {$0.147$};
    \node[align=left] at (8,2.1) {$(34)$};
    \node[align=left] at (9.5,2.5) {$0.111$};
    \node[align=left] at (9.5,2.1) {$(27)$};
    \node[align=left] at (11,2.5) {$0.250$};
    \node[align=left] at (11,2.1) {$(52)$};

    \draw[thick] (8,3.1) circle (0.4cm);
    \draw[thick] (9.5,3.1) circle (0.4cm);
    \draw[thick] (11,3.1) circle (0.4cm);

    \draw[thick] (8,3.5) -- (8,3.9); 
    \draw[thick] (9.5,3.5) -- (8,3.9);  
    \draw[thick] (11,3.5) -- (8,3.9); 

    % 3-rd layer

    \node[align=left,scale=0.8] at (2.5,1.2) {$0$};
    \node[align=left,scale=0.8] at (3.5,1.2) {$1$};
    \node[align=left,scale=0.8] at (5,1.2) {$1$};
    \node[align=left,scale=0.8] at (6.5,1.2) {$0$};
    \node[align=left,scale=0.8] at (7.5,1.2) {$1$};
    \node[align=left,scale=0.8] at (8.5,1.2) {$\geq 2$};
    
    \node[align=left,scale=0.8] at (2.5,0.7) {$0.140$};
    \node[align=left,scale=0.8] at (2.5,0.4) {$(100)$};
    \node[align=left,scale=0.8] at (3.5,0.7) {$0.100$};
    \node[align=left,scale=0.8] at (3.5,0.4) {$(20)$};
    \node[align=left,scale=0.8] at (5,0.7) {$0.091$};
    \node[align=left,scale=0.8] at (5,0.4) {$(22)$};
    \node[align=left,scale=0.8] at (6.5,0.7) {$0$};
    \node[align=left,scale=0.8] at (6.5,0.4) {$(9)$};
    \node[align=left,scale=0.8] at (7.5,0.7) {$0.333$};
    \node[align=left,scale=0.8] at (7.5,0.4) {$(15)$};
    \node[align=left,scale=0.8] at (8.5,0.7) {$0.185$};
    \node[align=left,scale=0.8] at (8.5,0.4) {$(27)$};

    \draw[thick] (2.5,1.2) circle (0.3cm);
    \draw[thick] (3.5,1.2) circle (0.3cm);
    \draw[thick] (5,1.2) circle (0.3cm);
    \draw[thick] (6.5,1.2) circle (0.3cm);
    \draw[thick] (7.5,1.2) circle (0.3cm);
    \draw[thick] (8.5,1.2) circle (0.3cm);

    \draw[thick] (2.5,1.5) -- (3.5, 1.9); 
    \draw[thick] (3.5,1.5) -- (3.5,1.9); 
    \draw[thick] (5,1.5) -- (5,1.9);
    \draw[thick] (6.5,1.5) -- (6.5,1.9);
    \draw[thick] (7.5,1.5) -- (6.5,1.9);
    \draw[thick] (8.5,1.5) -- (6.5,1.9);
    
    % 4-th layer
    \node[align=left,scale=0.8] at (8.5,-0.4) {$1$};
    \node[align=left,scale=0.8] at (9.5,-0.4) {$\geq 2$};

    \node[align=left,scale=0.8] at (8.5,-1.2) {$(6)$};
    \node[align=left,scale=0.8] at (8.5,-0.9) {$0$};
    \node[align=left,scale=0.8] at (9.5,-0.9) {$0.2778$};
    \node[align=left,scale=0.8] at (9.5,-1.2) {$(18)$};
    
    \draw[thick] (8.5,-0.4) circle (0.3cm);
    \draw[thick] (9.5,-0.4) circle (0.3cm);
    
    \draw[thick] (8.5,-0.1) -- (8.5,0.2); 
    \draw[thick] (9.5,-0.1) -- (8.5,0.2); 
\end{tikzpicture}}
        \caption{Structure of \(\widehat{\mathcal{K}}\), interference estimators, and control unit counts at tree nodes generated by algorithm $2$ (\(\lambda = 0.4\), node size \(\geq 5\)) 
        }
\label{Fig:real_data_pattern_0_1_2_lambda_0_4_est}
\end{figure}
\end{small}
\end{minipage}
\begin{minipage}{0.45\textwidth}
\begin{small}
\begin{figure}[H]
\resizebox{.7\textwidth}{!}{       
\begin{tikzpicture}
    % 0-th layer
    \node[align=left] at (5,7) {$0$};
    %\node[align=left] at (4,6.45) {$1$};
    \draw[thick] (5,7) circle (0.4cm);
    \node[align=left] at (5,6.4) {$0.044$};
    \node[align=left] at (5,6.0) {$(1670)$};
    
    \draw[thick] (2,5.4) -- (5,5.8); 
    \draw[thick] (5,5.4) -- (5,5.8);
    \draw[thick] (8,5.4) -- (5,5.8); 

    % 1-st layer
    \node[align=left] at (2,5) {$0$};
    \node[align=left] at (5,5) {$1$};
    \node[align=left] at (8,5) {$\geq 2$};
    
    \draw[thick] (2,5) circle (0.4cm);
    \draw[thick] (5,5) circle (0.4cm);
    \draw[thick] (8,5) circle (0.4cm);

    % estimator 
    \node[align=left] at (2,4.4) {$0.009$};
    \node[align=left] at (2,4.0) {$(1296)$};
    \node[align=left] at (5,4.4) {$0.153$};
    \node[align=left] at (5,4.0) {$(261)$};
    \node[align=left] at (8,4.4) {$0.186$};
    \node[align=left] at (8,4.0) {$(113)$};

    % 2-nd layer 
    \node[align=left] at (3.5,3) {$0$};
    \node[align=left] at (5,3) {$1$};
    \node[align=left] at (6.5,3) {$\geq 2$};
    
    \node[align=left] at (3.5,2.4) {$0.148$};
    \node[align=left] at (3.5,2) {$(128)$};
    \node[align=left] at (5,2.4) {$0.134$};
    \node[align=left] at (5,2) {$(82)$};
    \node[align=left] at (6.5,2.4) {$0.196$};
    \node[align=left] at (6.5,2) {$(51)$};

    \draw[thick] (3.5,3) circle (0.4cm);
    \draw[thick] (5,3) circle (0.4cm);
    \draw[thick] (6.5,3) circle (0.4cm);

    \draw[thick] (3.5,3.4) -- (5,3.8); 
    \draw[thick] (5,3.4) -- (5,3.8); 
    \draw[thick] (6.5,3.4) -- (5,3.8); 

    %\node[align=left] at (8.7,4.6) {$4$};
    %\node[align=left] at (9.4,4.6) {$\cdots$};
    %\node[align=left] at (10.1,4.6) {$8$};

    %\draw[thick] (8.7,4.6) circle (0.25cm);
    %\draw[thick] (10.1,4.6) circle (0.25cm);

    %\draw[thick] (8.7,4.85) -- (9.5,5.4); 
    %\draw[thick] (10.1,4.85) -- (9.5,5.4); 

    \node[align=left] at (8,3) {$0$};
    \node[align=left] at (9.5,3) {$1$};
    
    \node[align=left] at (8,2.4) {$0.147$};
    \node[align=left] at (8,2) {$(34)$};
    \node[align=left] at (9.5,2.4) {$0.111$};
    \node[align=left] at (9.5,2) {$(27)$};

    \draw[thick] (8,3) circle (0.4cm);
    \draw[thick] (9.5,3) circle (0.4cm);

    \draw[thick] (8,3.4) -- (8,3.8); 
    \draw[thick] (9.5,3.4) -- (8,3.8);  

    % third layer

    \node[align=left] at (2.5,1) {$1$};
    
    \node[align=left] at (2.5,0.3) {$0.100$};
    \node[align=left] at (2.5,-0.1) {$(20)$};

    \draw[thick] (2.5,1) circle (0.4cm);

    \draw[thick] (2.5,1.4) -- (3.5,1.8);

\end{tikzpicture}}
        \caption{Structure of \(\widehat{\mathcal{K}}\), interference estimators, and control unit counts at tree nodes generated by algorithm $2$ (\(\lambda = 0.7\), node size \(\geq 5\)) 
        }
\label{Fig:real_data_pattern_0_1_2_lambda_0_7_est}
\end{figure}
\end{small}
\end{minipage}

\end{appendix}
%%%%%%%%%%%%%%%%%%%%%%%%%%%%%%%%%%%%%%%%%%%%%%
%% Example with multiple Appendixes:        %%
%%%%%%%%%%%%%%%%%%%%%%%%%%%%%%%%%%%%%%%%%%%%%%

%%%%%%%%%%%%%%%%%%%%%%%%%%%%%%%%%%%%%%%%%%%%%%
%% Support information, if any,             %%
%% should be provided in the                %%
%% Acknowledgements section.                %%
%%%%%%%%%%%%%%%%%%%%%%%%%%%%%%%%%%%%%%%%%%%%%%
\begin{acks}[Acknowledgments]
We thank the participants from the AEA meeting 2024, 2024 Conference on Network Science and Economics, 2024 GraphEx Workshop, Design and Analysis of Networked Experiments 2024, Causal Inference and Prediction for Network Data Workshop at BIRS, Optimization-Conscious Econometrics Conference II, Triangle Econometrics Conference, and seminar participants from MIT,  Northwestern University, Columbia, Purdue, University of Pittsburg, and University of Washington. We are very thankful for detailed comments and suggestions from Dmitry Arkhangelsky and Michael Leung. 
\end{acks}

%%%%%%%%%%%%%%%%%%%%%%%%%%%%%%%%%%%%%%%%%%%%%%
%% Funding information, if any,             %%
%% should be provided in the                %%
%% funding section.                         %%
%%%%%%%%%%%%%%%%%%%%%%%%%%%%%%%%%%%%%%%%%%%%%%
\begin{funding}
Professor Volfovsky was partially supported by NSF DMS-2230074 and DMS-2046880.
\end{funding}

%%%%%%%%%%%%%%%%%%%%%%%%%%%%%%%%%%%%%%%%%%%%%%
%% Supplementary Material, including data   %%
%% sets and code, should be provided in     %%
%% {supplement} environment with title      %%
%% and short description. It cannot be      %%
%% available exclusively as external link.  %%
%% All Supplementary Material must be       %%
%% available to the reader on Project       %%
%% Euclid with the published article.       %%
%%%%%%%%%%%%%%%%%%%%%%%%%%%%%%%%%%%%%%%%%%%%%%

%%  MR numbers will be added by VTeX.                      %%
%%                                                         %%
%%  Use \cite{...} to cite references in text.             %%
%%                                                         %%
%%%%%%%%%%%%%%%%%%%%%%%%%%%%%%%%%%%%%%%%%%%%%%%%%%%%%%%%%%%%%

%% if your bibliography is in bibtex format, uncomment commands:
\bibliographystyle{agsm}
\bibliography{biblio}       % Bibliography file (usually '*.bib')

\end{document}